\RequirePackage{fix-cm}
\documentclass[smallextended]{svjour3}       
\smartqed  
\usepackage{graphicx}
\usepackage{amsfonts}
\usepackage{amsmath}
\usepackage{subfigure}
\usepackage{bm}
\usepackage{algorithmic}
\usepackage{algorithm}
\usepackage{appendix}

\usepackage[pagebackref=true,breaklinks=true,colorlinks,bookmarks=false]{hyperref}

\begin{document}

\title{Topology-Preserving 3D Image Segmentation Based On Hyperelastic Regularization}

\titlerunning{Topology-Preserving Image Segmentation}        

\author{Daoping Zhang \and Lok Ming Lui}



\institute{Daoping Zhang \at
              Department of Mathematics, The Chinese University of Hong Kong, 228, Lady Shaw Building,        Shatin, Hong Kong \\
              \email{dpzhang@math.cuhk.edu.hk.}           
           \and
           Lok Ming Lui \at
              Department of Mathematics, The Chinese University of Hong Kong, 207, Lady Shaw Building, Shatin, Hong Kong.\\
               \email{lmlui@math.cuhk.edu.hk.}
}

\date{Received: date / Accepted: date}

\maketitle

\begin{abstract}
Image segmentation is to extract meaningful objects from a given image. For degraded images due to occlusions, obscurities or noises, the accuracy of the segmentation result can be severely affected. To alleviate this problem, prior information about the target object is usually introduced. In \cite{chan2018topology}, a topology-preserving registration-based segmentation model was proposed, which is restricted to segment 2D images only. In this paper, we propose a novel 3D topology-preserving registration-based segmentation model with the hyperelastic regularization, which can handle both 2D and 3D images. The existence of the solution of the proposed model is established. We also propose a converging iterative scheme to solve the proposed model. Numerical experiments have been carried out on the synthetic and real images, which demonstrate the effectiveness of our proposed model.
\end{abstract}


\section{Introduction}\label{SIntroduction}

The goal of image segmentation is to partition a given image into multiple meaningful segments, which assists many image analysis tasks. As one of the most important problems in image processing, image segmentation has a wide range of applications, especially in medical imaging \cite{balafar2010review,heimann2009statistical}. 

Generally, the methods to image segmentation can be divided into two categories: region-based methods and edge-based methods. For the region-based methods, the Mumford-Shah model \cite{mumford1989optimal} and the Chan-Vese model \cite{chan2001active} may be the most famous two. The Mumford-Shan model aims to find a piecewise smooth approximation and segment the boundary of the target object. The Chan-Vese model employs the level set implementation to solve the piecewise constant Mumford-Shah model, which uses the Heaviside function as an indicator function. When the Chan-Vese model appeared, it quickly attracted more and more attentions and a lot of works based on it were proposed to deal with various situations. Specifically, \cite{chan2000active} extends the Chan-Vese model to the vector-valued images, \cite{chan2002active} considers the texture images and \cite{lie2006variant} aims segmenting the multiple objects simultaneously. However, these models are nonconvex and the solution may be a local minimizer and depend on the position of the initial guess. To overcome this difficulty, \cite{chan2006algorithms} proposed a method to find the global minimizer by restating the Chan-Vese model to a convex model under certain conditions.
For the edge-based methods, the first one is proposed in \cite{kass1988snakes} by deforming a parameter curve towards the final solution curve to locate the edges of the target object. Here, the curve is deformed by acting the internal and external forces. To make the improvement, \cite{caselles1997geodesic} introduces a geodesic active contour model which allows the topology changes of the evolving curves and \cite{li2007active} proposes a vector field convolutions as a new external force to enhance the robustness with respect to the noise.

The above mentioned methods are all global segmentation methods, namely finding all the meaning objects in the given image. However, under some circumstances, we just want to pick up some particular target objects rather than all the target objects. The model to meet this requirement is called the selective segmentation model. A selective segmentation model is proposed in \cite{gout2005segmentation} by combing the geodesic active contour with the geometric constraint. \cite{badshah2010image} further incorporates an intensity-based constraint to improve the robustness. However, this model may fail if the adjacent objects have the similar intensity value with the target object. Then \cite{zhang2014local} develops an adaptively varying narrow band algorithm to overcome this difficulty. More recently, \cite{roberts2019convex} proposes a convex selective model by incorporating an edge-weighted geodesic distance, which allows arbitrary initialization. In addition, \cite{zhang2015fast} proposes an effective and efficient algorithm for the 3D selective segmentation.

Although there are many kinds of different segmentation methods, the task of image segmentation remains challenging in many scenarios. For instance, when the given image is corrupted or obscured, many existing segmentation tools may fail to segment the desired object accurately, but excluding some crucial regions. Also, when images are degraded by overexposure or underexposure, most intensity-based segmentation models cannot segment the target object accurately as a whole. Thus, it is necessary to develop a more accurate segmentation method to handle these situations. In order to overcome this difficulty, prior information about the target object is usually introduced. \cite{cremers2003towards} combines the shape prior with the Chan-Vese model. Here, a labelling level set function is introduced to indicate the regions where the shape prior is enforced. \cite{chan2005level} also introduces a labelling level set function but here, it allows the translation, rotation and scaling of the prior shapes. \cite{thiruvenkadam2007segmentation} combines the shape prior with the level set method which can segment multiple covering objects simultaneously. 

In this work, we consider the topological structure of the target object as the prior information. In 2018, Chan et al. proposed a topology-preserving quasiconformal-based segmentation model \cite{chan2018topology}. The main idea of this work is to bijectively deform a reference template with the prescribed topology to segment the target object in an image. The deformation map is obtained by finding an optimal quasiconformal mapping. Quasiconformal mappings have been widely used for different imaging tasks \cite{QCImaging}, including image registration \cite{QCregistration1,QCregistration2,QCregistration3,QCregistration4,QCregistration5,QCregistration6} and image analysis \cite{Shapeanalysis1,Shapeanalysis2,Shapeanalysis4,Shapeanalysis5,Shapeanalysis6}. Since the deformation map is bijective, the topology-preserving property can be guaranteed. This method has shown to be successful to achieve accurate segmentation results, even for degraded 2D images. The method is further extended to image segmentation problem with convexity prior enforced \cite{QCsegmentationconvex}. Nevertheless, this model suffers from a major limitation. As the method is based on the quasi-conformal theory that is only defined in complex space, the segmentation model can only deal with 2D images and can not be extended to 3D (or higher dimensional) segmentation problems.  

In this paper, our main goal is to extend the idea of \cite{chan2018topology} to 3D segmentation. More specifically, we propose a novel 3D topology-preserving registration-based segmentation model with the hyperelastic regularization \cite{burger2013hyperelastic,droske2004variational}. The existence of the solution of our proposed model is theoretically established. In addition, we propose a converging Generalized Gauss-Newton iterative scheme to solve the proposed model. Numerical experiments have been carried out on the synthetic and real images, which illustrate the effectiveness of our proposed method. 

The contributions of this paper are three-fold:
\begin{enumerate}
    \item First, this work proposes a novel topology-preserving segmentation model, which can handle images in 3D.
    \item Second, we prove the existence of the solution of our proposed model. Here, compared with the conventional variational model for image registration, the proposed model contains not only an infinite dimensional variable but also a finite dimensional variable. To employ the direct method in the calculus of variations, we make a modification for \cite{ruthotto2012hyperelastic} and introduce a product space to complete the proof.
    \item Based on the generalized Gauss-Newton framework, we develop an iterative scheme to numerically solve the proposed model. We also prove the subsequence generated by this iterative scheme can converge to a critical point.
\end{enumerate}

The paper is organized as follows. The related works are reviewed in Section \ref{SRelatedWork}. In Section \ref{SProposedModel}, a proposed segmentation model is given in details. In Section \ref{SNumericalImplementation}, numerical implementation is described and numerical experiments are illustrated in Section \ref{SNumericalResults}. Finally, a conclusion is summarized in Section \ref{SConclusion}.

\section{Related Work}\label{SRelatedWork}
In this section, we review the related works, including image segmentation, image registration and registration-based segmentation.

\subsection{Image Segmentation}
Image segmentation is to identify the boundary of the target object to extract the target object. Let $\Omega$ be a rectangular image domain and $I:\Omega\rightarrow\mathbb{R}$ be an image. Then the boundary of the target object can be considered as a closed curve $C\subset\Omega$. 

As one of the most classical segmentation model, the piecewise constant Mumford-Shah model \cite{mumford1989optimal} is to find a best piecewise constant function with unknown values $c_{1}$ and $c_{2}$ to approximate the image function $I(\bm{x})$ with proper regularizations. Specifically, its variational model can be formulated as follows:
\begin{equation}\label{PMS}
\min_{C,c_{1},c_{2}} \mathcal{E}(C,c_{1},c_{2}) := \lambda_{1}\int_{\mathrm{int}(C)}(I(\bm{x})-c_{1})^{2}\mathrm{d}\bm{x}+\lambda_{2}\int_{\mathrm{ext}(C)}(I(\bm{x})-c_{2})^{2}\mathrm{d}\bm{x}+\lambda_{3}\mathrm{Length}(C)+\lambda_{4}\mathrm{Area}(\mathrm{int}(C)),
\end{equation}
where $\lambda_{i}, i=1,...,4$ are nonnegative weighting parameters. 

In order to solve the model \eqref{PMS}, by using the Heaviside function $H$, Dirac measure $\delta_{0}$ and the level set function $\phi$, the variational model \eqref{PMS} can be converted into the following equivalent formulation, namely, the famous Chan-Vese model:
\begin{equation}\label{CV}
\begin{split}
\min_{\phi,c_{1},c_{2}} \mathcal{E}(\phi,c_{1},c_{2}) := & \lambda_{1}\int_{\Omega}(I(\bm{x})-c_{1})^{2}H(\phi(\bm{x}))\mathrm{d}\bm{x}+\lambda_{2}\int_{\Omega}(I(\bm{x})-c_{2})^{2}(1-H(\phi(\bm{x})))\mathrm{d}\bm{x}\\
&+\lambda_{3}\int_{\Omega}\delta_{0}(\phi(\bm{x}))|\nabla\phi(\bm{x})|\mathrm{d}\bm{x}+\lambda_{4}\int_{\Omega}H(\phi(\bm{x}))\mathrm{d}\bm{x}.
\end{split}
\end{equation}
For more details about the Chan-Vese model, please refer to \cite{brown2012completely,chan2000active,chan2005image,chan2001active,getreuer2012chan}.

\subsection{Image Registration}

Image registration aims to find a plausible transformation to match the corresponding data. The general variational framework for image registration \cite{modersitzki2004numerical,modersitzki2009fair} can be built as follows:
\begin{equation}\label{RegistrationFramework}
\min_{\bm{y}} \mathcal{J}(\bm{y}):=\mathcal{D}(T\circ\bm{y},R)+\alpha\mathcal{R}(\bm{y}),
\end{equation}
where $\bm{y}(\bm{x}):\mathbb{R}^{d}\rightarrow\mathbb{R}^{d}$ is the transformation, $T(\bm{x}):\Omega\subset\mathbb{R}^{d}\rightarrow\mathbb{R}$ is the template, $R(\bm{x}):\Omega\subset\mathbb{R}^{d}\rightarrow\mathbb{R}$ is the reference, $d$ is the dimension of the image, $T\circ\bm{y}$ is the deformed template, $\mathcal{D}(T\circ\bm{y},R)$ is the fitting term to measure the difference between the deformed template and reference, $\mathcal{R}(\bm{y})$ is the regularization and $\alpha$ is a non-negative weighting parameter.

For the fitting term, under the the monomodality case, 
the most widely used choice is the sum of squared differences (SSD) \cite{modersitzki2004numerical,modersitzki2009fair,zhang2018novel}:
\begin{equation}
\mathcal{D}^{\mathrm{SSD}}(T\circ\bm{y},R):=\frac{1}{2}\int_{\Omega}(T(\bm{y}(\bm{x}))-R(\bm{x}))^{2}\mathrm{d}\bm{x}.
\end{equation}
For the multimodality image registration, mutual information, normalized cross correlation or normalized gradient field may be a good candidate \cite{haber2006intensity,haber2007intensity,maes1997multimodality,modersitzki2009fair}. 

For the regularization term, there also exist many different choices \cite{broit1981optimal,burger2013hyperelastic,christensen1996deformable,chumchob2011fourth,droske2004variational,fischer2002fast,fischer2003curvature,fischer2004unified,ibrahim2015novel,zhang2015variational} and here, we mainly highlight the hyperelastic regularizer \cite{burger2013hyperelastic,droske2004variational}. The hyperelastic regularizer in image registration was firstly used by Droske and Rumpf \cite{droske2004variational} in 2004. Generally, the formulation of the hyperelastic regularizer is defined as follows:
\begin{equation}\label{HRegularizer}
\mathcal{R}^{\mathrm{Hyper}}(\bm{y}) := \int_{\Omega}W(\nabla\bm{y},\mathrm{cof}\nabla\bm{y},\det\nabla\bm{y})\mathrm{d}\bm{x},
\end{equation}
where $\nabla\bm{y}$ is the Jacobian matrix of the transformation $\bm{y}$ with respect to $\bm{x}$, $\mathrm{cof}\nabla\bm{y}$ is the cofactor matrix of $\nabla\bm{y}$, $\mathrm{det}\nabla\bm{y}$ is the determinant of $\nabla\bm{y}$ and $W:\mathbb{R}^{3,3}\times\mathbb{R}^{3,3}\times\mathbb{R}\rightarrow\mathbb{R}$ is supposed to be convex. It is well known that $\nabla\bm{y}$, $\mathrm{cof}\nabla\bm{y}$ and $\mathrm{det}\nabla\bm{y}$ have the relationships with the change of the length, area and volume of the transformation $\bm{y}$, respectively. Hence, the hyperelastic regularizer can control the change of the length, area and volume of the transformation $\bm{y}$. Here, it also assumes that $\mathcal{R}^{\mathrm{Hyper}}(\bm{y})$ penalizes volume shrinkage, i.e., $W(A,S,V)\overset{V\rightarrow0}{\longrightarrow}\infty$. This will enable us to successfully control singularity sets and further, it can lead to a diffeomorphic transformation \cite{burger2013hyperelastic,droske2004variational}.

Specifically, in \cite{burger2013hyperelastic}, $W(\nabla\bm{y},\mathrm{cof}\nabla\bm{y},\det\nabla\bm{y})$ is defined as the following formulation:
\begin{equation}\label{hyperM}
W(\nabla\bm{y},\mathrm{cof}\nabla\bm{y},\det\nabla\bm{y}):=\alpha_{l}\phi_{l}(\nabla\bm{y})+\alpha_{s}\phi_{w,c}(\mathrm{cof}\nabla\bm{y})+\alpha_{v}\phi_{v}(\det\nabla\bm{y}),
\end{equation}
where $\phi_{l}(X)=\|X-I_{d}\|_{\mathrm{Fro}}^{2}/2$, $\phi_{w}(X)=(\|X\|_{\mathrm{Fro}}^{2}-3)^{2}/2$, $\phi_{c}(X)=\max\{\|X\|_{\mathrm{Fro}}^{2}-3,0\}^{2}/2$, $\phi_{v}(x)=((x-1)^{2}/x)^{2}$, $\|\cdot\|_{\mathrm{Fro}}$ represents the Frobenius norm and $I_{d}$ is the identity mapping. Here, $\phi_{c}$ is the convex envelop of $\phi_{w}$ and theoretically superior. But $\phi_{c}$ does not penalize surface shrinkage, so the double well $\phi_{w}$ is practically superior \cite{burger2013hyperelastic}. We also note that $\phi_{v}(x) = \phi_{v}(1/x)$ which means that shrinkage and growth have the same price.

In addition, we have the following existence theorem.
\begin{theorem}[Theorem 1 in \cite{burger2013hyperelastic}]\label{TheoremHR}
Given images $R,T\in C(\mathbb{R}^{3},\mathbb{R})$, compactly supported in $\Omega$, a polyconvex distance measure $\mathcal{D}(T\circ\bm{y},R):=\mathcal{D}(T,R;\bm{y}, \nabla\bm{y},\det\nabla\bm{y})$ with $\mathcal{D}\geq 0$, $\mathcal{R}^{\mathrm{Hyper}}$ as in \eqref{hyperM},
\begin{equation}
\mathcal{A}:=\{\bm{y}\in\mathcal{A}_{0}:\left|\int_{\Omega}\bm{y}(\bm{x})\mathrm{d}\bm{x}\right|\leq|\Omega|(M+\mathrm{diam}(\Omega))\}
\end{equation}
and
\begin{equation}
\mathcal{A}_{0}:=\{\bm{y}\in W^{1,2}(\Omega,\mathbb{R}^{3}):\mathrm{cof}\nabla\bm{y}\in L^{4}(\Omega,\mathbb{R}^{3,3}),\det\nabla\bm{y}\in L^{2}(\Omega,\mathrm{R}),\det\nabla\bm{y}>0 \ a.e.\},
\end{equation}
\end{theorem}
we assume that the registration functional $\mathcal{J}(\bm{y})$ \eqref{RegistrationFramework} satisfies $\mathcal{J}(I_{d})<\infty$. Then there exists at least one minimizer $\bm{y}^{*}\in\mathcal{A}$ of the functional $\mathcal{J}(\bm{y})$ \eqref{RegistrationFramework}. Here, $M\in\mathbb{R}$ is a constant and $|\Omega|$ and $\mathrm{diam}(\Omega)$ represent the volume and the diameter of the region $\Omega$, respectively.

\subsection{Registration-Based Segmentation}
The registration-based segmentation model is to partition the given image with the help of image registration. Specifically, we take the target image $I$ as the template and artificially construct a reference $J$ involving the prior information. Then by finding the suitable transformation $\bm{y}$ to match $I(\bm{y})$ and $J$, we can determine the boundary of the target object in the image $I$ by computing $\bm{y}(\hat{\bm{x}})$, where $\hat{\bm{x}}$ is on the boundary of the prior object in the image $J$. 

Next, we review a registration-based segmentation model using the Beltrami representation proposed in \cite{chan2018topology}. Before reviewing this work, we briefly recall the quasi-conformal theory. 

A mapping $f:\Omega\subset\mathbb{C}$ is quasi-conformal if it satisfies the following Beltrami equation in the distribution sense:
\begin{equation}
\frac{\partial f}{\partial\bar{z}}(z)=\mu(f)\frac{\partial f}{\partial z}(z),
\end{equation}
for some complex-valued Lebesgue measurable $\mu$ \cite{bers1977quasiconformal} satisfying $\|\mu\|<1$. Here, $\mu$ is called the Beltrami coefficient. Hence, a quasi-conformal mapping is an orientation-preserving homeomorphism and its first-order approximation takes small circles to small ellipses of bounded eccentricity \cite{gardiner2000quasiconformal}. The following theorem builds a link between a mapping $f$ and the Beltrami coefficient $\mu$.
\begin{theorem}[Measurable Riemann Mapping Theorem \cite{gardiner2000quasiconformal}]
Suppose $\mu:\mathbb{C}\rightarrow\mathbb{C}$ is Lebesgue measurable satisfying $\|\mu\|<1$, then there exists a quasi-conformal mapping $f:\mathbb{C}\rightarrow\mathbb{C}$ in the Sobolev space $W^{1,2}$ that satisfies the Beltrami equation in the distribution sense. Furthermore, assuming that the mapping is stationary at $0,1$ and $\infty$, then the associated quasi-conformal mapping $f$ is uniquely determined. 
\end{theorem}

Let $\Omega\subset\mathbb{C}$ be an image domain and $I:\Omega\rightarrow\mathbb{R}$ be an image including an object $D\subset\Omega$. Let $J:\Omega\rightarrow\mathbb{R}$ be an image of $\hat{D}$, called the topological prior image of $I$, where $\hat{D}$ is a simple object, possessing the same topological structure as $D$. Then the topological prior image $J$ can be defined by:
\begin{equation}
J(z)=
\left\{
\begin{split}
&c_{1}, \quad \mathrm{if}\ z\in\hat{D}, \\
&c_{2}, \quad \mathrm{if}\ z\in\Omega\setminus\hat{D}. 
\end{split}\right.
\end{equation}
Hence, the following variational model is proposed in \cite{chan2018topology}:
\begin{equation}\label{TPB}
\min_{c_{1},c_{2},\mu} \mathcal{E}(c_{1},c_{2},\mu) := \int_{\Omega}|\mu|^{2}+\eta\int_{\Omega}(I(f^{\mu})-J)^{2}+\lambda\int_{\Omega}|\nabla\mu|^{2},
\end{equation}
subject to the Dirichlet boundary condition and point-wise norm constraint
\begin{equation}
\mu=0 \ \mathrm{on} \ \partial\Omega, \quad \|\mu\|_{\infty}<1 \ \mathrm{in} \ \Omega,
\end{equation}
where $f^{\mu}$ denotes the mapping $f$ associated with the Beltrami coefficient $\mu$ and $\eta$ and $\lambda$ are weighting parameters. Since the model \eqref{TPB} can lead to a mapping $f$ with $\|\mu\|_{\infty}<1$, the resulting mapping $f$ is bijective. As mentioned before, we can find the boundary of the target object by computing $f(\hat{C})$ where $\hat{C}$ is the the boundary of the prior object in the image $J$. Further since $f$ is bijective, then the topological structure of $\hat{C}$ and $f(\hat{C})$ must be same. In this sense, the resulting segmentation is topology-preserving.

In order to solve \eqref{TPB}, an alternating method is employed. For finding $c_{1}$ and $c_{2}$ with $\mu$ fixed, it is similar to the Chan-Vese model:
\begin{equation}
c_{1}=\frac{\int_{\hat{D}}I(f^{\mu})\mathrm{d}z}{\int_{\hat{D}}\mathrm{d}z}
\end{equation}
and
\begin{equation}
c_{2}=\frac{\int_{\Omega\setminus\hat{D}}I(f^{\mu})\mathrm{d}z}{\int_{\Omega\setminus\hat{D}}\mathrm{d}z}.
\end{equation}
For finding $\mu$ with $c_{1}$ and $c_{2}$ fixed, a splitting method is chosen by introducing an auxiliary variable $\nu$ to solve the following variational problem:
\begin{equation}\label{AuxiliaryProblem}
\min_{\mu,\nu} \int_{\Omega}|\nu|^{2}+\eta\int_{\Omega}(I(f^{\mu})-J)^{2}+\lambda\int_{\Omega}|\nabla\nu|^{2} + \sigma\int_{\Omega}|\mu-\nu|^{2}.
\end{equation} 
The alternating method is used again to solve \eqref{AuxiliaryProblem}.

However, since the Beltrami coefficient is only defined in complex space, this model cannot be directly extended to the 3D topology-preserving image segmentation. 

\section{Proposed Model}\label{SProposedModel}
Let us recall the model \eqref{TPB}: it has two parts, fitting term and regularization term. Compared with the classical registration model \eqref{RegistrationFramework}, the fitting term in \eqref{TPB} is borrowed from the Chan-Vese model. Hence, the registration-based segmentation model can be considered as a combination of the Chan-Vese segmentation and image registration. 

For the fitting term in \eqref{TPB}, it can easily be extended to 3D case. However, since the Beltrami coefficient is only defined in complex space, the regularization term in \eqref{TPB} cannot be directly extended to 3D case. In order to overcome this difficulty, our idea is very simple: we replace the Beltrami regularizer with another regularizer, which can deal with 3D image registration and lead to a bijective transformation. Hence, our proposed model is a 3D topology-preserving registration-based segmentation model.

Now, we present our proposed model. For the fitting term, we extend the prior image containing multiple objects rather than only one target object. Hence, our prior image can be defined as follows:
\begin{equation}\label{JX}
J(\bm{x})=
\left\{
\begin{split}
&c_{1}, \quad \mathrm{if}\ \bm{x}\in\Omega_{1}, \\
&\cdots \\
&c_{m}, \quad \mathrm{if}\ \bm{x}\in\Omega_{m},
\end{split}\right.
\end{equation}
where $m$ is prescribed by the user, $\Omega = \Omega_{1}\cup...\cup\Omega_{m}$ and $\Omega_{i}\cap\Omega_{j}=\emptyset$, for $1\leq i\neq j\leq m$. By introducing the indicator function $\mathcal{X}_{\Omega}(\bm{x})$, we can convert \eqref{JX} into the following formulation:
\begin{equation}
J(\bm{x})=\sum_{l=1}^{m}c_{l}\mathcal{X}_{\Omega_{l}}(\bm{x}).
\end{equation}

 For the regularizer, we choose the hyperelastic regularizer \eqref{hyperM} reviewed in Section \ref{SRelatedWork}. Then our proposed model can be formulated in the following:
\begin{equation}\label{ProposedModel}
\min_{\bm{y},\bm{c}=(c_{1},...,c_{m})} \mathcal{F}(\bm{y},\bm{c}):= \frac{1}{2}\int_{\Omega}(I(\bm{y})-\sum_{l=1}^{m}c_{l}\mathcal{X}_{\Omega_{l}}(\bm{x}))^{2}\mathrm{d}\bm{x}+\mathcal{R}^{\mathrm{Hyper}}(\bm{y}).
\end{equation}

\begin{remark}
A similar fitting term has been used in the multi-modality image registration \cite{heldmann2010multimodal}. In \cite{heldmann2010multimodal}, a novel least-square distance measure was proposed:
\begin{equation}\label{MFT}
\frac{1}{2}\int_{\Omega}(T(\bm{x})-w(R(\bm{x})))^{2}\mathrm{d}\bm{x},
\end{equation}
where $w$ is a function from $\mathbb{R}\rightarrow \mathbb{R}$. If we further let $R(\bm{x})$ be constant on $\Omega_{l}$, then \eqref{MFT} is equivalent to the following formulation:
\begin{equation}\label{EMFT}
\frac{1}{2}\int_{\Omega}(T(\bm{x})-w(R(\bm{x})))^{2}\mathrm{d}\bm{x} = \frac{1}{2}\sum_{l=1}^{m}\int_{\Omega_{l}}(T(\bm{x})-w_{l})^{2}\mathrm{d}\bm{x},
\end{equation}
where $w_{l} = w(R(\bm{x}))$ for $\bm{x}$ in $\Omega_{l}$. Here, we want to point out that in \eqref{EMFT}, $\Omega_{l}$ is determined by the intensity value of $R(\bm{x})$ and for each $\Omega_{l}$, the intensity values of $R(\bm{x})$ are different. But for \eqref{ProposedModel}, $\Omega_{l}$ is determined by the prescribed region and more importantly, the corresponding $c_{l}$ for different $\Omega_{l}$ can be same. 
\end{remark}

Next, we investigate the existence of the solution of the proposed model \eqref{ProposedModel}. 
Comparing the registration functional $\mathcal{J}(\bm{y})$ \eqref{RegistrationFramework} with the proposed model $\mathcal{F}(\bm{y},\bm{c})$ \eqref{ProposedModel}, we can find that the only difference comes from the fitting term. Hence, we can prove the existence of the solution of the proposed model \eqref{ProposedModel} based on the Theorem \ref{TheoremHR} with a slight modification.

Here, we just follow the direct method in the calculus of variations. Usually, the direct method in the calculus of variations can be divided into three steps:
\begin{enumerate}
\item Take a minimizing sequence;
\item Show that some subsequence converges to a point in the feasible space with respect to some topology;
\item Show that the functional is weakly lower semi-continuous with respect to this topology.
\end{enumerate}
The first step is trivial if the functional is bounded below. If the feasible set is reflexive and the functional is coercive, then the second step is satisfied. Since $\mathcal{A}$ is reflexive, naturally we define a product space $\mathcal{A}\times\mathbb{R}^{m}$ which is also reflexive as the feasible space of the proposed model \eqref{ProposedModel}. Then we can build the following two lemmas to make the second and third steps satisfied.
\begin{lemma}\label{coercivity}
If $I\in C(\mathbb{R}^{3},\mathbb{R})$ is compactly supported in $\Omega$, then the functional $\mathcal{F}$ in \eqref{ProposedModel} satisfies a coercivity, i.e., there exist constants $\beta>0$ and $\gamma\in\mathbb{R}$ such that for all $(\bm{y},\bm{c})\in \mathcal{A}\times\mathbb{R}^{m}$ it holds
\begin{equation}
\mathcal{F}(\bm{y},\bm{c})\geq \beta+\gamma(\|\bm{y}\|_{W^{1,2}}^{2}+\|\mathrm{cof}\nabla\bm{y}\|_{L^{4}}^{4}+\|\det\nabla\bm{y}\|_{L^{2}}^{2}+\|\bm{c}\|_{l^{2}}^{2}).
\end{equation}
\end{lemma}
\begin{proof}
Firstly, by Lemma 1 in Section 3.3 of \cite{ruthotto2012hyperelastic}, there exist constants $\beta_{1}>0$ and $\gamma_{1}\in\mathbb{R}$ such that for all $\bm{y}\in \mathcal{A}$ it holds
\begin{equation*}
\mathcal{R}^{\mathrm{Hyper}}(y)\geq \beta_{1}+\gamma_{1}(\|\bm{y}\|_{W^{1,2}}^{2}+\|\mathrm{cof}\nabla\bm{y}\|_{L^{4}}^{4}+\|\det\nabla\bm{y}\|_{L^{2}}^{2}).
\end{equation*}

Secondly, since $I\in C(\mathbb{R}^{3},\mathbb{R})$ is compactly supported in $\Omega$, we have constants $\beta_{2}>0$ and $\gamma_{2}\in\mathbb{R}$ such that for all $\bm{c}\in \mathbb{R}^{m}$ it holds
\begin{equation*}
 \frac{1}{2}\int_{\Omega}(I(\bm{y})-\sum_{l=1}^{m}c_{l}\mathcal{X}_{\Omega_{l}}(\bm{x}))^{2}\mathrm{d}\bm{x} = \frac{1}{2}\sum_{l=1}^{m}\int_{\Omega_{l}}(I(\bm{y})-c_{l})^{2}\mathrm{d}\bm{x}\geq \beta_{2}+\gamma_{2}\|\bm{c}\|_{l^{2}}^{2}.
\end{equation*}

Hence, set $\beta=\min\{\beta_{1},\beta_{2}\}$ and $\gamma=\min\{\gamma_{1},\gamma_{2}\}$. Then we have
\begin{equation*}
\mathcal{F}(\bm{y},\bm{c})\geq \beta+\gamma(\|\bm{y}\|_{W^{1,2}}^{2}+\|\mathrm{cof}\nabla\bm{y}\|_{L^{4}}^{4}+\|\det\nabla\bm{y}\|_{L^{2}}^{2}+\|\bm{c}\|_{l^{2}}^{2}).
\end{equation*}
and the proof is complete.
\end{proof}

\begin{lemma}\label{wlsc}
If $I\in C(\mathbb{R}^{3},\mathbb{R})$ is compactly supported in $\Omega$, for the functional $\mathcal{F}$ in \eqref{ProposedModel}, when $\bm{y}^{k}\rightharpoonup \bm{y}$ in $W^{1,2}$, $\mathrm{cof}\nabla\bm{y}^{k}\rightharpoonup H$ in $L^{4}$, $\det\nabla\bm{y}^{k}\rightharpoonup v$ in $L^{2}$ and $\bm{c}^{k}\rightarrow\bm{c}$ in $\mathbb{R}^{m}$, the following inequality holds:
\begin{equation}
\lim_{k\rightarrow\infty}\inf\int_{\Omega}\mathcal{F}(\bm{x},\bm{y}^{k},\nabla\bm{y}^{k},\mathrm{cof}\nabla\bm{y}^{k},\det\nabla\bm{y}^{k}, \bm{c}^{k})\mathrm{d}\bm{x}\geq\int_{\Omega}\mathcal{F}(\bm{x},\bm{y},\nabla\bm{y}, H,v, \bm{c})\mathrm{d}\bm{x}
\end{equation}
\end{lemma}
\begin{proof}
Firstly, by Lemma 2 in Section 3.3 of \cite{ruthotto2012hyperelastic}, we have 
\begin{equation}\label{wlsc_1}
\lim_{k\rightarrow\infty}\inf\int_{\Omega}\mathcal{R}^{\mathrm{Hyper}}(\bm{x},\bm{y}^{k},\nabla\bm{y}^{k},\mathrm{cof}\nabla\bm{y}^{k},\det\nabla\bm{y}^{k})\mathrm{d}\bm{x}\geq\int_{\Omega}\mathcal{R}^{\mathrm{Hyper}}(\bm{x},\bm{y},\nabla\bm{y}, H,v)\mathrm{d}\bm{x}.
\end{equation}

Secondly, due to the compact embedding of $W^{1,2}\subset L^{2}$, $\bm{y}^{k}\rightharpoonup \bm{y}$ in $W^{1,2}$ implies $\bm{y}^{k}\rightarrow \bm{y}$ in $L^{2}$. In addition, since $I$ is continuous, then we have 
\begin{equation}\label{wlsc_2}
\begin{split}
\lim_{k\rightarrow\infty}\inf \frac{1}{2}\int_{\Omega}(I(\bm{y}^{k})-\sum_{l=1}^{m}c^{k}_{l}\mathcal{X}_{\Omega_{l}}(\bm{x}))^{2}\mathrm{d}\bm{x} &=\lim_{k\rightarrow\infty} \frac{1}{2}\int_{\Omega}(I(\bm{y}^{k})-\sum_{l=1}^{m}c^{k}_{l}\mathcal{X}_{\Omega_{l}}(\bm{x}))^{2}\mathrm{d}\bm{x} \\
&=  \frac{1}{2}\int_{\Omega}(I(\bm{y})-\sum_{l=1}^{m}c_{l}\mathcal{X}_{\Omega_{l}}(\bm{x}))^{2}\mathrm{d}\bm{x}.
\end{split}
\end{equation}

Then combining \eqref{wlsc_1} and \eqref{wlsc_2}, the proof is complete.
\end{proof}

Now, we are in position to give the existence of the solution of the proposed model \eqref{ProposedModel}.
\begin{theorem}\label{ExistenceofProposedModel}
If $I\in C(\mathbb{R}^{3},\mathbb{R})$ is compactly supported in $\Omega$, then for the functional $\mathcal{F}$ in \eqref{ProposedModel}, there exists at least one minimizer $(\bm{y}^{*},\bm{c}^{*})\in \mathcal{A}\times\mathbb{R}^{m}$.
\end{theorem}
\begin{proof}
Since $\mathcal{F}(\bm{x},\bm{0})$ is finite and $\mathcal{F}$ is nonnegative, there exists a minimizing sequence $\{(\bm{y}^{k},\bm{c}^{k})\}_{k\in\mathbb{N}}$ such that
\begin{equation*}
\lim_{k\rightarrow\infty}\mathcal{F}(\bm{y}^{k},\bm{c}^{k}) = \inf_{(\bm{y},\bm{c})\in \mathcal{A}\times\mathbb{R}^{m}} \mathcal{F}(\bm{y},\bm{c}).
\end{equation*}
In addition, we can assume that $\{\mathcal{F}(\bm{y}^{k},\bm{c}^{k})\}_{k\in\mathcal{N}}$ is bounded by a constant $\rho>0$.

By Lemma \ref{coercivity}, we have that the sequence $\{(\bm{y}^{k},\mathrm{cof}\nabla\bm{y}^{k},\det\nabla\bm{y}^{k})\}_{k\in\mathrm{N}}$ is bounded in the Banach space $\Psi = W^{1,2}\times L^{4}\times L^{2}$ and the sequence $\{\bm{c}^{k}\}_{k\in\mathbb{N}}$ is bounded in $\mathbb{R}^{m}$. Hence, there exists a subsequence $\{(\bm{y}^{k_{j}},\mathrm{cof}\nabla\bm{y}^{k_{j}},\det\nabla\bm{y}^{k_{j}}, \bm{c}^{k_{j}})\}_{j\in\mathrm{N}}$, such that $\bm{y}^{k_{j}}\rightharpoonup \bm{y}^{*}$ in $W^{1,2}$, $\mathrm{cof}\nabla\bm{y}^{k_{j}}\rightharpoonup H$ in $L^{4}$, $\det\nabla\bm{y}^{k_{j}}\rightharpoonup v$ in $L^{2}$ and $\bm{c}^{k_{j}}\rightarrow\bm{c}^{*}$ in $\mathbb{R}^{m}$. By Theorem 4 in Section 3.3 of \cite{ruthotto2012hyperelastic}, we have $H=\mathrm{cof}\nabla\bm{y}^{*}$ and $v=\det\nabla\bm{y}^{*}$.  
Hence, according to Lemma \ref{wlsc}, we have 
\begin{equation*}
\lim_{k\rightarrow\infty}\inf\int_{\Omega}\mathcal{F}(\bm{x},\bm{y}^{k},\nabla\bm{y}^{k},\mathrm{cof}\nabla\bm{y}^{k},\det\nabla\bm{y}^{k}, \bm{c}^{k})\mathrm{d}\bm{x}\geq\int_{\Omega}\mathcal{F}(\bm{x},\bm{y}^{*},\nabla\bm{y}^{*}, \mathrm{cof}\nabla\bm{y}^{*},\det\nabla\bm{y}^{*}, \bm{c}^{*})\mathrm{d}\bm{x}.
\end{equation*}

In addition, we also need to check $\nabla\bm{y}^{*}>0$ almost everywhere in $\Omega$. Fix $\epsilon\in[0,1)$ and consider the set $\Omega_{\epsilon} = \{\bm{x}\in\Omega|\det\nabla\bm{y}^{*}(\bm{x})<\epsilon\}$. We thus have
\begin{equation*}
\begin{split}
\alpha_{v}\phi_{v}(\epsilon)|\Omega_{\epsilon}| &= \int_{\Omega_{\epsilon}}\alpha_{v}\phi_{v}(\epsilon)\mathrm{d}\bm{x}\\
&\leq \int_{\Omega_{\epsilon}}\alpha_{v}\phi_{v}(\det\nabla\bm{y}^{*})\mathrm{d}\bm{x}\\
&\leq \int_{\Omega_{\epsilon}}\mathcal{F}(\bm{x},\bm{y}^{*},\nabla\bm{y}^{*}, \mathrm{cof}\nabla\bm{y}^{*},\det\nabla\bm{y}^{*}, \bm{c}^{*})\mathrm{d}\bm{x} \\
&\leq \lim_{k\rightarrow\infty}\inf\int_{\Omega_{\epsilon}}\mathcal{F}(\bm{x},\bm{y}^{k},\nabla\bm{y}^{k},\mathrm{cof}\nabla\bm{y}^{k},\det\nabla\bm{y}^{k}, \bm{c}^{k})\mathrm{d}\bm{x} \\
& \leq \rho.
\end{split}
\end{equation*}
Due to $\phi_{v}(\epsilon)\rightarrow\infty$ as $\epsilon\rightarrow 0^{+}$, we have $|\Omega_{\epsilon=0}|$ is zero. So $\nabla\bm{y}^{*}>0$ almost everywhere in $\Omega$ and $\bm{y}^{*}$ actually lies in $\mathcal{A}$. 

Hence, $(\bm{y}^{*},\bm{c}^{*})\in \mathcal{A}\times\mathbb{R}^{m}$ is a minimizer of $\mathcal{F}$ in \eqref{ProposedModel}, which follows from
\begin{equation}
\inf_{(\bm{y},\bm{c})\in \mathcal{A}\times\mathbb{R}^{m}} \mathcal{F}(\bm{y},\bm{c}) = \lim_{k\rightarrow\infty}\mathcal{F}(\bm{y}^{k},\bm{c}^{k}) = \lim_{j\rightarrow\infty}\mathcal{F}(\bm{y}^{k_{j}},\bm{c}^{k_{j}}) \geq \mathcal{F}(\bm{y}^{*},\bm{c}^{*})\geq \inf_{(\bm{y},\bm{c})\in \mathcal{A}\times\mathbb{R}^{m}} \mathcal{F}(\bm{y},\bm{c}) 
\end{equation}
and completes the proof.
\end{proof}

Here, we can note that Theorem \ref{ExistenceofProposedModel} ensures that the proposed model \eqref{ProposedModel} can generate a bijective transformation $\bm{y}$, namely, the proposed model \eqref{ProposedModel} is indeed topology-preserving.

\section{Numerical Implementation}\label{SNumericalImplementation}
In this section, we show the details about how to solve the proposed model \eqref{ProposedModel}. Here, we choose the first-discretize-then-optimize method. The main idea of this method is that: directly discretize the variational model \eqref{ProposedModel} by a proper discretization scheme to derive an unconstrained finite dimensional optimization problem and then choose a suitable optimization algorithm to solve the resulting unconstrained finite dimensional optimization.

\subsection{Discretization}
For simplicity, we discretize our proposed model \eqref{ProposedModel} on the spatial domain $\Omega = [0,1]^{3}$. In the implementation, we employ the nodal grid and define a spatial partition $\Omega_{h}^{n} = \{\bm{x}^{i,j,k}\in\Omega | \bm{x}^{i,j,k}=(x_{1}^{i},x_{2}^{j},x_{3}^{k})=(ih,jh,kh), 0 \leq i \leq n , 0 \leq j \leq n, 0 \leq k \leq n\}$, where $h = \frac{1}{n}$. Similarly, the spatial partition of $\Omega_{l}, 1\leq l \leq m$ is defined as ${\Omega_{l}}_{h}^{n} = \{\bm{x}^{i,j,k}\in\Omega_{l} | \bm{x}^{i,j,k}=(x_{1}^{i},x_{2}^{j},x_{3}^{k})=(ih,jh,kh), 0 \leq i \leq n , 0 \leq j \leq n, 0 \leq k \leq n \}$. We discretize the transformation $\bm{y}$ on the nodal grid, namely $\bm{y}^{i,j,k} = (y_{1}^{i,j,k},y_{2}^{i,j,k},y_{3}^{i,j,k}) = (y_{1}(x_{1}^{i},x_{2}^{j},x_{3}^{k}), y_{2}(x_{1}^{i},x_{2}^{j},x_{3}^{k}),y_{3}(x_{1}^{i},x_{2}^{j},x_{3}^{k}))$. In order to simplify the presentation, according to the lexicographical ordering, we reshape
\begin{displaymath}
X = (x_{1}^{0},...,x_{1}^{n},x_{2}^{0},...,x_{2}^{n},x_{3}^{0},...,x_{3}^{n})^{T} \in \mathbb{R}^{3(n+1)^{3}\times 1},
\end{displaymath}
\begin{displaymath}
Y = (y_{1}^{0,0,0},...,y_{1}^{n,n,n}, y_{2}^{0,0,0},...,y_{2}^{n,n,n},y_{3}^{0,0,0},...,y_{3}^{n,n,n})^{T} \in \mathbb{R}^{3(n+1)^{3}\times 1}
\end{displaymath}
and 
\begin{displaymath}
C = (c_{1},...,c_{m})^{T} \in \mathbb{R}^{m\times 1}.
\end{displaymath}

\subsubsection{Discretization of Fitting Term in \eqref{ProposedModel}}\label{SDF}
Here, we assume that the intensity values of the discretized image are defined on the cell-centered grid.  Hence, we first give an averaging matrix $P$ from the nodal grid $Y$ to the cell-centered grid $PY$ \cite{haber2004numerical,haber2007image}. Then for the deformed template image $I(\bm{y})$, we can set $\vec I(PY) \in \mathbb{R}^{n^{3}\times 1}$ as the discretized deformed template image. To discretize the prior image $J(\bm{x})=\sum_{l=1}^{m}c_{l}\mathcal{X}_{\Omega_{l}}(\bm{x})$, we define a matrix $M\in\mathbb{R}^{n^{3}\times m}$, where $M_{i,j}$ is $1$ if the intensity value of the $i$-th voxel of the discretized prior image is $c_{j}$ otherwise $M_{i,j}$ is $0$. Then we have $MC$ as the discretized prior image.

Consequently, for the fitting term, we obtain the following discretization:
\begin{equation}\label{disSSD}
\frac{1}{2}\int_{\Omega}(I(\bm{y})-J)^{2}\mathrm{d}\bm{x} \approx \frac{h^{3}}{2}(\vec I(PY)-MC)^{T}(\vec I(PY)-MC).
\end{equation}

\subsubsection{Discretization of Regularizer in \eqref{ProposedModel}}

For the length part of the regularizer \eqref{hyperM}, by using the forward difference, we have the following approximation:
\begin{equation}\label{disR1}
\int_{\Omega} \alpha_{l}\phi_{l}(\nabla\bm{y})\mathrm{d}\bm{x}\approx
\frac{\alpha_{l} h^{3}}{2}(Y-X)^{T}A^{T}A(Y-X),
\end{equation}
where $A$ is shown in \textbf{Appendix} \ref{A}.

\begin{figure}[!ht]
\centering
\includegraphics[width=2.0in,height=2.0in]{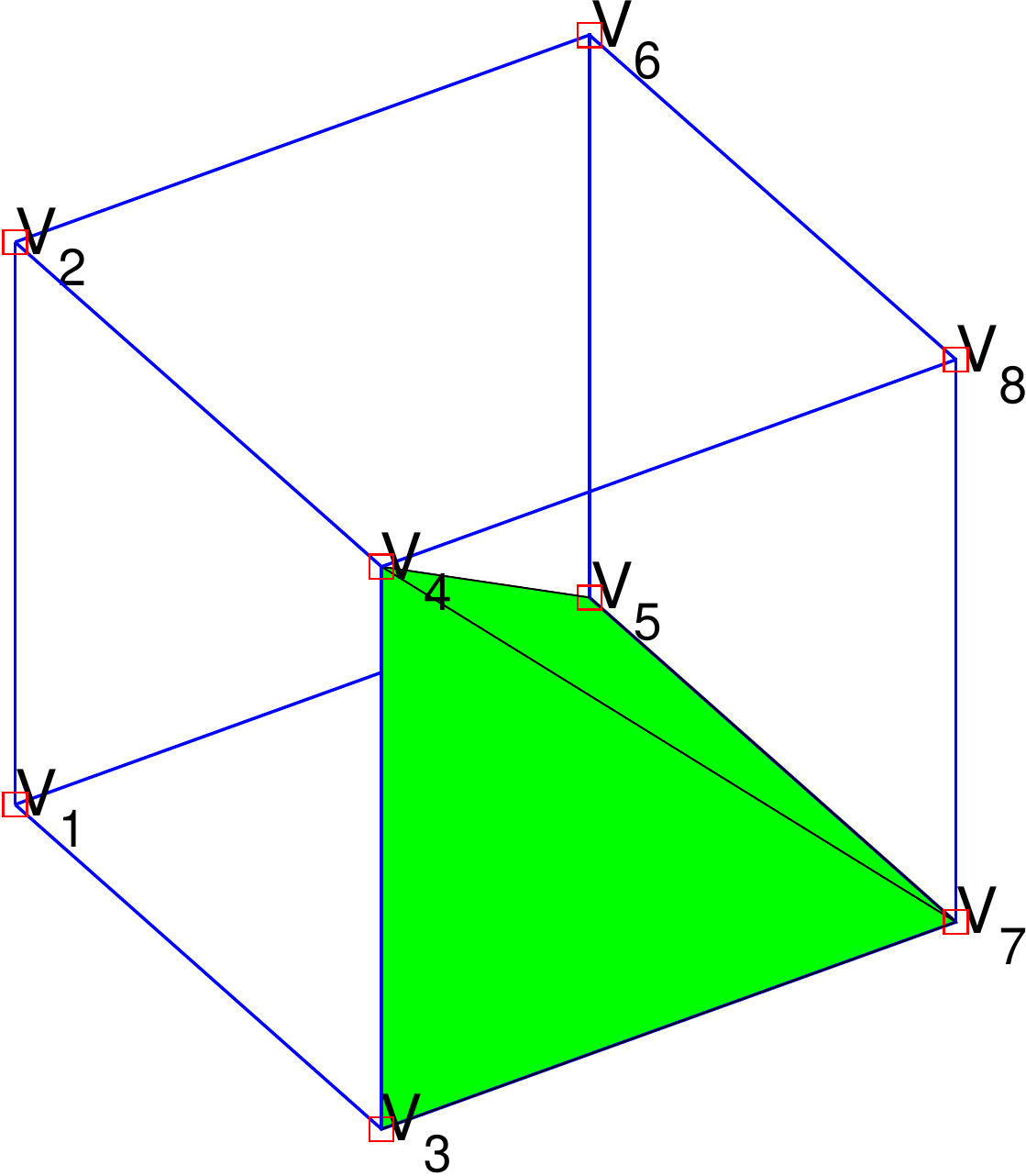}
\caption{Partition of a voxel. $V_{1},...,V_{8}$ are vertices.}\label{partition}
\end{figure}

To discretize the surface term and volume term in \eqref{hyperM}, \cite{burger2013hyperelastic} has ensured the regularity of various partitions and in \cite{burger2013hyperelastic}, each voxel divided into 24 tetrahedrons is employed. Although this division possesses the symmetry, its disadvantage is the computational costs. In addition, in order to take the multilevel strategy, an interpolation operator from coarse level to fine level is indispensable. However, the usually used bilinear interpolation is not consistent for the Jacobian determinant in the coarse and fine levels, which means that the discretized Jacobian determinant is positive in the coarse level but after interpolation, it may be negative in the fine level. Hence, in the implementation, we take the standard finite element division to divide each voxel into 6 tetrahedrons (Figure \ref{partition}). The computational cost is less and much importantly, in the multilevel strategy, we can use nodal interpolation to make the Jacobian determinant in the coarse and fine levels consistent.

In each tetrahedron, we use three linear interpolation functions to approximate $y_{1}$, $y_{2}$ and $y_{3}$. Hence, for the surface term and volume term in \eqref{hyperM}, we have the following approximation:
\begin{equation}\label{disR2}
\int_{\Omega} \alpha_{s}\phi_{w}(\mathrm{cof}\nabla\bm{y})+\alpha_{v}\phi_{v}(\det\nabla\bm{y})\mathrm{d}\bm{x}\approx \frac{h^{3}}{6} (\alpha_{s}\bm{\phi}_{w}(\bm{s}(Y))+\alpha_{v}\bm{\phi}_{v}(\bm{v}(Y)))^{T}e,
\end{equation}
where $\bm{s}(Y)$ and $\bm{v}(Y)$ are shown in \textbf{Appendix} \ref{S_V}, $\bm{\phi}_{w}(\bm{s}(Y))$ is a vector function whose $i$th component is $\phi_{w}(\bm{s}(Y)_{i})$, $\bm{\phi}_{v}(\bm{v}(Y))$ is a vector function whose $i$th component is $\phi_{v}(\bm{v}(Y)_{i})$ and $e$ is a vector whose all components are all equal to $1$.

Combining  \eqref{disSSD}, \eqref{disR1} and \eqref{disR2}, for the proposed model \eqref{ProposedModel}, we get its corresponding finite dimensional optimization problem:
\begin{equation}\label{DProposedModel}
\begin{split}
\min_{Y,C} F(Y,C):= \frac{h^{3}}{2}(\vec I(PY)-MC)^{T}(\vec I(PY)-MC)&+\frac{\alpha_{l} h^{3}}{2}(Y-X)^{T}A^{T}A(Y-X)
\\
&+ \frac{h^{3}}{6} (\alpha_{s}\bm{\phi}_{w}(\bm{s}(Y))+\alpha_{v}\bm{\phi}_{v}(\bm{v}(Y)))^{T}e.
\end{split}
\end{equation}

\begin{remark}
(i) Since $PY$ usually does not correspond to voxel points, interpolation operator is necessary at all steps. Here we choose cubic-spline interpolation \cite{modersitzki2009fair} to compute
$\vec I(P{Y})$. Linear interpolation cannot be applied because it is not differentiable at grid points. If $\bm{y}(\bm{x})$ is out of $\Omega$, the intensity value of $T(\bm{y}(\bm{x}))$ is set to be $0$.
(ii) For the boundary conditions, we consider both the Dirichlet boundary conditions and the natural boundary conditions.
\end{remark}

\subsection{Optimization Method}
In this part, we show the details about the optimization method to solve the resulting finite optimization problem \eqref{DProposedModel}. Different from the Chan-Vese model \eqref{CV} and Beltrami representation based model \eqref{TPB}, we consider the variables $Y$ and $C$ as a whole part. Hence, we do not need to employ the alternating direction method.

Our aim is to generate a sequence $\{(Y^{k},C^{k})|\bm{v}(Y^{k})>0\}_{k\in\mathbb{N}}$ converging to a point $(Y,C)$ that satisfies $\bm{v}(Y)>0$, where $\bm{v}(Y)$ in \textbf{Appendix} \ref{S_V} is the discretized Jacobian determinant and $\bm{v}(Y)>0$ indicates that all the components of $\bm{v}(Y)$ are larger than $0$. 

Here, we choose the line search method and its iterative scheme is as follows:
\begin{equation}\label{update}
\begin{pmatrix}
Y^{k+1}\\
C^{k+1}\\
\end{pmatrix} = \begin{pmatrix}
Y^{k}\\
C^{k}\\
\end{pmatrix} + \eta^{k} p^{k},
\end{equation}
where $\eta^{k}$ is the step length generated by a line search strategy and $p^{k}$ is the search direction. We first discuss how to compute the search direction $p^{k}$ and then discuss how to define the step length $\eta^{k}$.

\subsubsection{Search Direction $p$}
In the implementation, the search direction $p$ is generated by solving the generalized Gauss-Newton system:
\begin{equation}\label{GaussnNewtonsystem}
\hat{H}p = -d,
\end{equation}
where $d$ and $\hat{H}$ are the gradient and the approximated Hessian of \eqref{DProposedModel}, respectively. 

Before computing the approximated Hessian $\hat{H}$, we first briefly review the Gauss-Newton method and the generalized Gauss-Newton method. 

Consider the following least-square problem:
\begin{equation}
\min_{\bm{x}} f_{1}(\bm{x}):=\frac{1}{2}\|\bm{r}(\bm{x})\|^{2},
\end{equation}
where $\bm{r}(\bm{x}) = (r_{1}(\bm{x}),...,r_{m}(\bm{x}))^{T}: \mathbb{R}^{n}\rightarrow\mathbb{R}^{m}$ is a residual vector function and each $r_{j}$ for $1\leq j\leq m$ is a smooth function from $\mathbb{R}^{n}\rightarrow\mathbb{R}$. Then the gradient and Hessian of $f_{1}(\bm{x})$ are as follows, respectively:
\begin{equation}\label{GHF}
\begin{split}
\nabla f_{1}(\bm{x}) &= J_{\bm{r}}(\bm{x})^{T}\bm{r}(\bm{x}), \\
\nabla^{2}f_{1}(\bm{x}) & = J_{\bm{r}}(\bm{x})^{T}J_{\bm{r}}(\bm{x})+\sum_{j=1}^{m}r_{j}(\bm{x})\nabla^{2}r_{j}(\bm{x}),\\
J_{\bm{r}}(\bm{x}) & = \begin{bmatrix} \nabla r_{1}(\bm{x}) \cdots \nabla r_{m}(\bm{x}) \end{bmatrix}^{T}.
\end{split}
\end{equation}
The Gauss-Newton method is to solve the Gauss-Newton system
\begin{equation}\label{GNS}
J_{\bm{r}}^{T}J_{\bm{r}}p = -\nabla f_{1}
\end{equation}
to obtain the search direction. From \eqref{GHF}, we can see that the Gauss-Newton method is a modified Newton's method. Here, if $J_{\bm{r}}$ is full rank, \eqref{GNS} will lead to a descent direction. In addition, the Gauss-Newton system only involves the first order information and omit the second order information which can save the computational cost \cite{nocedal2006numerical}.  

Next, consider a general minimization problem:
\begin{equation}
\min_{\bm{x}} f_{2}(\bm{x}) := g(\bm{h}(\bm{x})),
\end{equation}
where $g: \mathbb{R}^{m}\rightarrow\mathbb{R}$ is a smooth function and $\bm{h} = (h_{1},...,h_{m})^{T}:\mathbb{R}^{n}\rightarrow\mathbb{R}^{m}$ is a smooth vector function. Then the gradient and Hessian of $f_{2}(\bm{x})$ are as follows, respectively:
\begin{equation}
\begin{split}
\nabla f_{2}(\bm{x}) &= J_{\bm{h}}(\bm{x})^{T}\nabla g(\bm{h}(\bm{x})), \\
\nabla^{2}f_{2}(\bm{x}) & = J_{\bm{h}}(\bm{x})^{T}\nabla^{2}g(\bm{h}(\bm{x}))J_{\bm{h}}(\bm{x})+\sum_{j=1}^{m}[\nabla g(\bm{h}(\bm{x}))]_{j}\nabla^{2}h_{j}(\bm{x}),\\
J_{\bm{h}}(\bm{x}) & = \begin{bmatrix} \nabla h_{1}(\bm{x}) \cdots \nabla h_{m}(\bm{x}) \end{bmatrix}^{T}.
\end{split}
\end{equation}
Directly following the Gauss-Newton method and omitting the second order term, we can get the generalized Gauss-Newton system \cite{diehl2019local}:
\begin{equation}\label{GGNS}
J_{\bm{h}}(\bm{x})^{T}\nabla^{2}g(\bm{h}(\bm{x}))J_{\bm{h}}(\bm{x})p = -\nabla f_{2}.
\end{equation}
Here, if $J_{\bm{h}}$ is full rank and $\nabla^{2}g(\bm{h}(\bm{x}))$ is symmetric positive definite, then the search direction $p$ derived by \eqref{GGNS} is a descent direction becasue the matrix of \eqref{GGNS} is symmetric positive definite.

Now, we return to the computation of the gradient $d$ and the approximated Hessian $\hat{H}$ of \eqref{DProposedModel}.
To compute the gradient $d$ easily, we introduce $\bm{s}_{1},...,\bm{s}_{9}$ and set 
\begin{equation}
\begin{split}
&\bm{s}_{1} = D_{5}Y\odot D_{9}Y-D_{6}Y\odot D_{8}Y,\ \bm{s}_{2} = D_{6}Y\odot D_{7}Y-D_{4}Y\odot D_{9}Y,\ \bm{s}_{3}=D_{4}Y\odot D_{8}Y-D_{5}Y\odot D_{7}Y, \\
&\bm{s}_{4} =D_{3}Y\odot D_{8}Y-D_{2}Y\odot D_{9}Y, \ \bm{s}_{5} = D_{1}Y\odot D_{9}Y-D_{3}Y\odot D_{7}Y,\ \bm{s}_{6} = D_{2}Y\odot D_{7}Y-D_{1}Y\odot D_{8}Y, \\
&\bm{s}_{7} =D_{2}Y\odot D_{6}Y-D_{3}Y\odot D_{5}Y,\ \bm{s}_{8} = D_{3}Y\odot D_{4}Y-D_{1}Y\odot D_{6}Y,\ \bm{s}_{9} = D_{1}Y\odot D_{5}Y-D_{2}Y\odot D_{4}Y.
\end{split}
\end{equation}
Here, $D_{l}, 1\leq l \leq9$ are defined in \textbf{Appendix} \ref{S_V} and $\odot$ indicates the Hadamard product. Recall the definition of $\bm{\phi}_{w}(\bm{s}(Y))$ in \eqref{disR2} and we have $\bm{\phi}_{w}(\bm{s}(Y)) = \frac{1}{2}(S-3)\odot(S-3)$, where $S = \sum_{i=1}^{9}\bm{s}_{i}\odot\bm{s}_{i}$. Then the gradient $d$ of \eqref{DProposedModel} is shown:
\begin{equation}
d= 
\begin{pmatrix}
h^{3}P^{T}\vec I_{PY}^{T}(\vec I(PY)-MC)+\alpha_{l}h^{3}A^{T}A(Y-X)+\frac{h^{3}}{6}(\alpha_{s}\mathrm{d}S^{T}(S-3)+\alpha_{v}\mathrm{d}\bm{v}^{T}\mathrm{d}\bm{\phi}_{v})  \\
-h^{3}M^{T}(\vec I(PY)-MC)
\end{pmatrix},
\end{equation}
where $\vec I_{PY}$ is the Jacobian of $\vec I$ with respect to $PY$, $\mathrm{d}S$ is the Jacobian of $S$ with respect to $Y$, $\mathrm{d}\bm{v}$ is the Jacobian of $\bm{v}$ with respect to $Y$ and $\mathrm{d}\bm{\phi}_{v}$ is a vector whose $i$th component is $\phi'_{v}(\bm{v}_{i})$.

Then following the idea of the generalized Gauss-Newton method, we omit the second order term and get the approximated Hessian $\hat{H}$ of \eqref{DProposedModel}:
\begin{equation}\label{AHessian}
\hat{H} = 
\begin{pmatrix}
h^{3}P^{T}\vec I_{PY}^{T}\vec I_{PY}P+\alpha_{l}h^{3}A^{T}A+\frac{h^{3}}{6}(\alpha_{s}\mathrm{d}S^{T}\mathrm{d}S+\alpha_{v}\mathrm{d}\bm{v}^{T}\mathrm{d}^{2}\bm{\phi}_{v}\mathrm{d}\bm{v}) + \gamma I &  -h^{3}P^{T}\vec I_{PY}^{T}M \\
-h^{3} M^{T}\vec I_{PY}P & h^{3}M^{T}M
\end{pmatrix},
\end{equation}
where $I$ is the identity matrix, $\mathrm{d}^{2}\bm{\phi}_{v}$ is a diagonal matrix and the $i$th component of the diagonal is $\phi''_{v}(\bm{v}_{i})$. The choice of $\gamma$ depends on the choice of the boundary condition: $\gamma$ is $0$ for the Dirichlet boundary conditions or $\gamma$ is a positive number for the natural boundary conditions.

The following lemma illustrates that the approximated Hessian $\hat{H}$ \eqref{AHessian} is symmetric positive definite.
\begin{lemma}\label{lemma}
The approximated Hessian $\hat{H}$ \eqref{AHessian} is symmetric positive definite.
\end{lemma}
\begin{proof}
If $\hat{H}$ \eqref{AHessian} is not symmetric positive definite, there exists a nonzero vector $v=(v^{T}_{1},v^{T}_{2})^{T}$ such that $v^{T}\hat{H}v\leq 0$. Hence, we have 
\begin{equation}\label{Equality}
v^{T}\hat{H}v = h^{3}(\vec{I}_{PY}Pv_{1}-Mv_{2})^{T}(\vec{I}_{PY}Pv_{1}-Mv_{2}) + \alpha_{l}h^{3}v_{1}^{T}A^{T}Av_{1}+\frac{h^{3}}{6}v_{1}^{T}(\alpha_{s}\mathrm{d}S^{T}\mathrm{d}S+\alpha_{v}\mathrm{d}\bm{v}^{T}\mathrm{d}^{2}\bm{\phi}_{v}\mathrm{d}\bm{v})v_{1} + \gamma v_{1}^{T}v_{1}.
\end{equation}
Since we assume that $\bm{v}(Y)>0$, $\mathrm{d}^{2}\bm{\phi}_{v}$ is a positive definite matrix. Then each term in the right hand side of \eqref{Equality} is nonnegative. 
\begin{enumerate}
\item[i] Natural boundary conditions. Here, $\gamma$ is a positive number. To satisfy $v^{T}\hat{H}v\leq 0$, $v_{1}$ must be a zero vector and $Mv_{2}$ is also a zero vector. Because $M$ is a full row rank matrix constructed in Section \ref{SDF}, $v_{2}$ is also a zero vector, which is a contraction.
\item[ii] Dirichlet boundary conditions. Here, $\gamma$ is $0$. But $A$ is full rank under the Dirichlet boundary conditions and hence, $A^{T}A$ is symmetric positive definite. Following the above discussion, we obtain that $v_{1}$ and $v_{2}$ are both zero vectors, which is also a contradiction.
\end{enumerate}
Hence, the approximated Hessian $\hat{H}$ \eqref{AHessian} is symmetric positive definite.
\end{proof}

Lemma \ref{lemma} ensures that the search direction $p$ generated by solving the generalized Gauss-Newton system \eqref{GaussnNewtonsystem} is a descent direction. Here, we choose MINRES to solve this system \eqref{GaussnNewtonsystem} \cite{barrett1994templates,paige1975solution} and the tolerance for the relative residual is set to 0.1. At the same time, we consider a preconditioner, which is a band matrix $Q = \begin{pmatrix} T & \\ & h^{3}M^{T}M  \end{pmatrix}$ and T is a tridiagonal matrix composed of the diagonals of blocks of upper right part of the approximated Hessian $\hat{H}$. We note that $M^{T}M$ is a diagonal matrix. By using the Cholesky decomposition and two back substitutions, the computational cost of solving $Qx=b$ is only $\mathcal{O}(3(n+1)^{3}+m)$.

\begin{remark}
(i) Here, the reason that we choose MINRES rather than CG is based on our numerical experience. Compared with CG, MINRES can use less iterations to reach the stopping criteria.
(ii) In the implementation, we provide a matrix-free version which can speed up the algorithm since we do not need to formulate and store the matrix $\hat{H}$.
\end{remark}

\subsubsection{Step Length $\eta$} 
The step length $\eta$ is chosen according to the backtracking strategy and simultaneously satisfies the sufficient decrease condition and guarantees the bijectivity. Hence, the line search strategy can be summarized in Algorithm \ref{Alg:LS}.

\newcommand{\ID}[0]{\mbox{ID}}
\begin{algorithm}[!ht]
\caption{Line Search Strategy for finding the step length $\eta$:
 $\eta \leftarrow$ LS($F,Y,C,p,d$)}
\label{Alg:LS}
\begin{algorithmic}
\STATE{Step 1: Set $\delta=10^{-4}$;}
\STATE{Step 2: Find the smallest integer $i_{k}\geq0$ such that $\eta=(0.5)^{i_{k}}$ ensures \\
\quad \qquad \quad $F(\tilde{Y},\tilde{C})\leq F(Y,C)+\eta\delta d^{T}p$ and $\bm{v}(\tilde{Y})>0$, where $(\tilde{Y}^{T},\tilde{C}^{T})^{T}=(Y^{T},C^{T})^{T}+\eta p$.
}
\end{algorithmic}
\end{algorithm}

The following lemma guarantees the existence of the step length $\eta$ provided by Algorithm \ref{Alg:LS}.
\begin{lemma}\label{STE}
If the current iterative point $(Y,C)$ satisfies $\bm{v}(Y)>0$, $p$ is obtained by solving \eqref{GaussnNewtonsystem} and $\delta\in(0,1)$, there exists $\xi>0$ such that
\begin{displaymath}
F(\tilde{Y},\tilde{C})\leq F(Y,C)+t\delta d^{T}p \quad\mathrm{and}\quad \bm{v}(\tilde{Y})>0, 
\end{displaymath}
for all $t\in[0,\xi)$, where $(\tilde{Y}^{T},\tilde{C}^{T})^{T}=(Y^{T},C^{T})^{T}+t p$.
\end{lemma}
\begin{proof}
Firstly, as $F$ is differentiable and $\delta\in(0,1)$, we have
\begin{equation}
\begin{split}
\lim_{t\rightarrow 0}\frac{F(\tilde{Y},\tilde{C})-F(Y,C)}{t} &= d^{T}p = -d^{T}\hat{H}^{-1}d \\
& < \delta d^{T}\hat{H}^{-1}d \ \ (\hat{H}\ \mathrm{is\ SPD})\\
& = \delta d^{T}p. 
\end{split}
\end{equation}
Hence, there exists $\xi_{1}>0$ such that
\begin{equation}
\frac{F(\tilde{Y},\tilde{C})-F(Y,C)}{t}<  \delta d^{T}p,
\end{equation}
for all $t\in(0,\xi_{1})$. Therefore,
\begin{displaymath}
F(\tilde{Y},\tilde{C})\leq F(Y,C)+t\delta d^{T}p,\ \forall\ t\in[0,\xi_{1}).
\end{displaymath}

Secondly, recall the definition of $\bm{v}$ in \eqref{matrixrepresentation}. Then we have 
\begin{equation}
\bm{v}(\tilde{Y}) = \bm{v}(Y) + t\bm{f}_{1}+t^{2}\bm{f}_{2}+t^{3}\bm{f}_{3},
\end{equation}
where $t\bm{f}_{1}, t^{2}\bm{f}_{2}$ and $t^{3}\bm{f}_{3}$ are the combinations of the corresponding terms containing $t, t^{2}$ and $t^{3}$. Since $\bm{v}(Y)>0$, there exists $\xi_{2}>0$ such that $\bm{v}(\tilde{Y})>0$ for all $t\in[0,\xi_{2})$. 

Finally, set $\xi=\min\{\xi_{1},\xi_{2}\}$ and the proof is complete.
\end{proof}

Hence, for the step length $\eta$, we just need to find the smallest integer $i_{k}\geq 0$ such that $\eta = (0.5)^{i_{k}}\leq \xi$.

\subsubsection{Convergence of The Generalized Gauss-Newton Method}

Now, we can summarize the generalized Gauss-Newton method in Algorithm \ref{Alg:GNM}.
\begin{algorithm}[!ht]
\caption{Generalized Gauss-Newton Method for Topology-Preserving Image Segmentation:
 $(Y,C) \leftarrow$ GGN($\alpha_{l}, \alpha_{s}, \alpha_{v},Y^{0}, C^{0}, I, J$)}
\label{Alg:GNM}
\begin{algorithmic}
\STATE{Step 1:  For \eqref{DProposedModel}, compute $F(Y^{0},C^{0})$, $d^{0}$ and $\hat{H}^{0}$;}
\STATE{Step 2: Set $k=0$;}
\WHILE{``the stopping criteria are not satisfied''}
\STATE{--- Solve $\hat{H}^{k}p^{k}=-d^{k}$ from \eqref{GaussnNewtonsystem};}
\STATE{--- Update $(Y^{k+1},C^{k+1})$ by \eqref{update};}
\STATE{--- $k=k+1$;}
\STATE{--- compute $F(Y^{k},C^{k})$, $d^{k}$ and $\hat{H}^{k}$;}
\ENDWHILE
\end{algorithmic}
\end{algorithm}
Here, the stopping criteria in Algorithm \ref{Alg:GNM} is consistent with the literature \cite{modersitzki2009fair,zhang2018novel}, namely,  when the change in the objective function, the norm of the update and the norm of the gradient are all sufficiently small, the iterations are terminated.

Next, we discuss the convergence of Algorithm \ref{Alg:GNM}. 
First, we review a theorem from \cite{chen2019improved}.
\begin{theorem}[Theorem 2 in \cite{chen2019improved}]\label{GC}
Consider a finite-dimensional optimization problem:
\begin{equation}
\min_{x\in\mathbb{R}^{n}} f(x) \quad \mathrm{s.t.} \quad x\in\mathcal{X},
\end{equation}
where $f:\mathbb{R}^{n}\rightarrow\mathbb{R}$ is a differentiable function and $\mathcal{X}\subset\mathbb{R}^{n}$ is an open set. The iterative scheme is as follows:
\begin{equation}\label{IS}
x^{k+1} = x^{k} -\eta^{k}(B^{k})^{-1}\nabla f(x_{k}) \quad \mathrm{and} \quad x^{k+1}\in\mathcal{X},
\end{equation}
where $\eta^{k}$ is derived by Armijo strategy and $B^{k}$ is a symmetric and positive definite matrix. If the following conditions are satisfied:
\begin{itemize}
\item[A1] The set $L(x^{0}) = \{x\in\mathbb{R}^{n}|f(x)\leq f(x^{0})\}$ is compact.
\item[A2] $\nabla f:\mathbb{R}^{n}\rightarrow\mathbb{R}^{n}$ is $L$-Lipschitz.
\item[A3] There exist constants $\kappa_{1}\geq\kappa_{0}>0$ such that
\begin{displaymath}
\kappa_{0}I \preceq B^{k} \preceq \kappa_{1}I, \quad \forall k.
\end{displaymath}
\end{itemize}
then given $x^{0}\in\mathcal{X}$, the sequence $\{x^{k}\}\subset\mathcal{X}$ generated by the iterative scheme \eqref{IS} from $x^{0}$ admits a subsequence that converges either to a point in the boundary of $\mathcal{X}$ or to a critical point of $f$ in $\mathcal{X}$.
\end{theorem}

Based on Theorem \ref{GC}, we have the following convergence theorem for Algorithm \ref{Alg:GNM}.
\begin{theorem}
For the resulting finite-dimensional optimization problem \eqref{DProposedModel}, given $(Y^{0},C^{0})$ satisfying $\bm{v}(Y^{0})>0$, the sequence $\{(Y^{k},C^{k})|\bm{v}(Y^{k})>0\}_{k\in\mathbb{N}}$ generated by Algorithm \ref{Alg:GNM} from $(Y^{0},C^{0})$ admits a sequence that converges to a critical point $(Y^{*},C^{*})$ of $F$ and $\bm{v}(Y^{*})>0$.
\end{theorem}
\begin{proof}
Define $\mathcal{X} = \{Y|\bm{v}(Y)>0\}\times\mathbb{R}^{m}$ and obviously, $\mathcal{X}$ is an open set. Because the initial guess point $(Y^{0},C^{0})$ satisfies $\bm{v}(Y^{0})>0$, according to Lemma \ref{STE}, Algorithm \ref{Alg:GNM} will generate a sequence $\{(Y^{k},C^{k})|\bm{v}(Y^{k})>0\}_{k\in\mathbb{N}} \subset \mathcal{X}$ such that $F(Y^{0},C^{0})>F(Y^{1},C^{1})>\cdots>F(Y^{k},C^{k})>\cdots$. 

Next, we will prove that the three conditions in Theorem \ref{GC} are satisfied. 

For A1, in reality, we just need to prove that there exists a compact set $\mathcal{X}_{1}$ such that $\{(Y^{k},C^{k})|\bm{v}(Y^{k})>0\}_{k\in\mathbb{N}} \subset \mathcal{X}_{1}$. Since the discretized deformed template $\vec I(PY)$ is bounded, $C$ is bounded and no matter the Dirichlet boundary conditions or the natural boundary conditions are employed, $Y$ is also bounded. Recall \eqref{DProposedModel} and note that $F(Y,C)$ will be infinity if any component of $\bm{v}(Y)$ goes to $0$ or $\infty$. So there exist $a, a_{1},...,a_{m}$ and $b, b_{1},...,b_{m}$ such that $\{(Y^{k},C^{k})|\bm{v}(Y^{k})>0\}_{k\in\mathbb{N}} \subset \mathcal{X}_{1} = \{Y|a\leq\bm{v}(Y)\leq b\}\times \Pi_{l=1}^{m}[a_{l},b_{l}]$.

For A2, since $F(Y,C)$ in \eqref{DProposedModel} is second order differentiable for $(Y,C)\in\mathcal{X}_{1}$, $\nabla F(Y,C)$ is $L$-Lipschitz for $(Y,C)\in\mathcal{X}_{1}$.

For A3, since $\{(Y^{k},C^{k})\}_{k\in\mathbb{N}}$ generated by Algorithm \ref{Alg:GNM} are in $\mathcal{X}_{1}$, this condition is also satisfied. First, since $\mathcal{X}_{1}$ is compact, we can find $\kappa_{1}$ such that $\hat{H}^{k} \preceq \kappa_{1}I$. If the Dirichlet boundary conditions are employed, since $A^{T}A$ is a constant matrix which is symmetric and positive definite, we can set the smallest eigenvalue of $A^{T}A$ as $\kappa_{0}$ such that $\kappa_{0}I \preceq \hat{H}^{k}$. Or if the natural boundary conditions are employed, since $\gamma$ is a fixed positive number, we can set  $\gamma$ as $\kappa_{0}$ such that $\kappa_{0}I \preceq \hat{H}^{k}$.

Hence, from Theorem \ref{GC}, we can get that the sequence $\{(Y^{k},C^{k})|\bm{v}(Y^{k})>0\}_{k\in\mathbb{N}}$ generated by Algorithm \ref{Alg:GNM} from $(Y^{0},C^{0})$ satisfying $\bm{v}(Y^{0})>0$ admits a subsequence that converges to a critical point $(Y^{*},C^{*})$ of $F$ and $\bm{v}(Y^{*})>0$ or to a point $(Y^{*},C^{*})$ which is in the boundary of $\mathcal{X}$, namely satisfying $\bm{v}(Y^{*})=0$. But we know that $F(Y^{*},C^{*})$ will be infinite if $\bm{v}(Y^{*})=0$, which is a contradiction. So the proof is complete.
\end{proof}

\begin{remark}
According to the worst-case analysis in \cite{chen2019improved}, Algorithm \ref{Alg:GNM} takes at most $\mathcal{O}(\epsilon^{-2})$ iterations to generate a point $(\hat{Y},\hat{C})$ satisfying $\bm{v}(\hat{Y})>0$ such that $\|\nabla F(\hat{Y},\hat{C})\|\leq\epsilon$.
\end{remark}

\subsubsection{Multilevel Strategy}

As a standard procedure to provide a good initial guess, the multilevel strategy is used in the implementation \cite{haber2006multilevel,modersitzki2009fair}. Firstly, we coarsen the template and the reference by $L$ levels. 
On the coarsest level, set $Y^{0} = X$ and $c_{l} = \frac{\sum_{\bm{x}^{i,j,k}\in\Omega_{l}}I(\bm{y}(\bm{x}^{i,j,k}))}{\sharp(\bm{x}^{i,j,k}\in\Omega_{l})}, 1\leq l\leq m$ as the initial guess point, where $\sharp(\bm{x}^{i,j,k}\in\Omega_{l}^{c})$ denotes the number of cell-centered points in $\Omega_{l}$.
Then we can obtain $(Y_{1},C_{1})$ by solving our model \eqref{DProposedModel} on the coarsest level. To give a good initial guess for the finer level, we adopt the nodal interpolation on $Y_{1}$ to obtain $Y_{2}^{0}$ and recompute $C_{2}^{0}$ as the initial guess for the next level. We repeat this process and get the final registration on the finest level.
The most important advantage of the this strategy is that it can save computational time to provide a good initial guess for the finer level because there are fewer variables on the coarser level. Also, it can help to avoid to trap into a local minimum since the coarser level only shows the main features and patterns.

\section{Numerical Experiments}\label{SNumericalResults}
In this section, we test the proposed model \eqref{ProposedModel} with 2D and 3D images. All codes are implemented by Matlab R2019a on a MacbookPro with 2.2 GHz Quad-Core Intel Core i7 processor and 16 GB RAM. Our implementation is made public and can be downloaded from https://sites.google.com/view\\/daopingzhang/home.

\subsection{2D Examples}

In this subsection, we test our proposed model on three real 2D images. All images are resized into $256\times256$ and their intensities are rescaled into $[0, 255]$. Here, for the parameters in the 2D case, $\alpha_{s}$ should be set $0$ and according to our numerical experiences, $[50,10^{3}]\times[10,10^{3}]$ may be the suitable range for $(\alpha_{l},\alpha_{v})$. Here, for all the three examples, we just set $\alpha_{l}=10^{2}$ and $\alpha_{v}=10^{2}$. We compare our proposed model with the Chan-Vese model and a 2D selective model \cite{zhang2014local}. For the Chan-Vese model, we call the MATLAB function \textit{activecontour} and the code of \cite{zhang2014local} is downloaded from https://www.liverpool.ac.uk/$\sim$cmchenke/softw/select$\_$2D-B-2014.htm.

\bigskip 

\noindent {\bf Example 1:} In this example, we test our proposed model \eqref{ProposedModel} on a 2D image from FAIR \cite{modersitzki2009fair} with two connected components with complex geometry. We test our proposed model with four different topological priors as the initial contours, which are shown in the first column of Figure \ref{2DExample_1}. The first three priors possess the same topological structure and the fourth one has another topological structure. The results generated by the Chan-Vese model, the selective model and the proposed model \eqref{ProposedModel} are listed in the second to forth columns of Figure \ref{2DExample_1}, respectively. The fifth column displays the corresponding transformation $\bm{y}$ obtained by the proposed model \eqref{ProposedModel}. We can observe that the results generated by the proposed model \eqref{ProposedModel} are indeed topological preserving. This can be confirmed by monitoring the minimum of the Jacobian determinant of the transformations, which are positive. However, for the Chan-Vese model, it does not preserve the topological structure of the prior information. In addition, for these four priors, the Chan-Vese model all can segment these two pieces of tissues. This is because the Chan-Vese model is a global segmentation model. For the selective model, it can lead to the similar results with the proposed model for the first, second and fourth case. But it also segments two pieces for the third case, which means that it does not preserve topology. For the proposed model \eqref{ProposedModel}, the segmentation results are different by the different priors. More specifically, by placing the initial contour at different locations, our proposed method is able to capture different connected components (see (a)$\&$(c) and (e)$\&$(g)). Hence, this model can be considered as a selective model. The users can provide different priors and place them at different locations, according to the structure of the target object and their preferences, to capture different objects in an image. The energy plot versus iterations is displayed in Figure \ref{Energy2D}. The computational time is also listed in Figure \ref{2DExample_1}. We can see that although the computational time of the proposed model is more than the Chan-Vese model, it saves too much time compared with the selective model.

In addition, we also investigate the case of wrong priors. The result is displayed in Figure \ref{2DExample_1_2}. From Figure \ref{2DExample_1_2}, first we can see that our proposed model can keep the topological structure. Second, the proposed model still segments the object target because one part of the prior is right. For the other part of the prior, its position seems fixed and does not affect the results. But for the Chan-Vese model and the selective model, they give the segmentation which are not topology-preserving.  

\begin{figure}[htbp]
\centering
\subfigure[$I$ and $J$ (red)]{
\includegraphics[width=1.1in,height=1.1in]{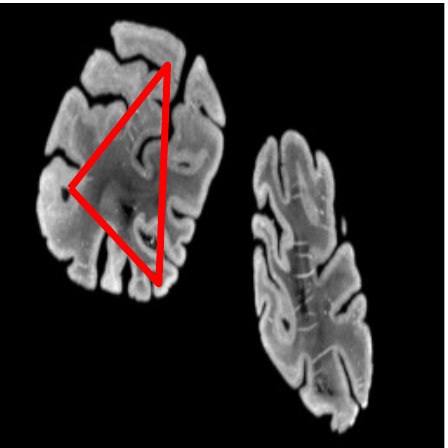}}
\subfigure[CV (4.39 sec)]{
\includegraphics[width=1.1in,height=1.1in]{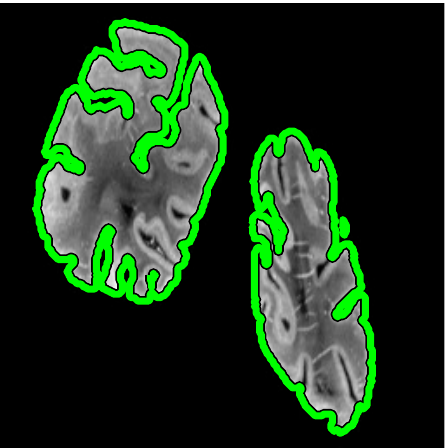}}
\subfigure[\cite{zhang2014local} (119.12 sec)]{
\includegraphics[width=1.1in,height=1.1in]{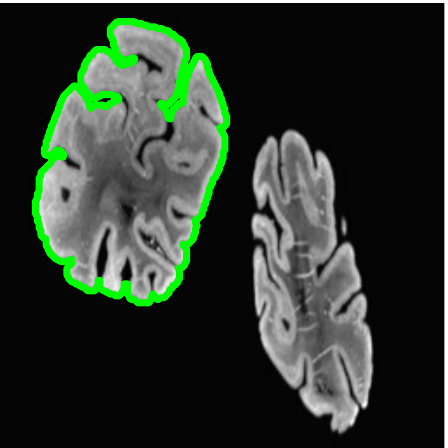}}
\subfigure[PM (3.90 sec)]{
\includegraphics[width=1.1in,height=1.1in]{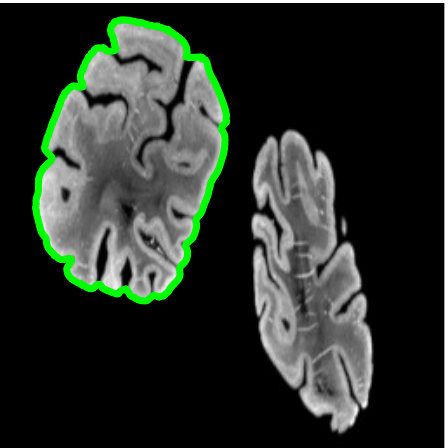}}
\subfigure[$\bm{y}$, $\det\nabla\bm{y}\in\lbrack 0.30,14.20\rbrack$]{
\includegraphics[width=1.1in,height=1.1in]{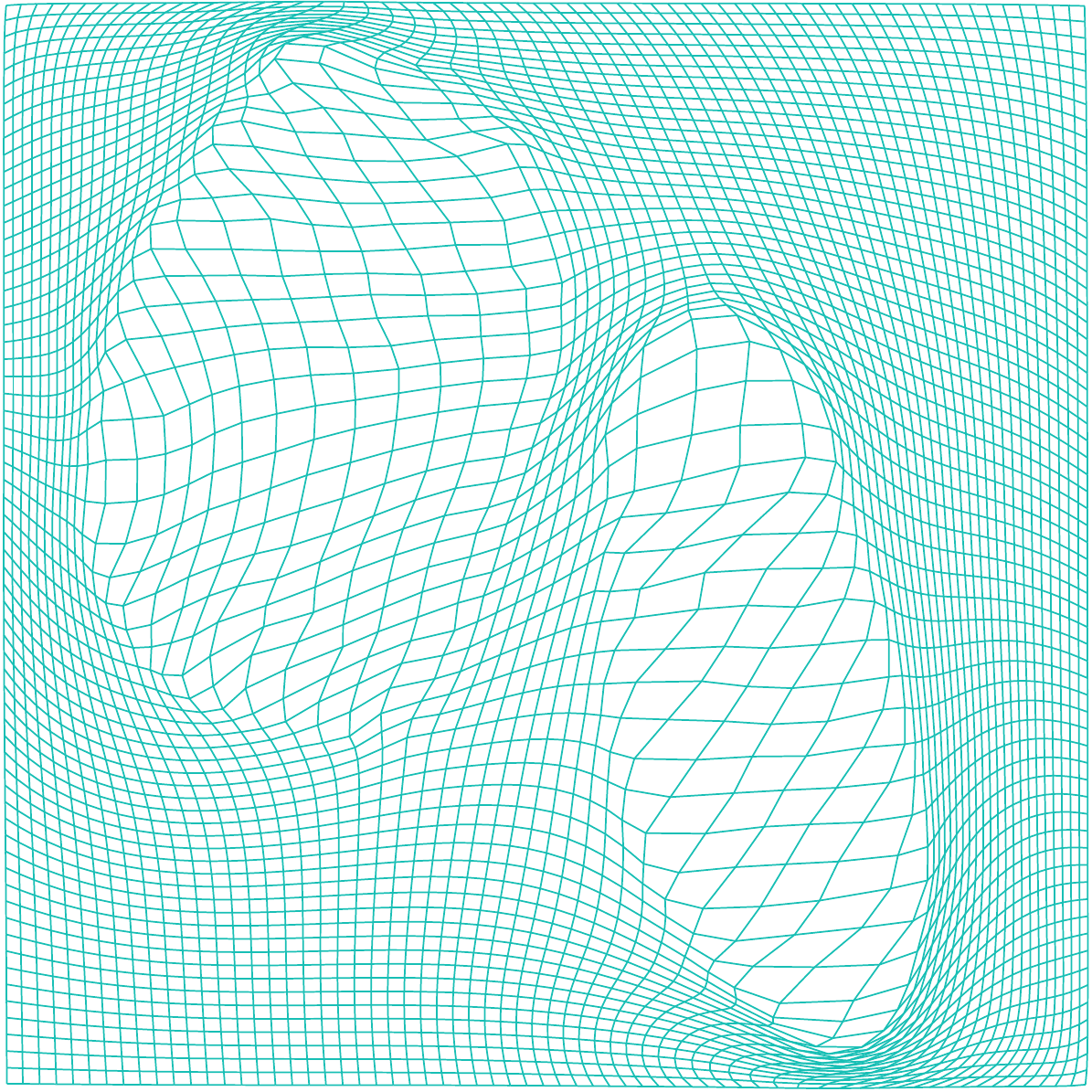}}\\
\subfigure[$I$ and $J$ (red)]{
\includegraphics[width=1.1in,height=1.1in]{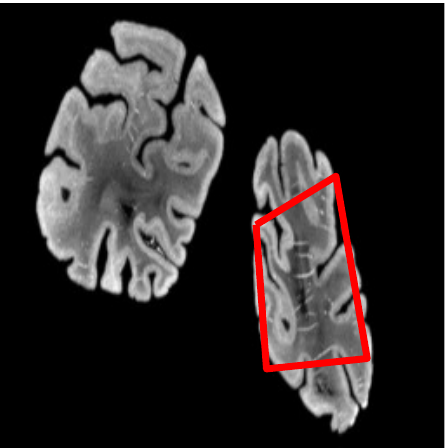}}
\subfigure[CV (2.69 sec)]{
\includegraphics[width=1.1in,height=1.1in]{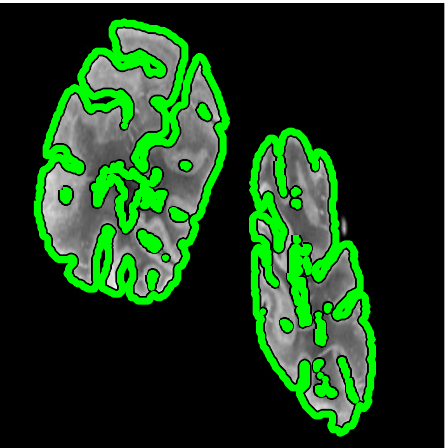}}
\subfigure[\cite{zhang2014local} (35.40 sec)]{
\includegraphics[width=1.1in,height=1.1in]{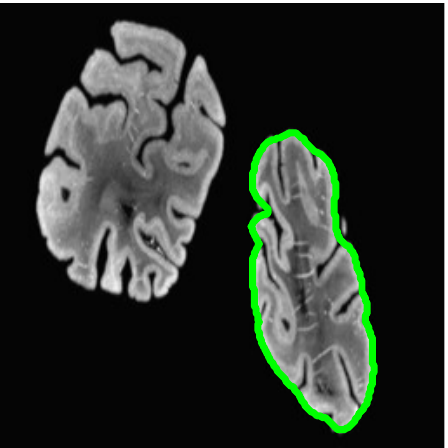}}
\subfigure[PM (4.24 sec)]{
\includegraphics[width=1.1in,height=1.1in]{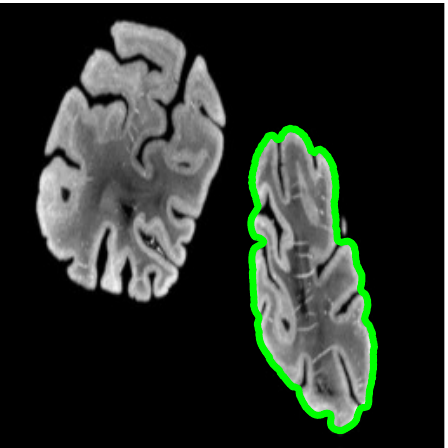}}
\subfigure[$\bm{y}$, $\det\nabla\bm{y}\in\lbrack 0.30,16.31\rbrack$]{
\includegraphics[width=1.1in,height=1.1in]{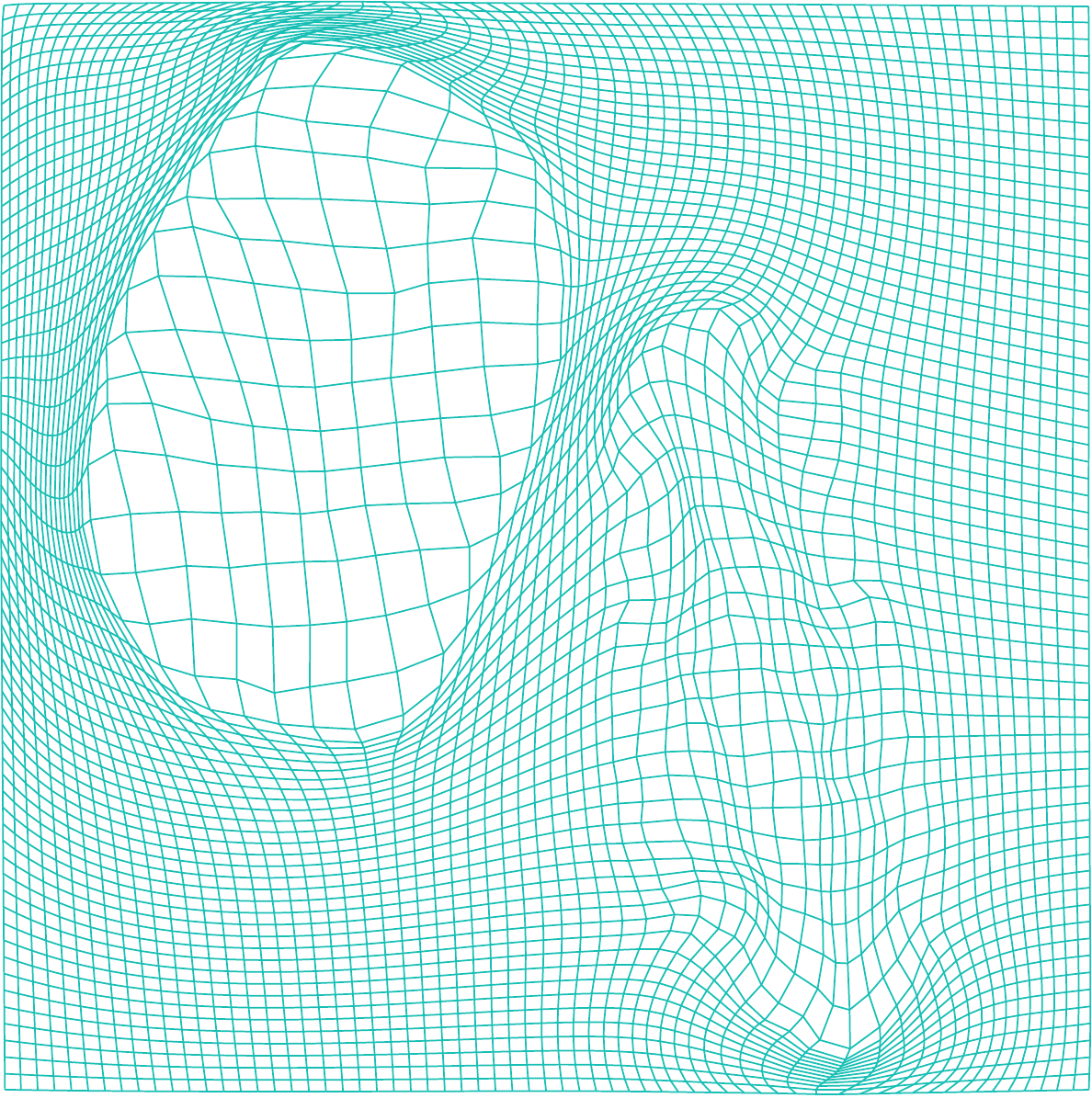}}\\
\subfigure[$I$ and $J$ (red)]{
\includegraphics[width=1.1in,height=1.1in]{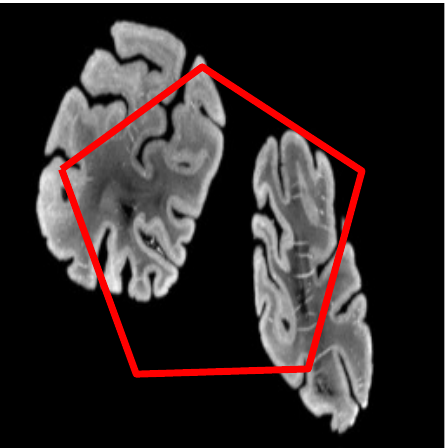}}
\subfigure[CV (2.66 sec)]{
\includegraphics[width=1.1in,height=1.1in]{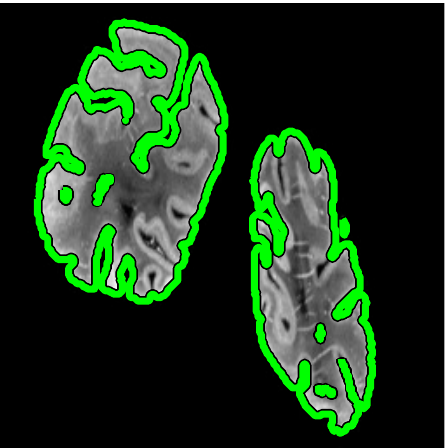}}
\subfigure[\cite{zhang2014local} (145.33 sec)]{
\includegraphics[width=1.1in,height=1.1in]{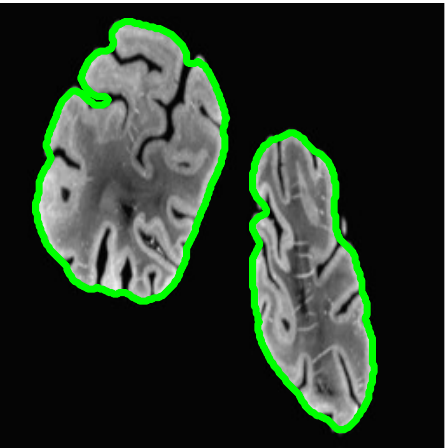}}
\subfigure[PM (6.71 sec)]{
\includegraphics[width=1.1in,height=1.1in]{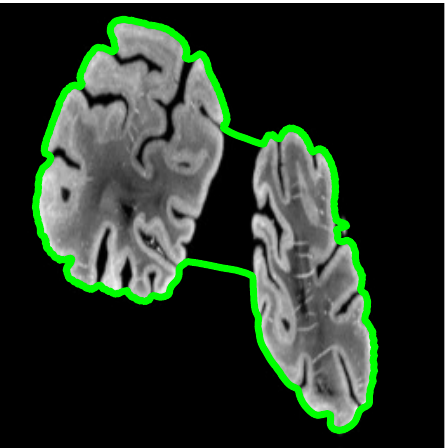}}
\subfigure[$\bm{y}$, $\det\nabla\bm{y}\in\lbrack 0.29,7.74\rbrack$]{
\includegraphics[width=1.1in,height=1.1in]{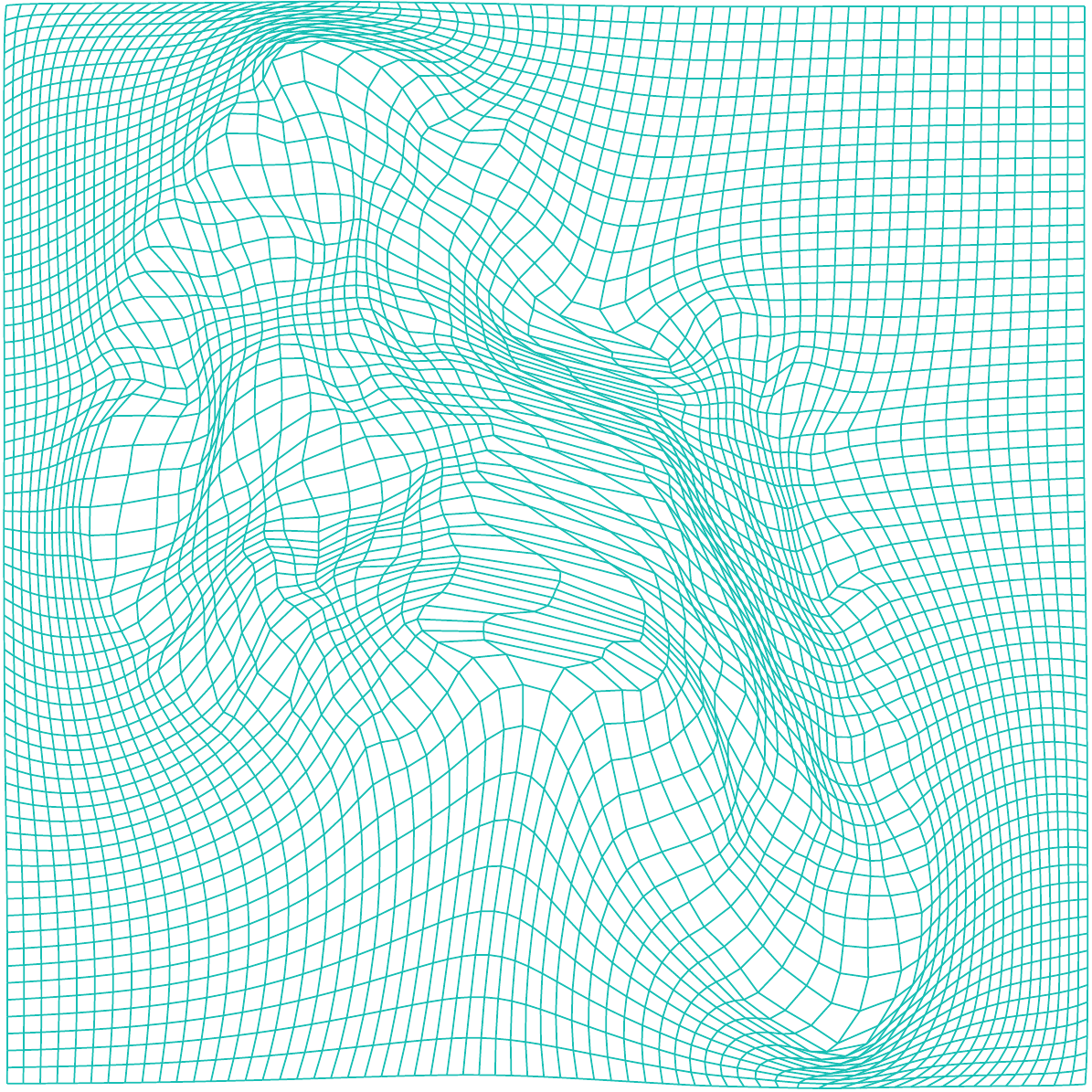}}\\
\subfigure[$I$ and $J$ (red)]{
\includegraphics[width=1.1in,height=1.1in]{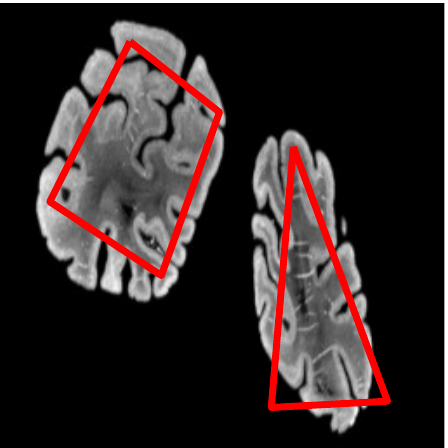}}
\subfigure[CV (3.17 sec)]{
\includegraphics[width=1.1in,height=1.1in]{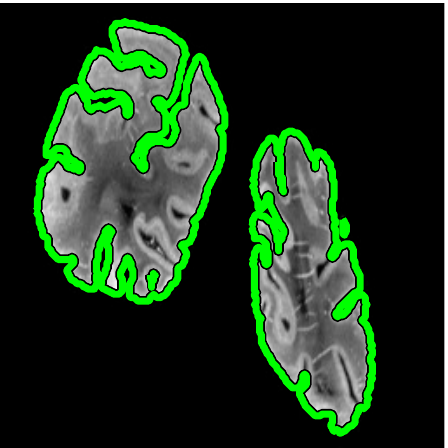}}
\subfigure[\cite{zhang2014local} (58.77 sec)]{
\includegraphics[width=1.1in,height=1.1in]{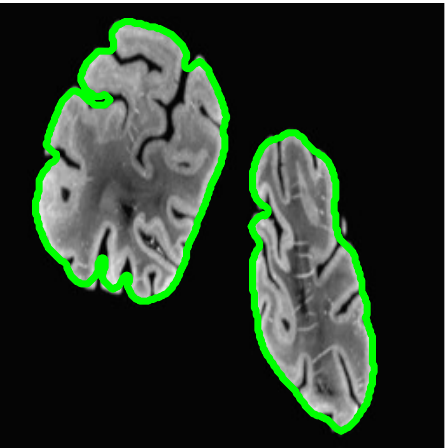}}
\subfigure[PM (3.29 sec)]{
\includegraphics[width=1.1in,height=1.1in]{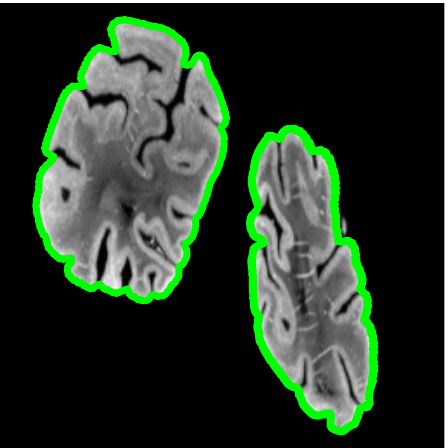}}
\subfigure[$\bm{y}$, $\det\nabla\bm{y}\in\lbrack 0.29,8.31\rbrack$]{
\includegraphics[width=1.1in,height=1.1in]{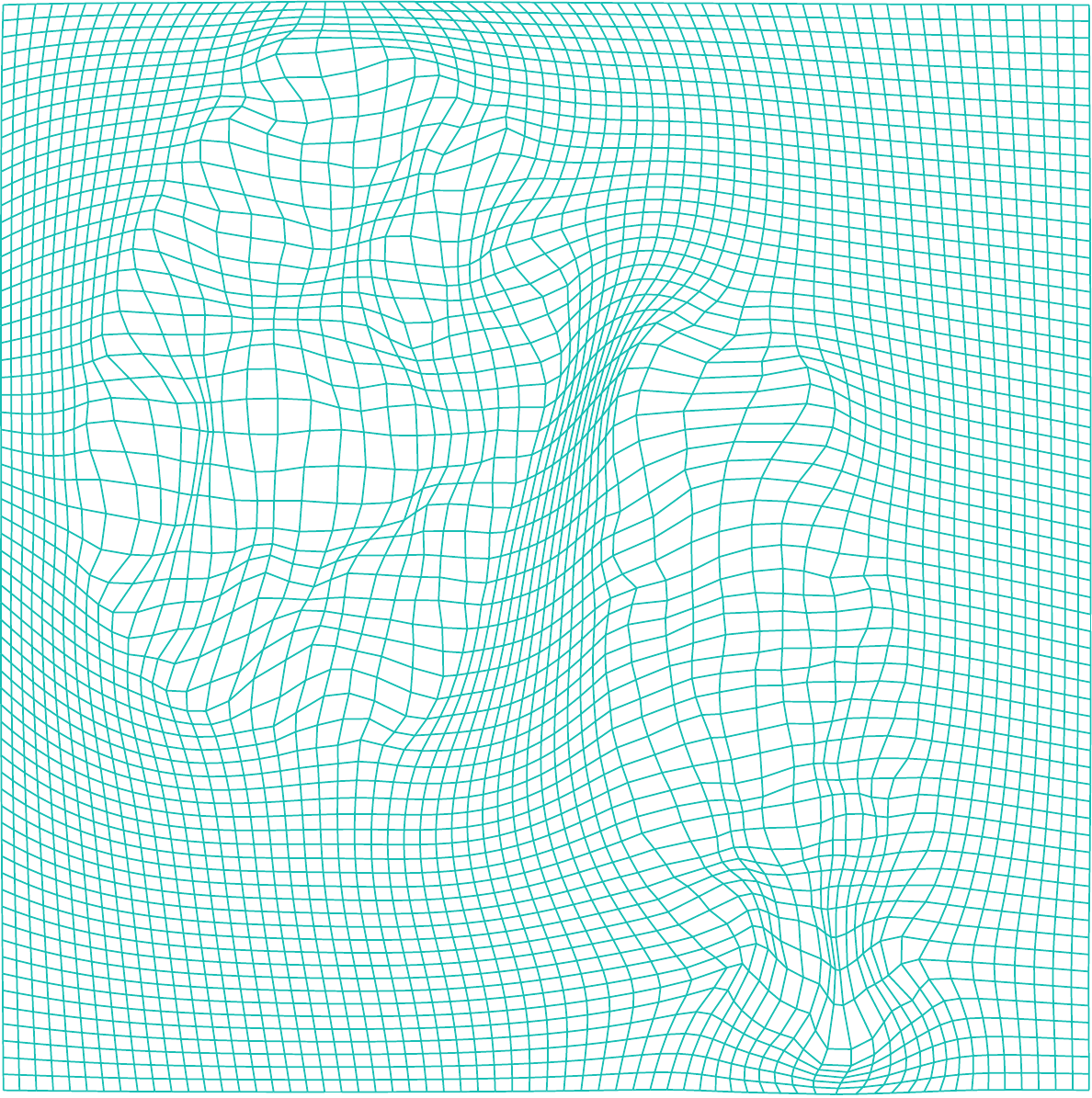}}
\caption{Results of First 2D Example. First column: boundary of the topological prior $J$ (red) superimposed on the input image $I$. Second column: results generated by the Chan-Vese model. Third column: results generated by a selective model \cite{zhang2014local}. Fourth column: results generated by the proposed model (PM) \eqref{ProposedModel}. Fifth column: the corresponding transformations $\bm{y}$ generated by the proposed model \eqref{ProposedModel}.}\label{2DExample_1}
\end{figure}

\begin{figure}[htbp]
\centering
\includegraphics[width=4.0in,height=2.0in]{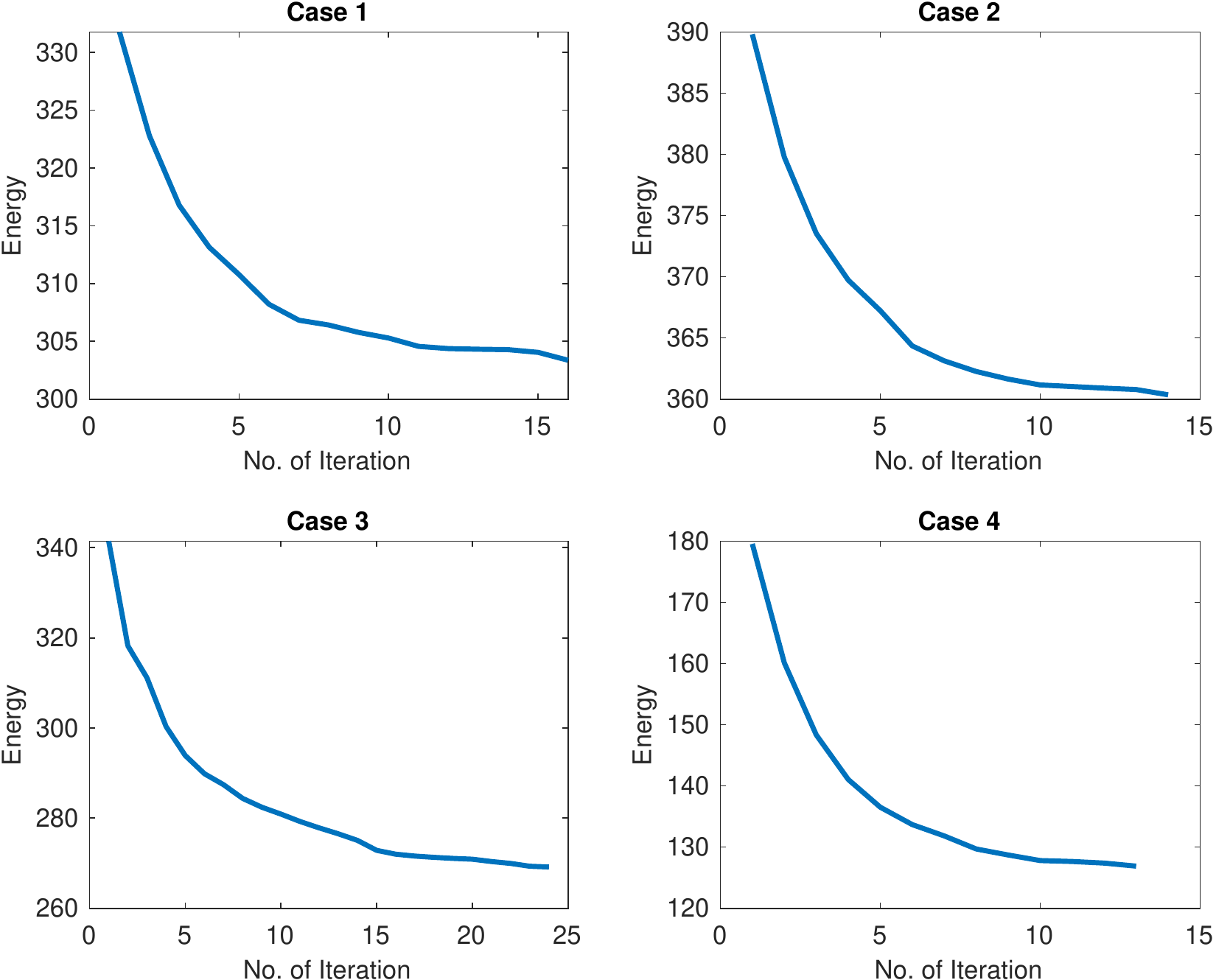}
\caption{Energy versus iterations of First 2D Example.}\label{Energy2D}
\end{figure}

\begin{figure}[htbp]
\centering
\subfigure[$I$ and $J$ (red)]{
\includegraphics[width=1.1in,height=1.1in]{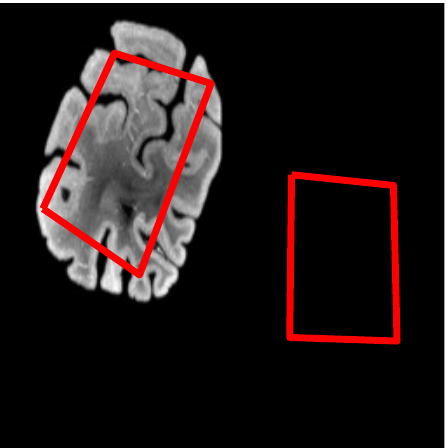}}
\subfigure[CV (2.36 sec)]{
\includegraphics[width=1.1in,height=1.1in]{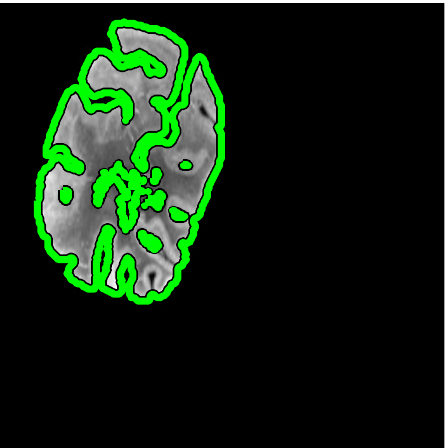}}
\subfigure[\cite{zhang2014local} (77.49 sec)]{
\includegraphics[width=1.1in,height=1.1in]{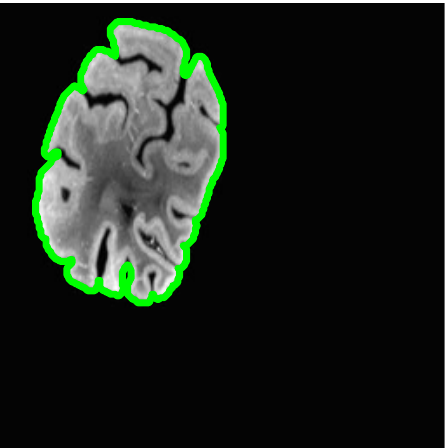}}
\subfigure[PM (5.81 sec)]{
\includegraphics[width=1.1in,height=1.1in]{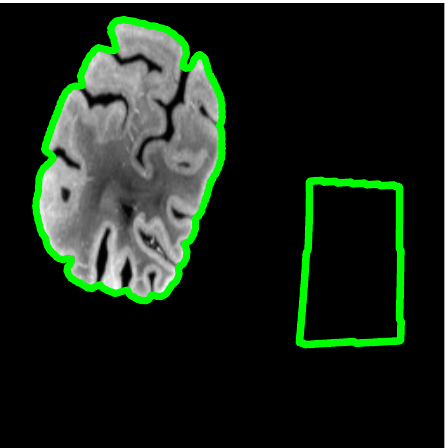}}
\subfigure[$\bm{y}$, $\det\nabla\bm{y}\in\lbrack 0.27,8.31\rbrack$]{
\includegraphics[width=1.1in,height=1.1in]{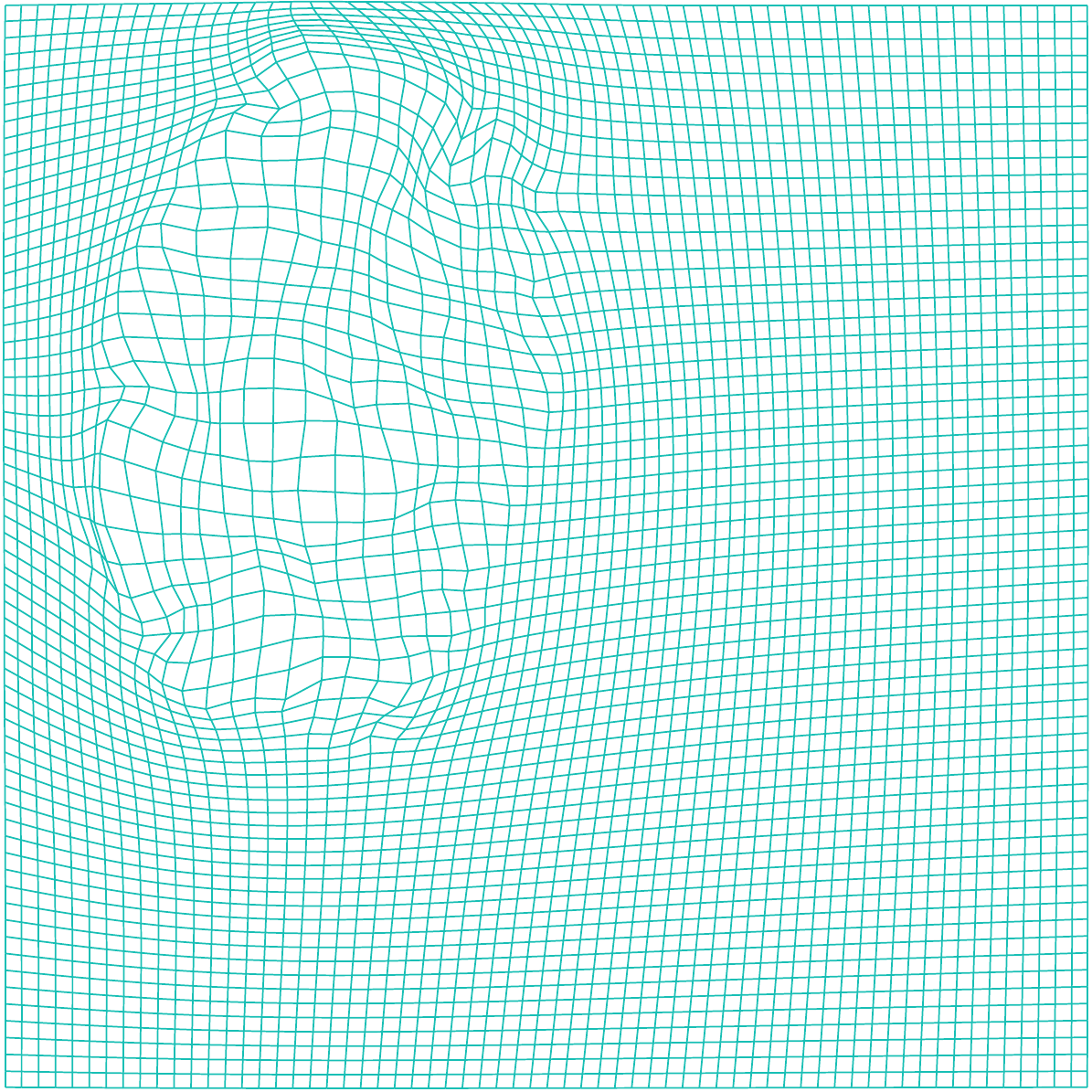}}
\caption{Results of First 2D Example. First column: boundary of the topological prior $J$ (red) superimposed on the input image $I$. Second column: result generated by the Chan-Vese model. Third column: result generated by a selective model \cite{zhang2014local}. Fourth column: result generated by the proposed model (PM) \eqref{ProposedModel}. Fifth column: the corresponding transformation $\bm{y}$ generated by the proposed model \eqref{ProposedModel}.}\label{2DExample_1_2}
\end{figure}

\bigskip

\noindent {\bf Example 2:} In this example, we test our proposed model on a 2D brain MRI also from FAIR \cite{modersitzki2009fair}, which is shown in Figure \ref{2DExample_2}. We use two different topological priors, as shown in the first column of Figure \ref{2DExample_2}. The second to fourth column show the results obtained by the Chan-Vese model, the selective model and our proposed model respectively. The first row shows the results with a simply-connected initial contour. As shown in (b), the segmentation result obtained by Chan-Vese model cannot preserve the topological prior. The result obtained by our proposed model is shown in (d). Our method successfully preserves the topology. The second row shows the results with a doubly-connected initial contour. Again, the result obtained by the Chan-Vese model cannot preserve the topological prior, as shown in (g). The result obtained by our proposed method is shown in (i), which preserves the topology. The transformations obtained by the proposed model \eqref{ProposedModel} are bijective and the minimum of the Jacobian determinant of the transformations are positive. For the selective model, we can see that for these two different priors, it leads to the same results. In addition, the results obtained by the Chan-Vese model take the different topological structure, which is unpredictable. But the proposed model \eqref{ProposedModel} can ensure the topological structures of the results are consistent with the topological structures of priors.

\bigskip

\begin{figure}[htbp]
\centering
\subfigure[$I$ and $J$ (red)]{
\includegraphics[width=1.1in,height=1.1in]{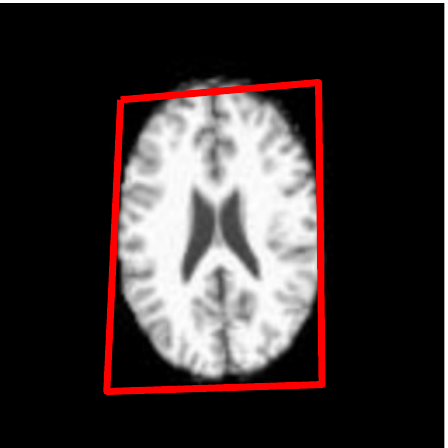}}
\subfigure[CV (1.63 sec)]{
\includegraphics[width=1.1in,height=1.1in]{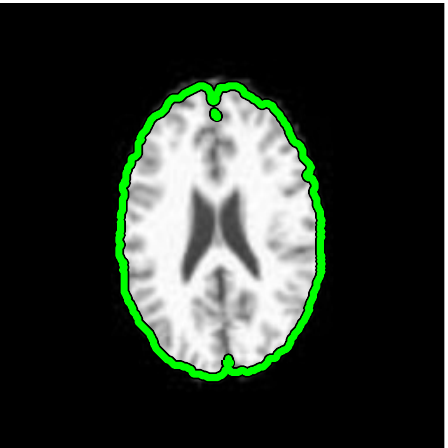}}
\subfigure[\cite{zhang2014local} (18.87 sec)]{
\includegraphics[width=1.1in,height=1.1in]{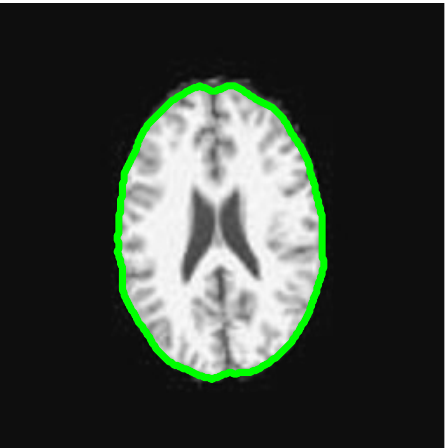}}
\subfigure[PM (4.54 sec)]{
\includegraphics[width=1.1in,height=1.1in]{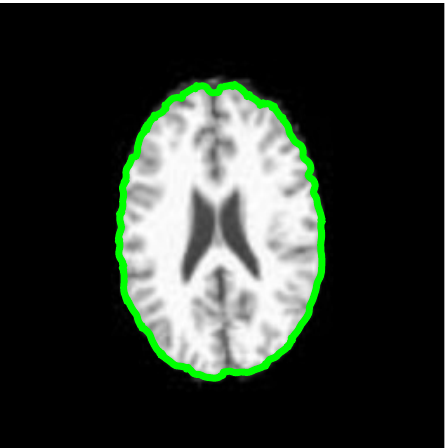}}
\subfigure[$\bm{y}$, $\det\nabla\bm{y}\in\lbrack 0.27,11.11\rbrack$]{
\includegraphics[width=1.1in,height=1.1in]{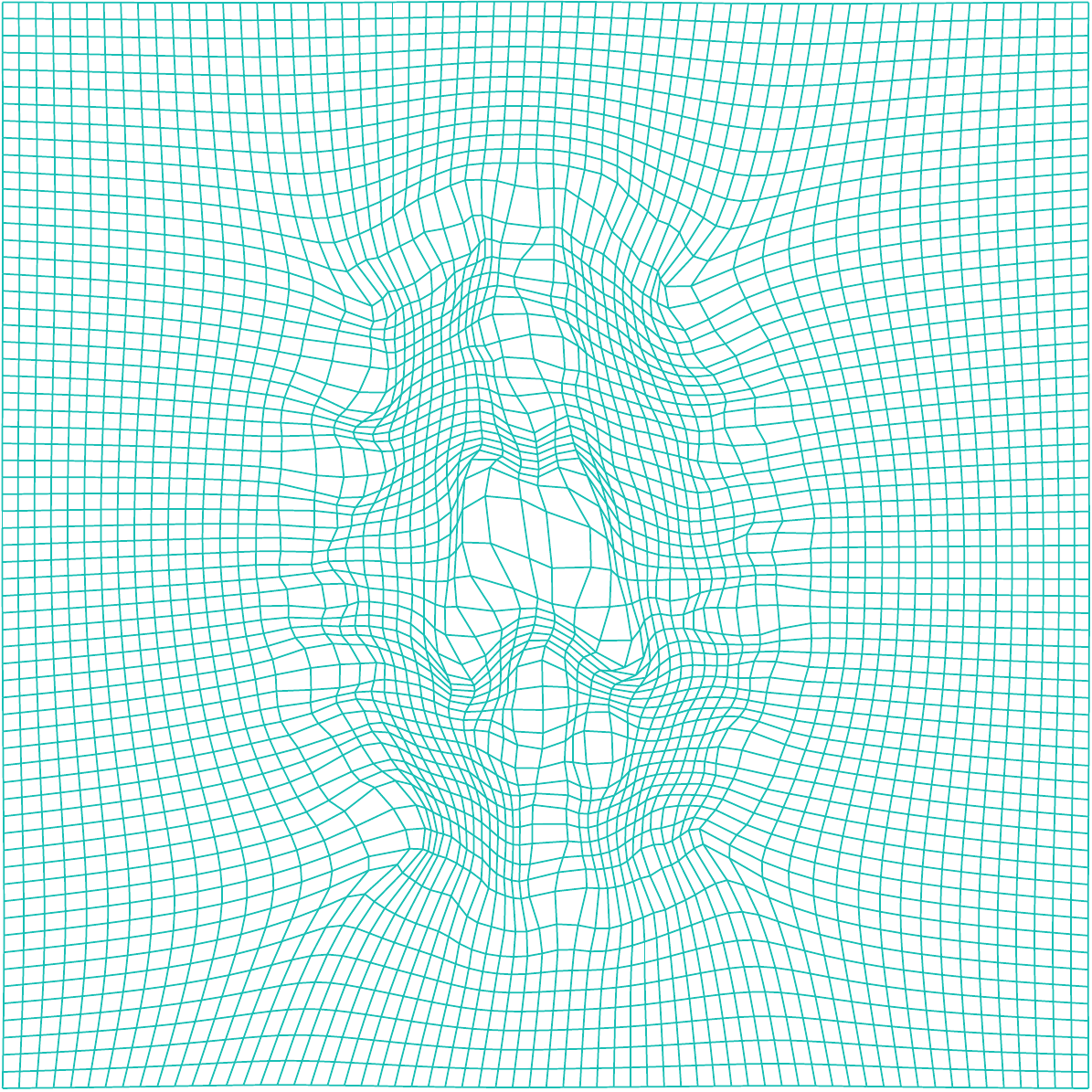}}\\
\subfigure[$I$ and $J$ (red)]{
\includegraphics[width=1.1in,height=1.1in]{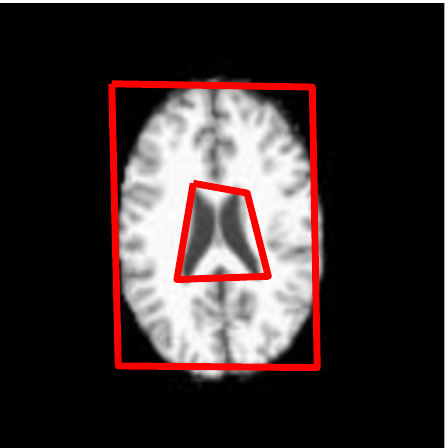}}
\subfigure[CV (1.84 sec)]{
\includegraphics[width=1.1in,height=1.1in]{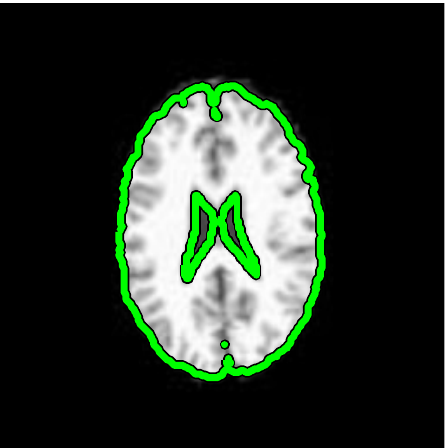}}
\subfigure[\cite{zhang2014local} (17.05 sec)]{
\includegraphics[width=1.1in,height=1.1in]{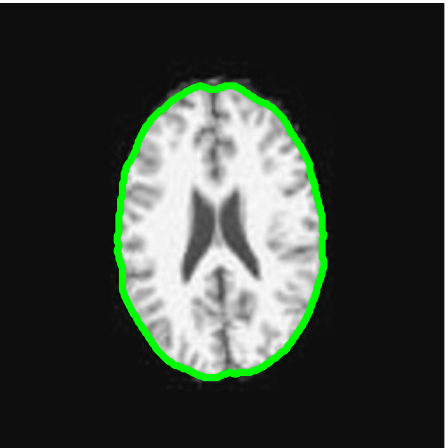}}
\subfigure[PM (3.66 sec)]{
\includegraphics[width=1.1in,height=1.1in]{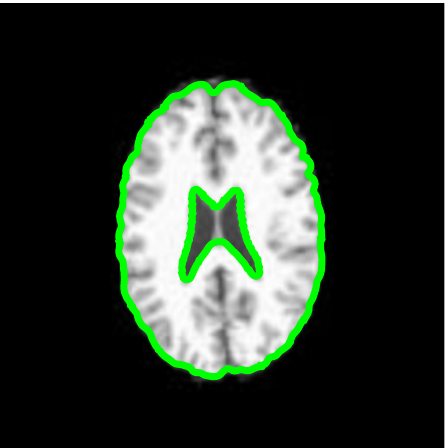}}
\subfigure[$\bm{y}$, $\det\nabla\bm{y}\in\lbrack 0.28,8.81\rbrack$]{
\includegraphics[width=1.1in,height=1.1in]{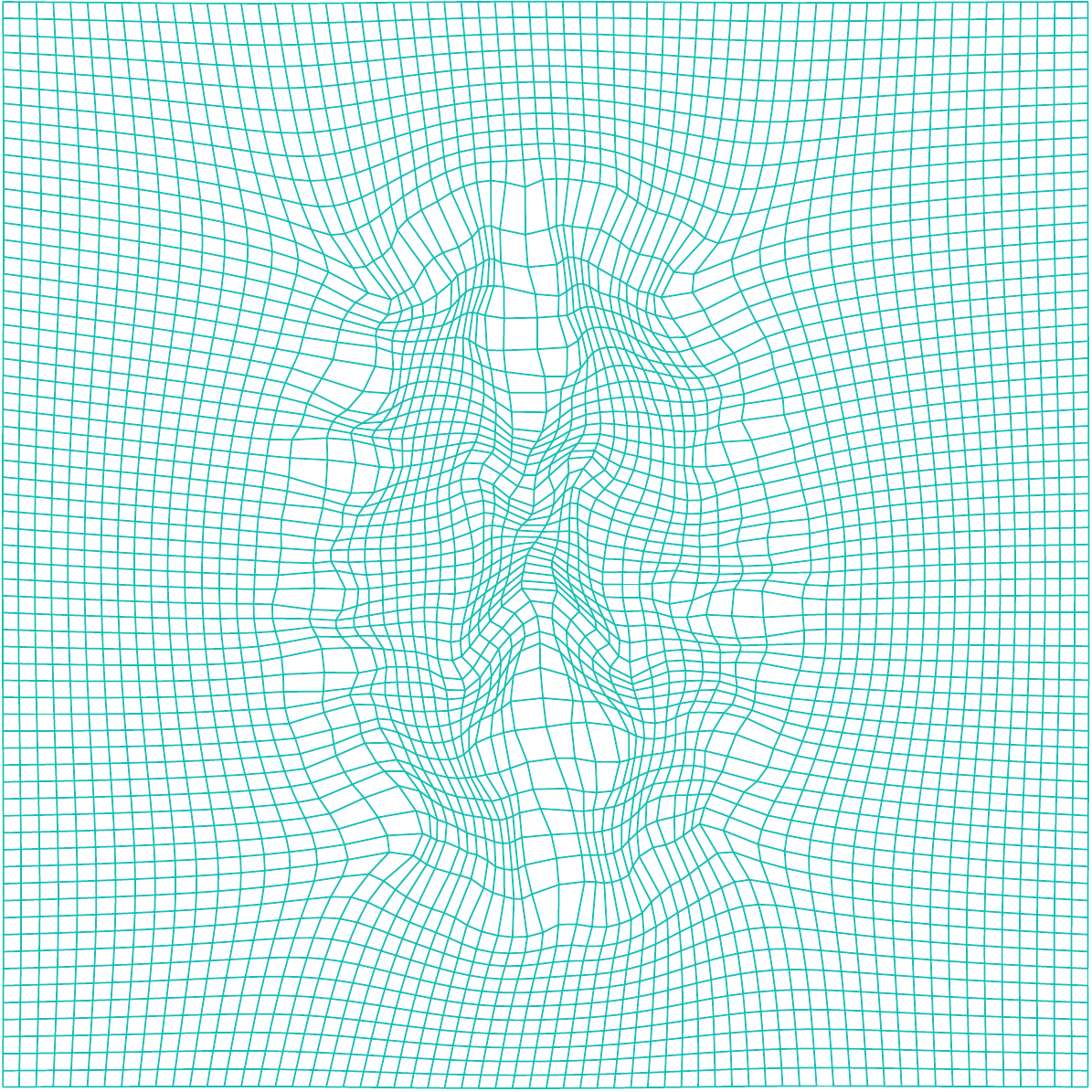}}

\caption{Results of Second 2D Example. First column: boundary of the topological prior $J$ (red) superimposed on the input image $I$. Second column: results generated by the Chan-Vese model. Third column: results generated by a selective model \cite{zhang2014local}. Fourth column: results generated by the proposed model (PM) \eqref{ProposedModel}. Fifth column: the corresponding transformations $\bm{y}$ generated by the proposed model \eqref{ProposedModel}.}\label{2DExample_2}
\end{figure}

\noindent {\bf Example 3:} In Figure \ref{2DExample_3}, we test our proposed model on a 2D lung CT scan, one slice of a video file downloaded from https://www.youtube.com/watch?v=RMYzgm4eJDE. In this example, we consider three different topological priors, which are shown in the first column. The segmentation results generated by our proposed model with different initial contours are shown in the fourth column. They successfully preserves the topological structures as prescribed by the priors. As shown in the last column, the minimum of the Jacobian determinant of the transformations are all positive, meaning that our results are indeed topology preserving. The second column shows the results obtained by the Chan-Vese model and the third column shows the results obtained by the selective model. The results by the Chan-Vese model are inaccurate with a number of topological noise. For the selective model, it can give the similar results with the proposed model but it costs more computational time and for the second case, it changes the topological structure. This example illustrate that the accuracy of the segmentation result can be significantly improved by imposing the topological prior. 

\begin{figure}[htbp]
\centering
\subfigure[$I$ and $J$ (red)]{
\includegraphics[width=1.1in,height=1.1in]{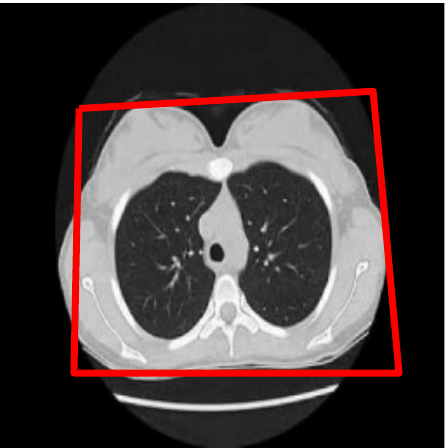}}
\subfigure[CV (1.85 sec)]{
\includegraphics[width=1.1in,height=1.1in]{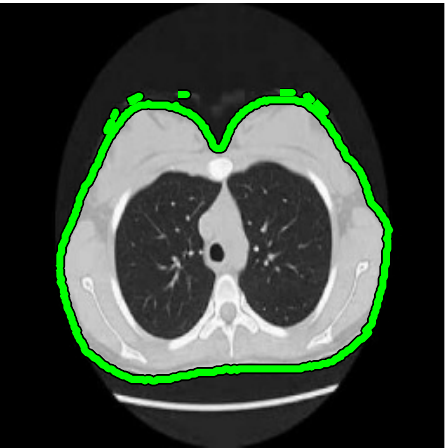}}
\subfigure[\cite{zhang2014local} (27.85 sec)]{
\includegraphics[width=1.1in,height=1.1in]{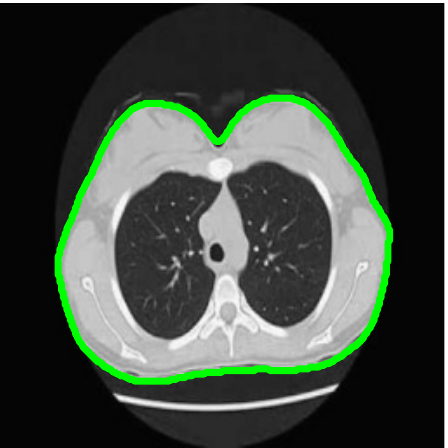}}
\subfigure[PM (8.66 sec)]{
\includegraphics[width=1.1in,height=1.1in]{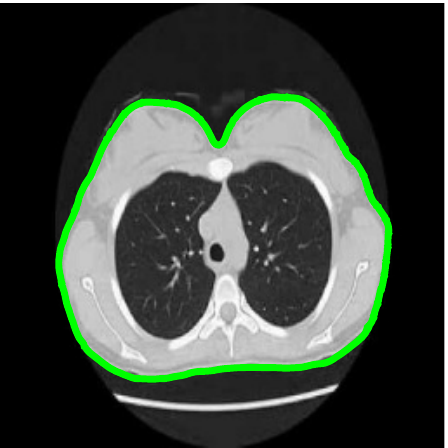}}
\subfigure[$\bm{y}$, $\det\nabla\bm{y}\in\lbrack 0.23,11.51\rbrack$]{
\includegraphics[width=1.1in,height=1.1in]{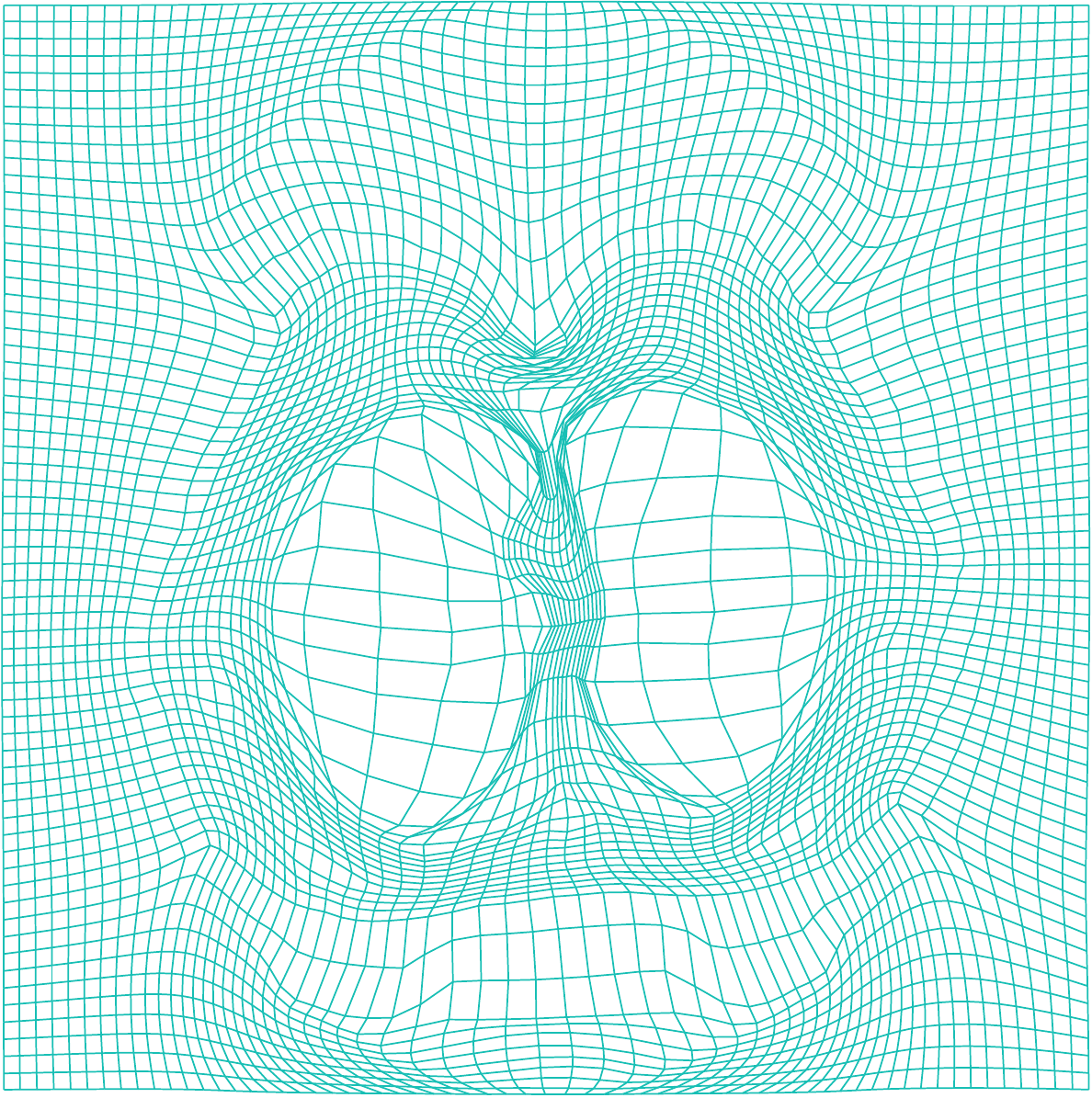}}\\
\subfigure[$I$ and $J$ (red)]{
\includegraphics[width=1.1in,height=1.1in]{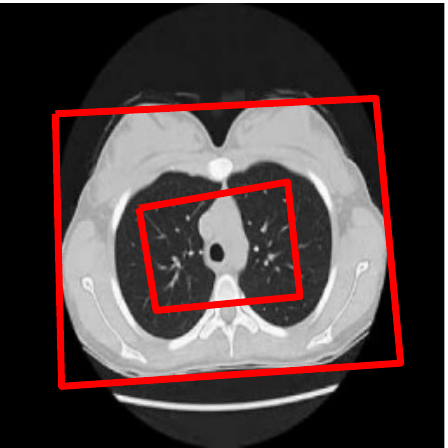}}
\subfigure[CV (2.62 sec)]{
\includegraphics[width=1.1in,height=1.1in]{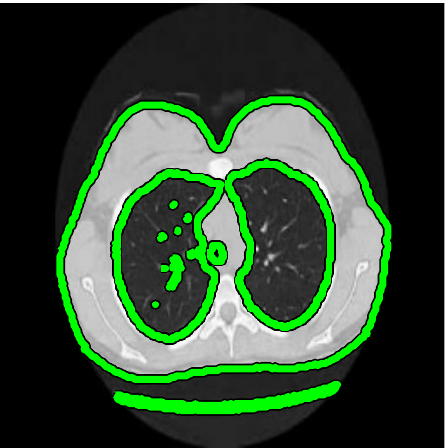}}
\subfigure[\cite{zhang2014local} (48.51 sec)]{
\includegraphics[width=1.1in,height=1.1in]{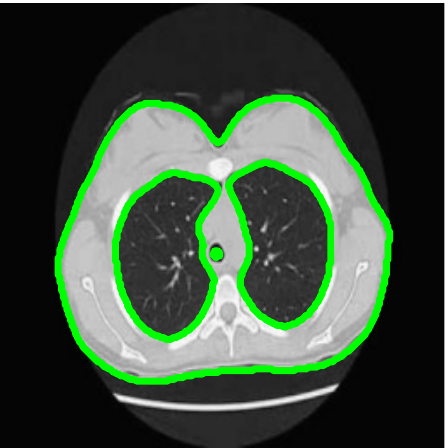}}
\subfigure[PM (4.14 sec)]{
\includegraphics[width=1.1in,height=1.1in]{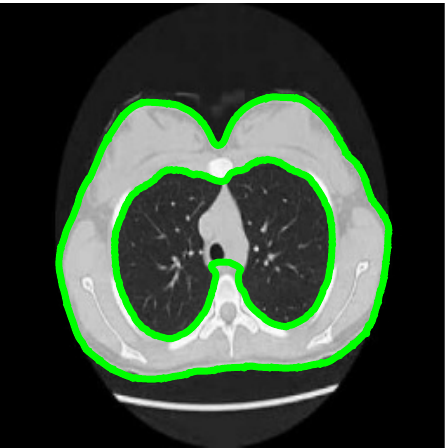}}
\subfigure[$\bm{y}$, $\det\nabla\bm{y}\in\lbrack 0.26,10.01\rbrack$]{
\includegraphics[width=1.1in,height=1.1in]{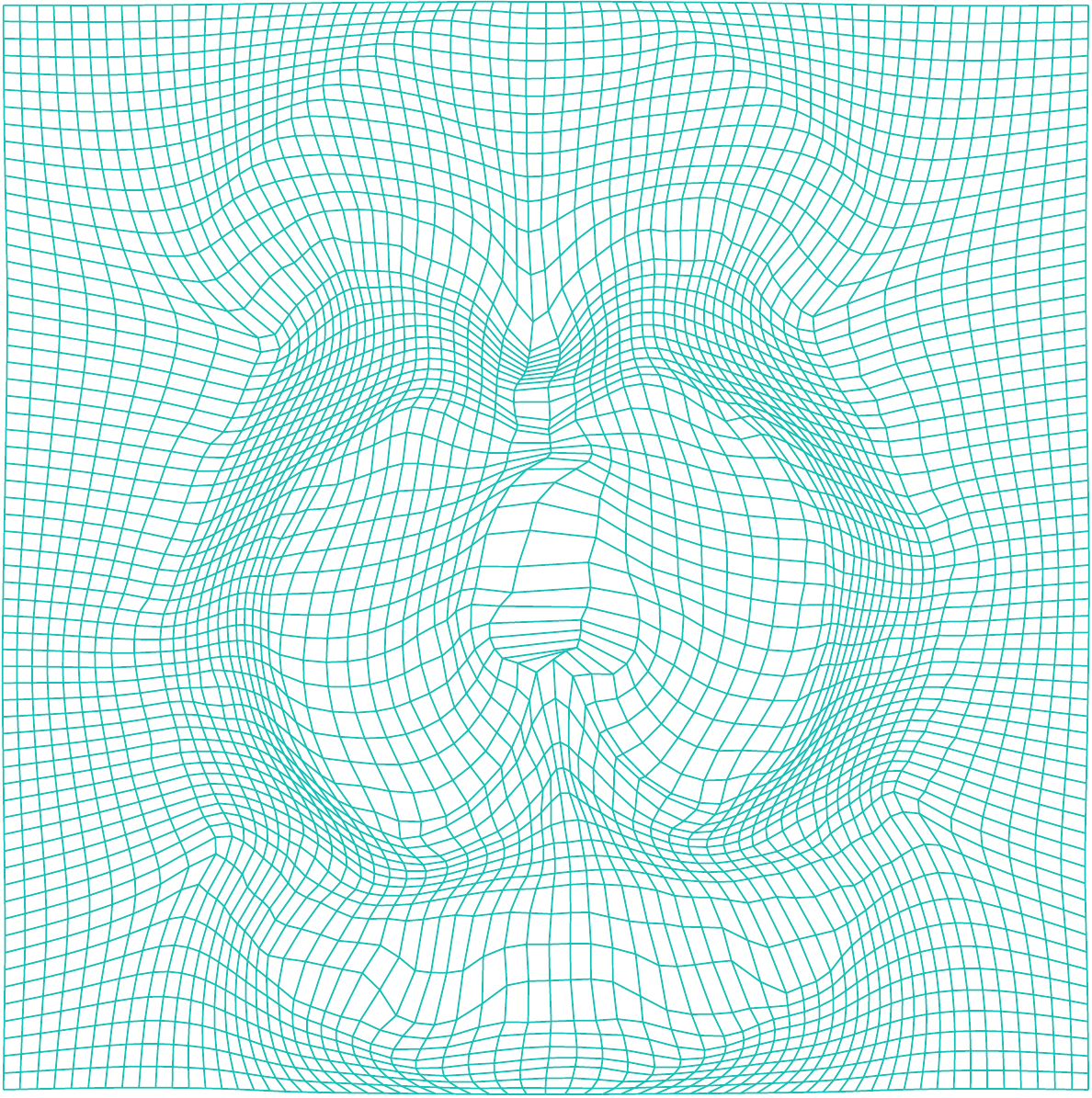}}\\
\subfigure[$I$ and $J$ (red)]{
\includegraphics[width=1.1in,height=1.1in]{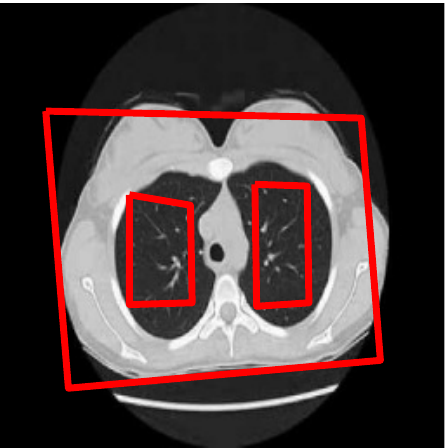}}
\subfigure[CV (2.89 sec)]{
\includegraphics[width=1.1in,height=1.1in]{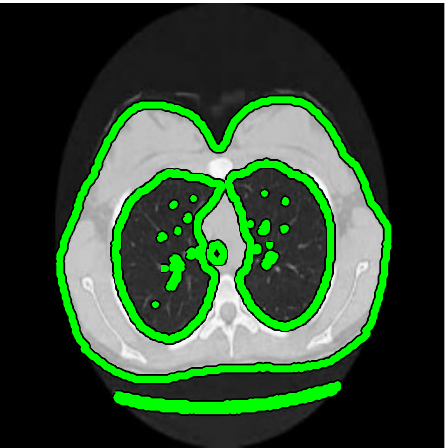}}
\subfigure[\cite{zhang2014local} (26.16 sec)]{
\includegraphics[width=1.1in,height=1.1in]{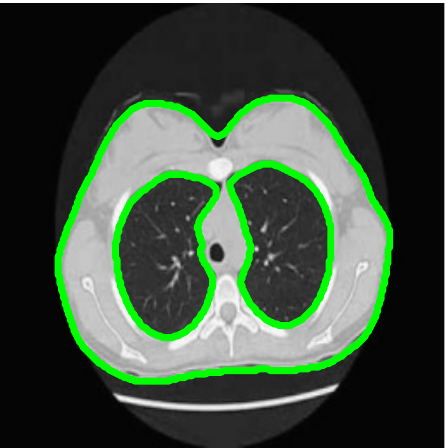}}
\subfigure[PM (6.24 sec)]{
\includegraphics[width=1.1in,height=1.1in]{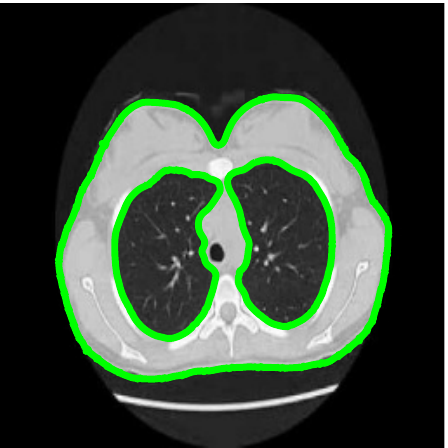}}
\subfigure[$\bm{y}$, $\det\nabla\bm{y}\in\lbrack 0.29, 12.33\rbrack$]{
\includegraphics[width=1.1in,height=1.1in]{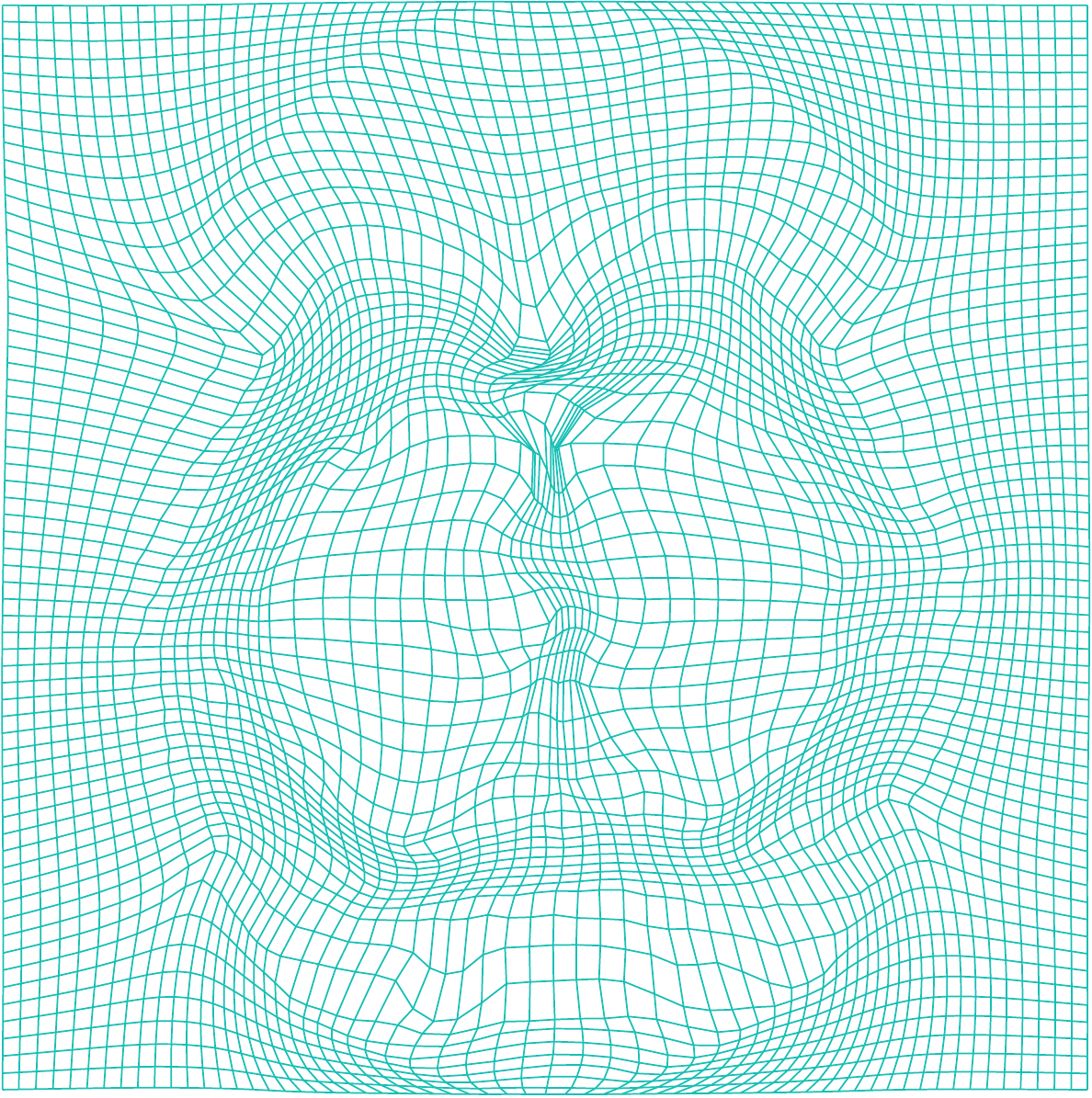}}

\caption{Results of Third 2D Example. First column: boundary of the topological prior $J$ (red) superimposed on the input image $I$. Second column: results generated by the Chan-Vese model. Third column: results generated by a selective model \cite{zhang2014local}. Fourth column: results generated by the proposed model (PM) \eqref{ProposedModel}. Fifth column: the corresponding transformations $\bm{y}$ generated by the proposed model \eqref{ProposedModel}.}\label{2DExample_3}
\end{figure}

\begin{figure}[htbp]
\centering
\includegraphics[width=4.0in,height=2.0in]{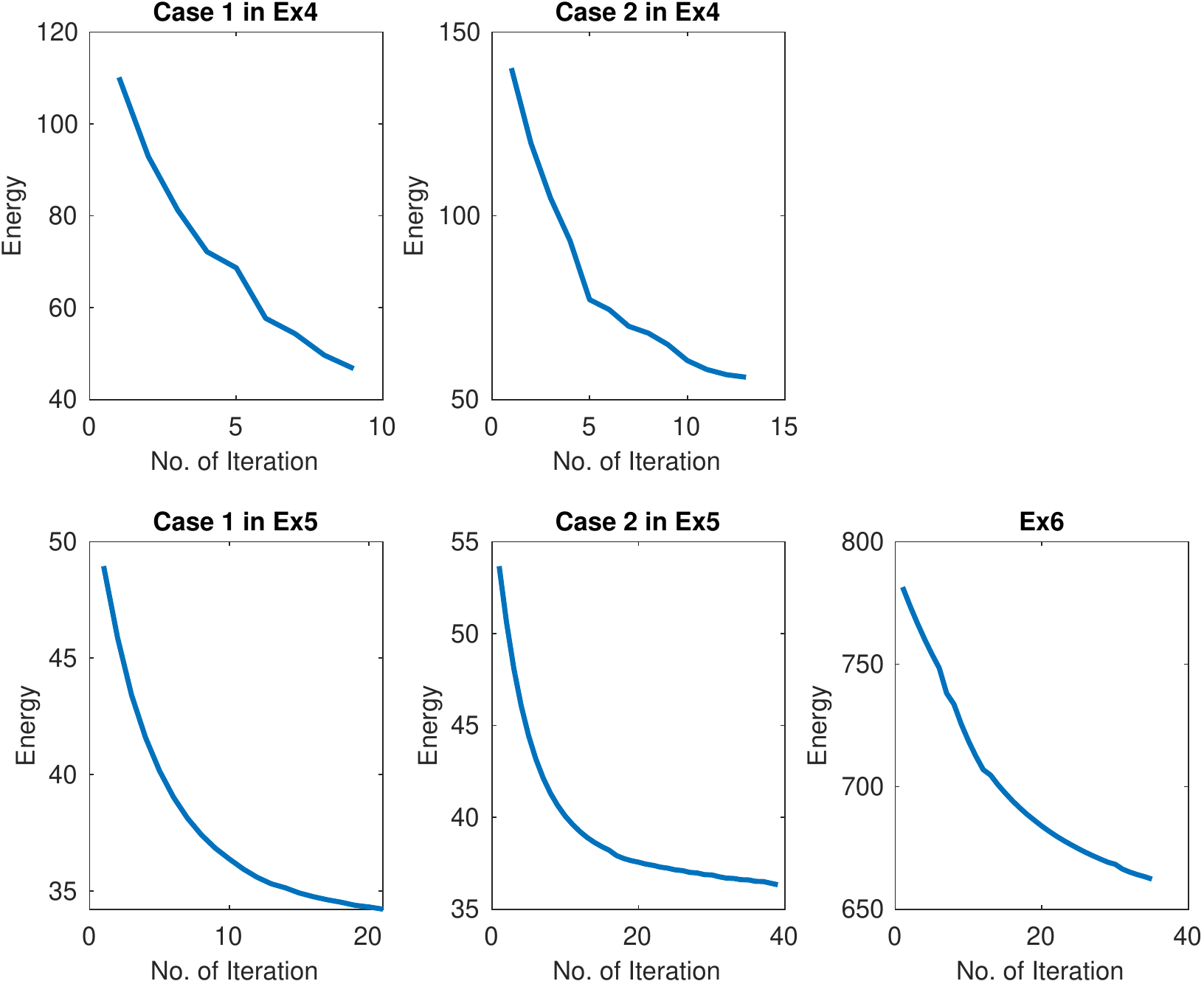}
\caption{Energy versus iterations of 3D Examples.}\label{Energy3D}
\end{figure}

\subsection{3D Example}
In this subsection, we test our proposed model on three 3D images, one synthetic image and two real images. All images are resized into $128\times128\times128$ and their intensities are rescaled into $[0, 255]$. All the real 3D images are downloaded from https://www.dir-lab.com. Except for the Chan-Vese model, we also compare our proposed model \eqref{ProposedModel} with a 3D selective model \cite{zhang2015fast}, whose code is downloaded from https://www.liverpool.ac.uk/$\sim$cmchenke/softw/select$\_$3D-2015.htm. For the parameters in the 3D case, based on our numerical experiences, $[10,10^{2}]\times[1,10] \times [1,10]$ may be the suitable range for $(\alpha_{l},\alpha_{s},\alpha_{v})$ with respect to the accuracy and computational time.

\bigskip

\noindent {\bf Example 4:} In this example, we test our proposed model on a synthetic image, which is shown in Figure \ref{Synthetic_Image_1} (a). A simple topological prior is introduced, as shown in Figure \ref{Synthetic_Image_1} (b). For the parameters of our proposed model, we set $\alpha_{l}=10^{2}$, $\alpha_{s}=10$ and $\alpha_{v}=10$. The segmentation result obtained by our proposed model is shown in (c). Our method can produce a topology-preserving segmentation result. We compare our method with the Chan-Vese model and the selective model by using the default parameters. Their segmentation results are shown in (d,e). Here, we can see that both of the two models can also produce topology-preserving segmentation results.    

\begin{figure}[htbp]
\centering
\subfigure[Target Image]{
\includegraphics[width=1.9in,height=1.5in]{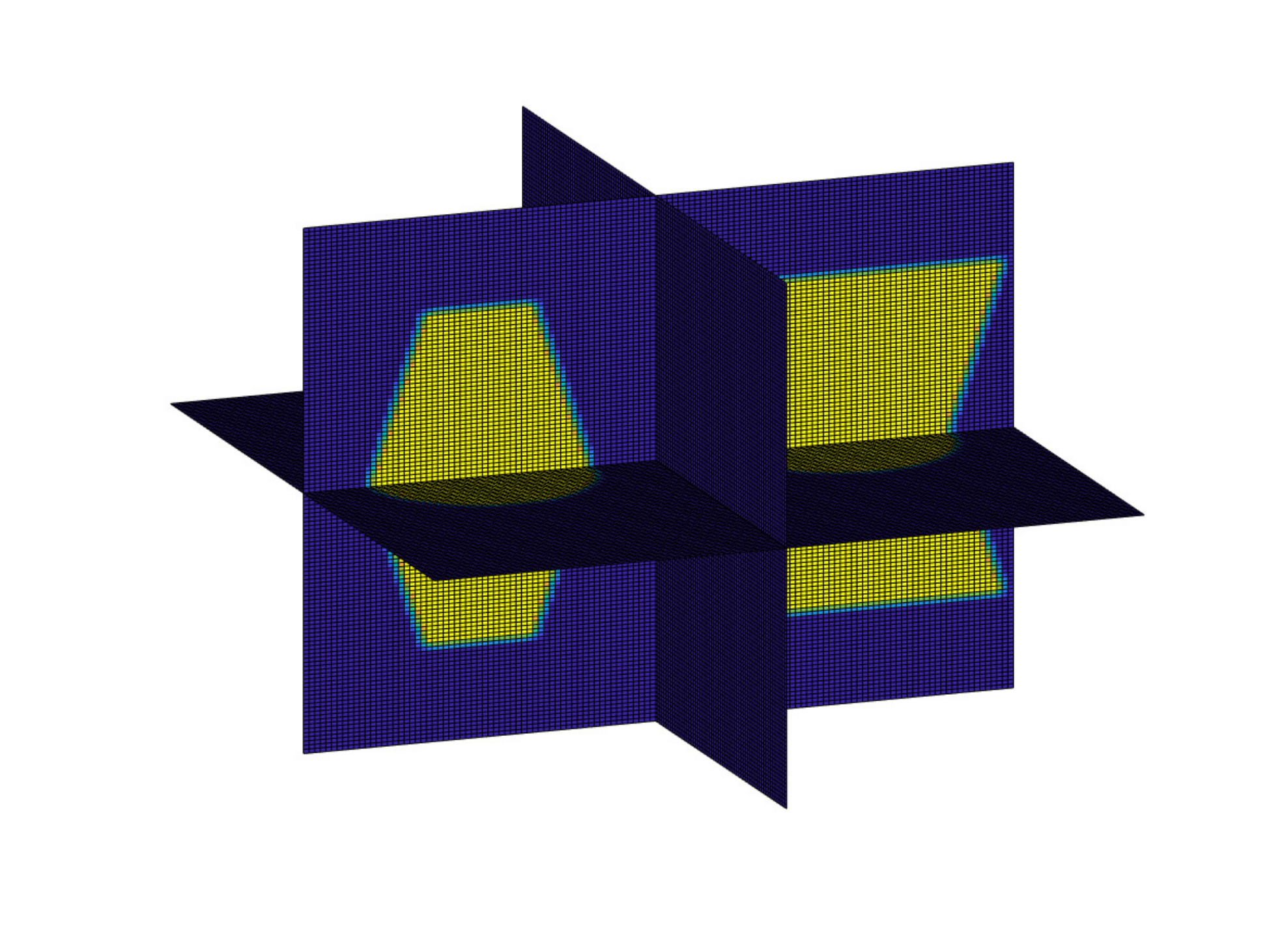}}
\subfigure[Prior Image]{
\includegraphics[width=1.9in,height=1.5in]{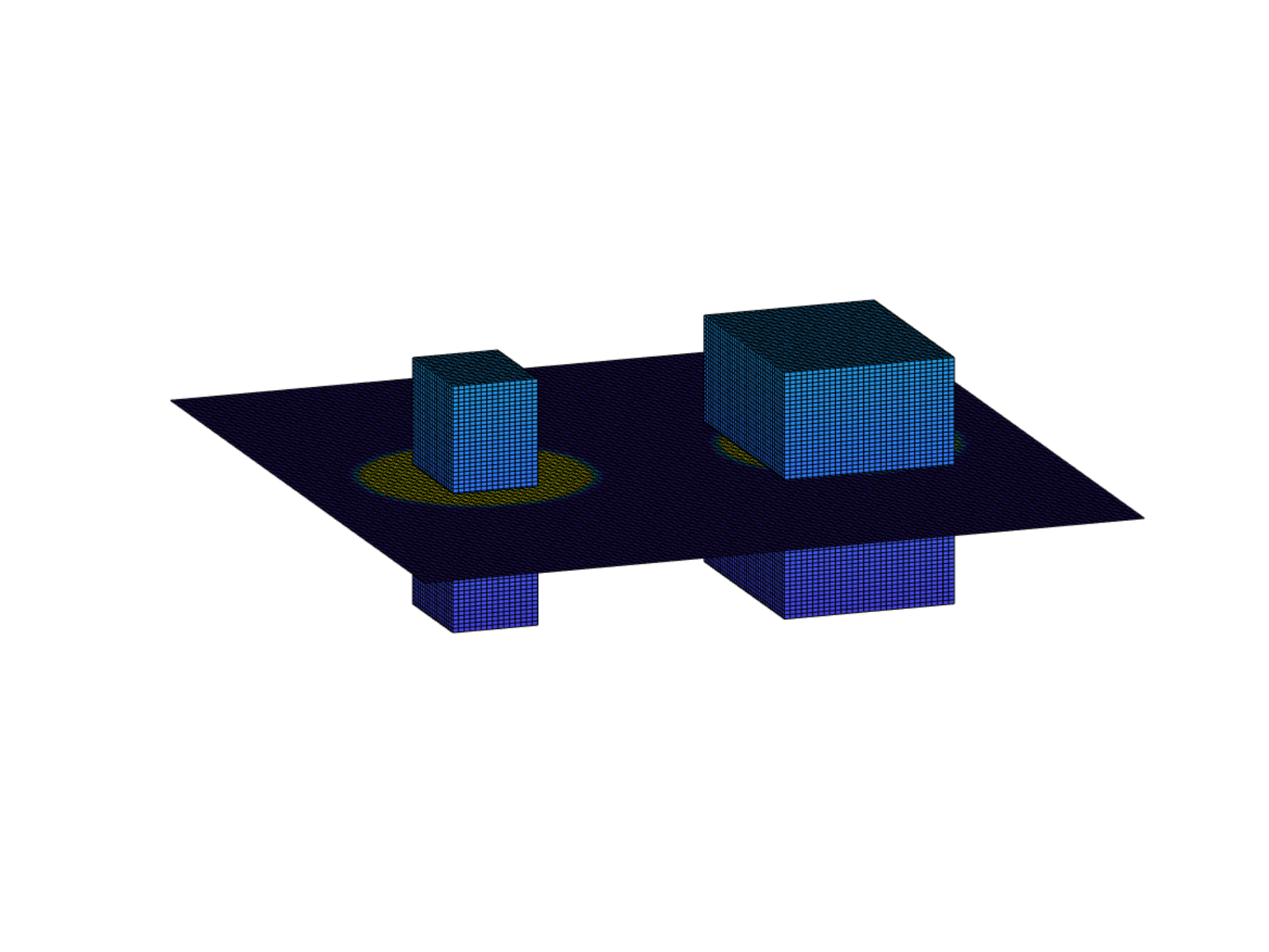}}\\
\subfigure[PM (276.64 sec)]{
\includegraphics[width=1.9in,height=1.5in]{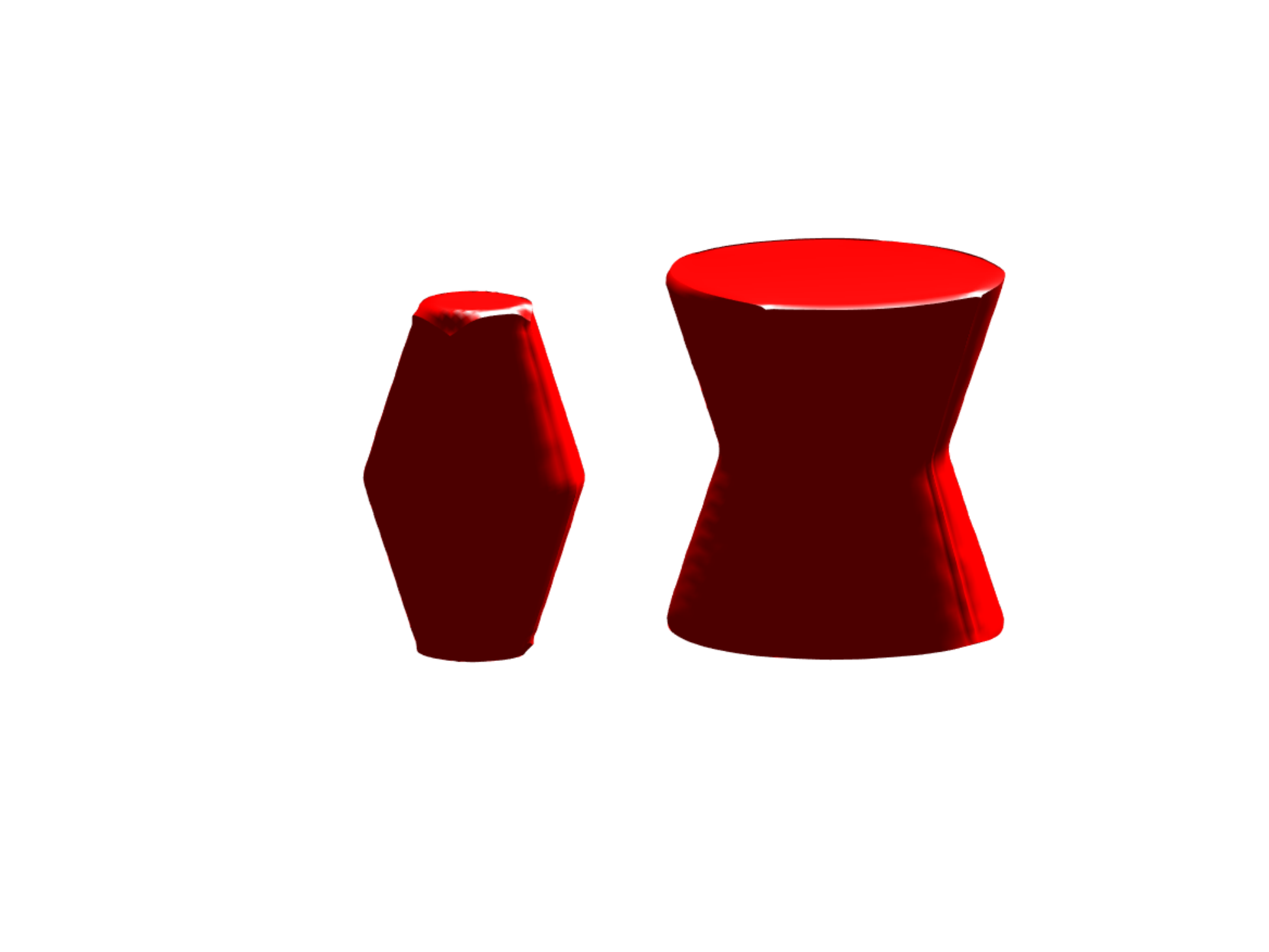}}
\subfigure[CV (83.84 sec)]{
\includegraphics[width=1.9in,height=1.5in]{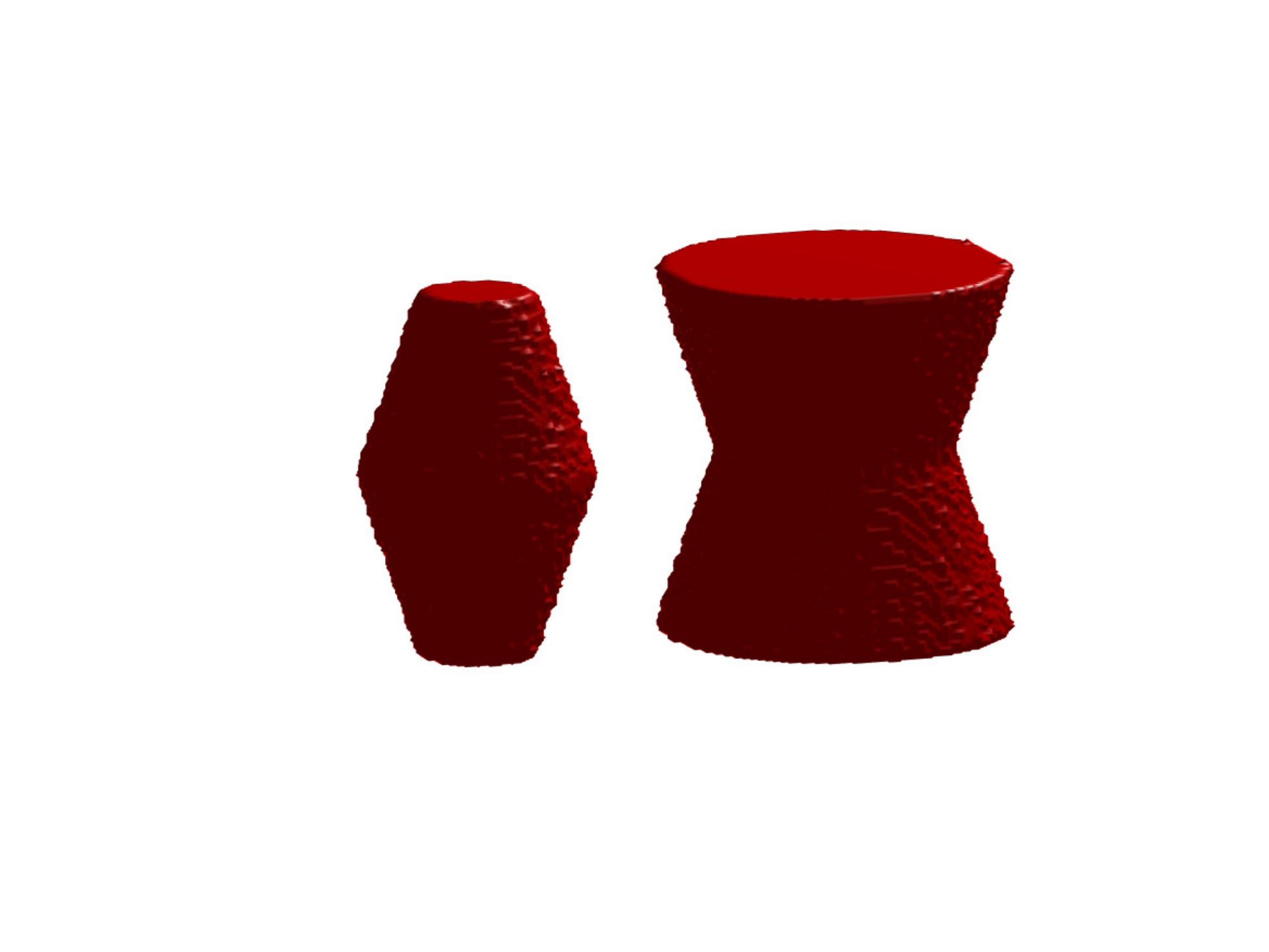}}
\subfigure[\cite{zhang2015fast} (125.46 sec)]{
\includegraphics[width=1.9in,height=1.5in]{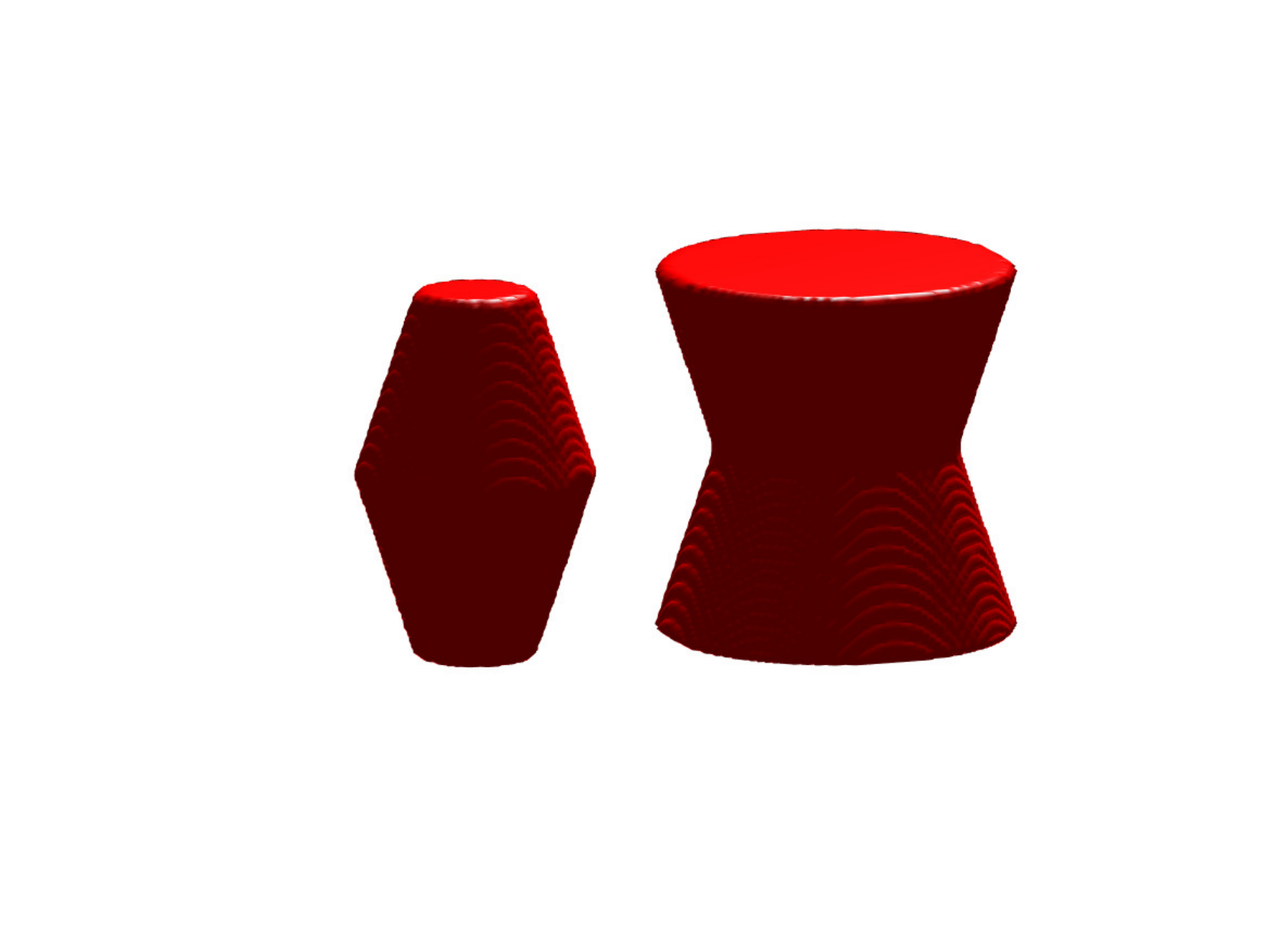}}
\caption{Here, we show the target image and prior image of the first 3D example in the first row. The second row shows the segmentation results by the proposed model \eqref{ProposedModel}, the Chan-Vese model and the selective model \cite{zhang2015fast}, respectively.}\label{Synthetic_Image_1}
\end{figure}

Next, we change the intensity value of the central part of the synthetic image to $0$ as shown in Figure \ref{Synthetic_Image_2} (a,b). We apply our proposed model and the other two models with the same parameters and topological prior as the above case on this degraded image. From (c,f), (d,g) and (e,h), we observe that our model can again generate a topology-preserving segmentation result but the others cannot.

\begin{figure}[htbp]
\centering
\subfigure[Target Image]{
\includegraphics[width=1.9in,height=1.5in]{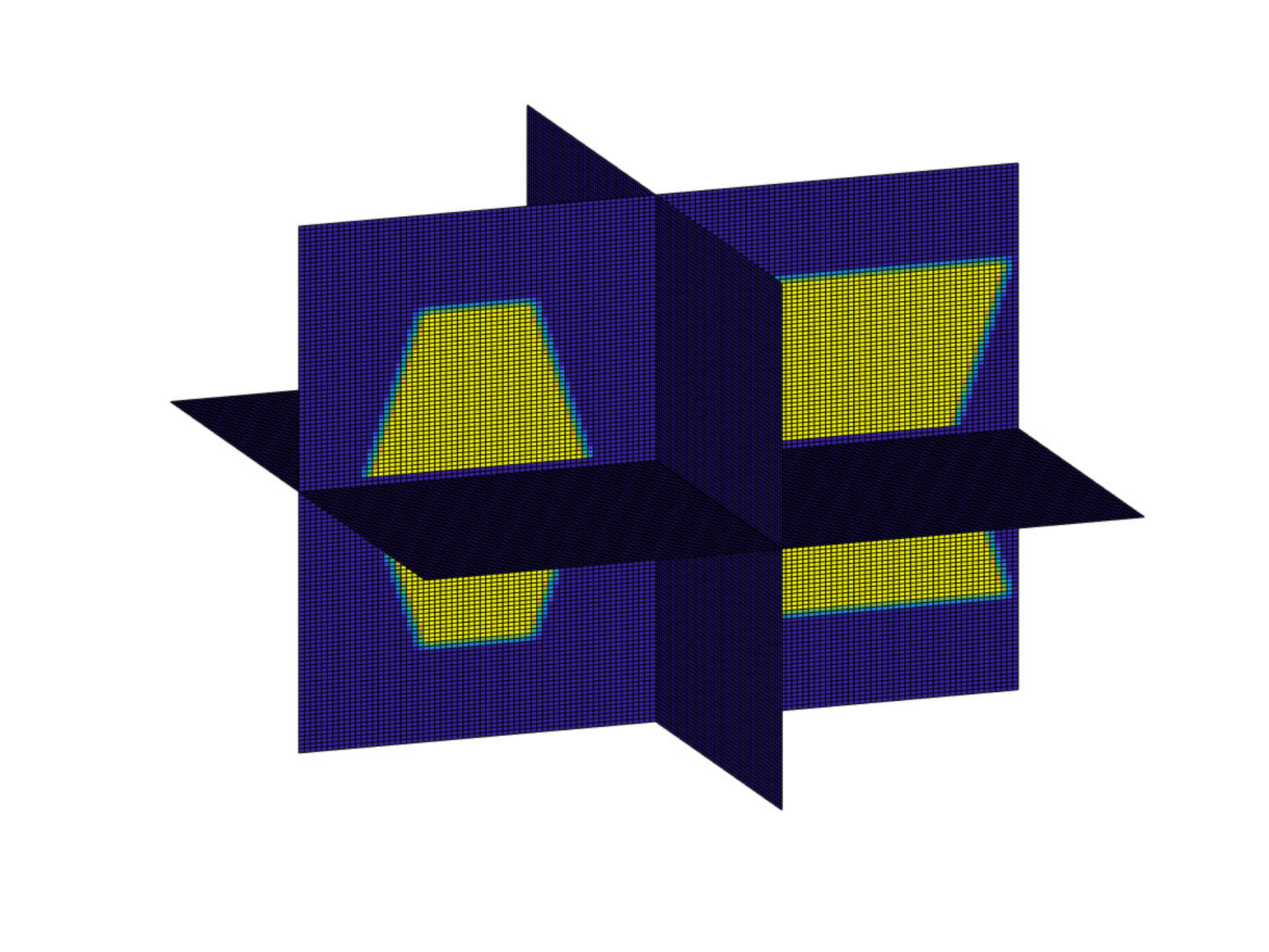}}
\subfigure[Target Image]{
\includegraphics[width=1.9in,height=1.5in]{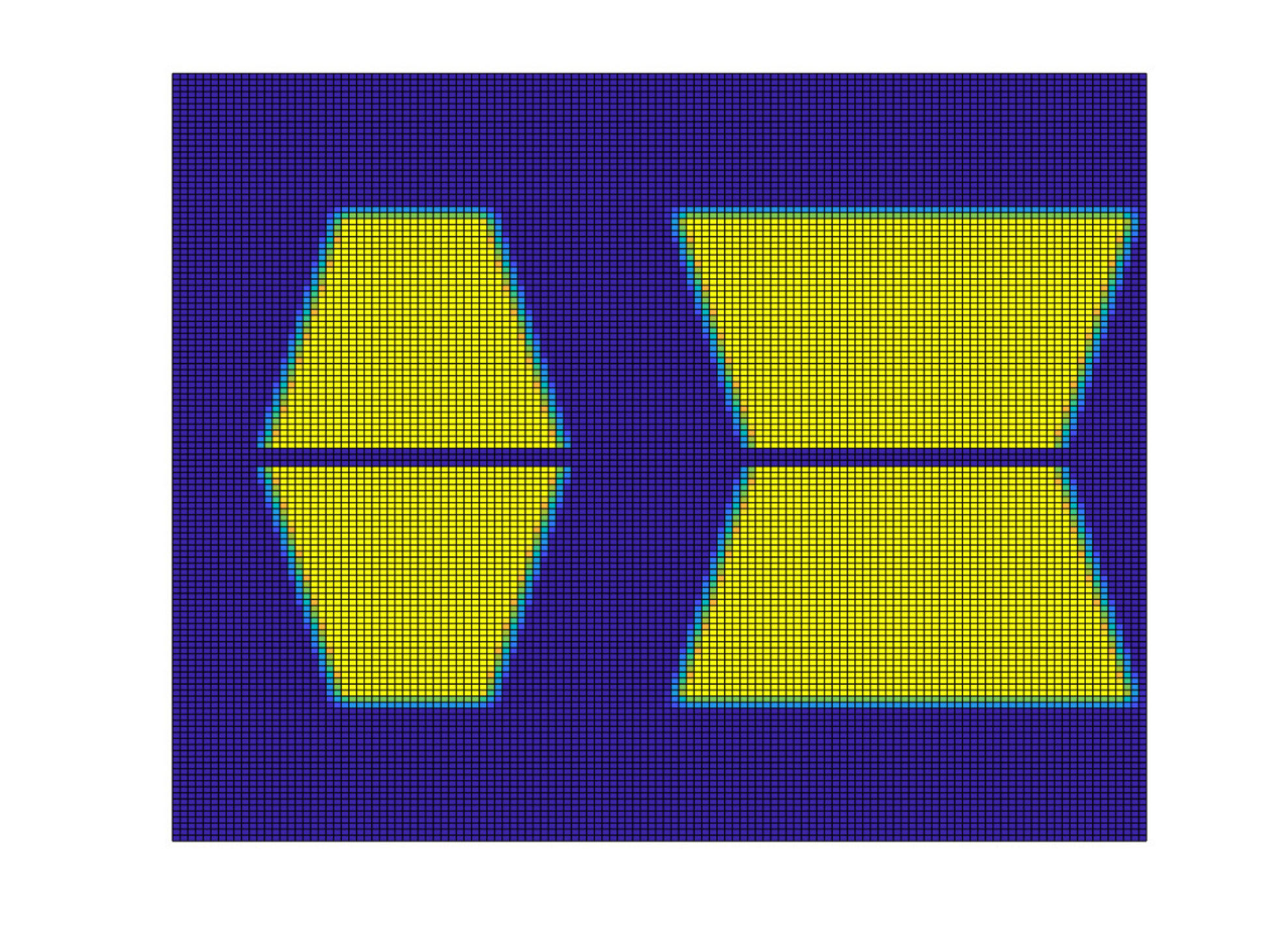}}\\
\subfigure[PM (433.07 sec)]{
\includegraphics[width=1.9in,height=1.5in]{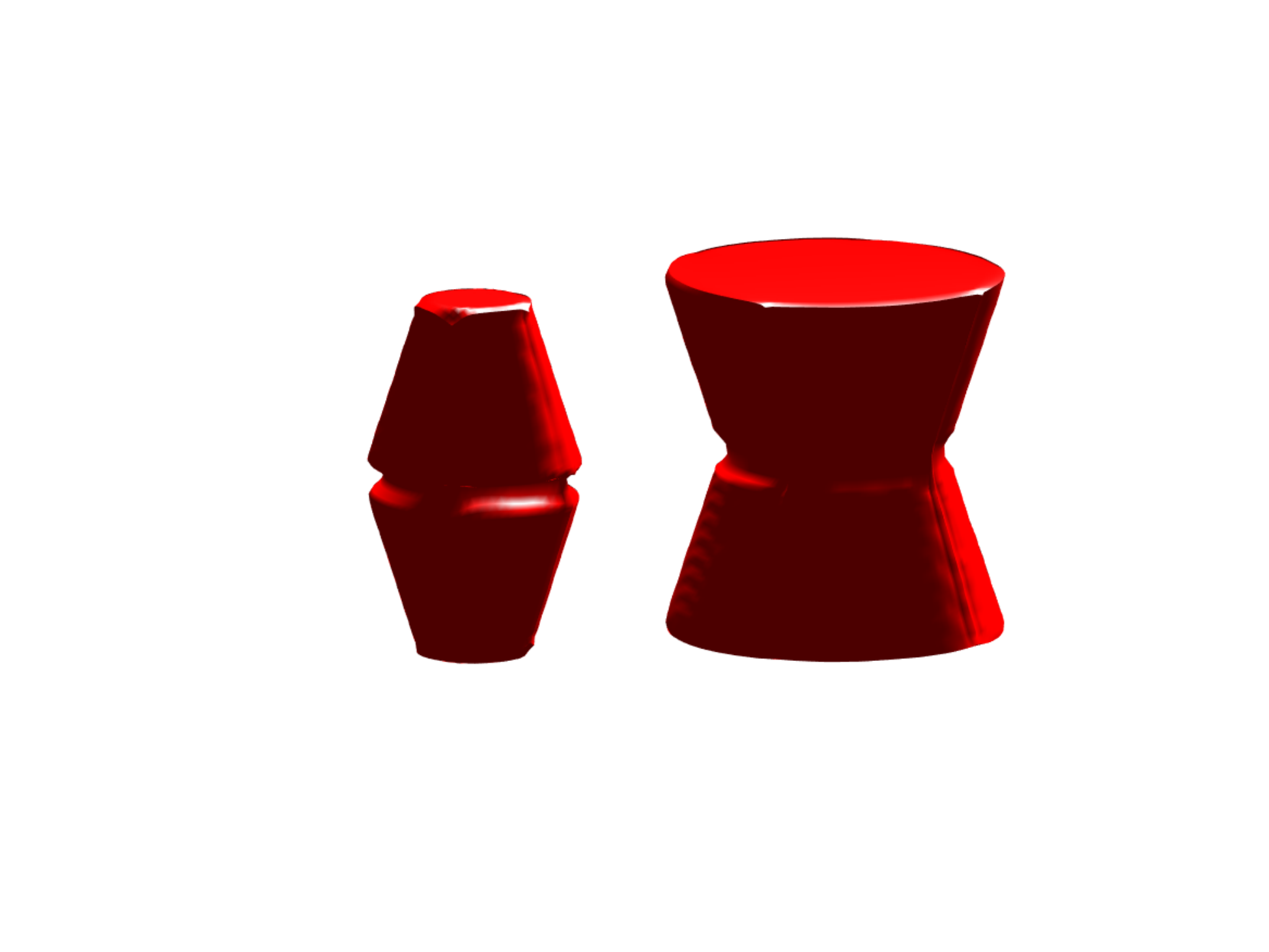}}
\subfigure[CV (92.55 sec)]{
\includegraphics[width=1.9in,height=1.5in]{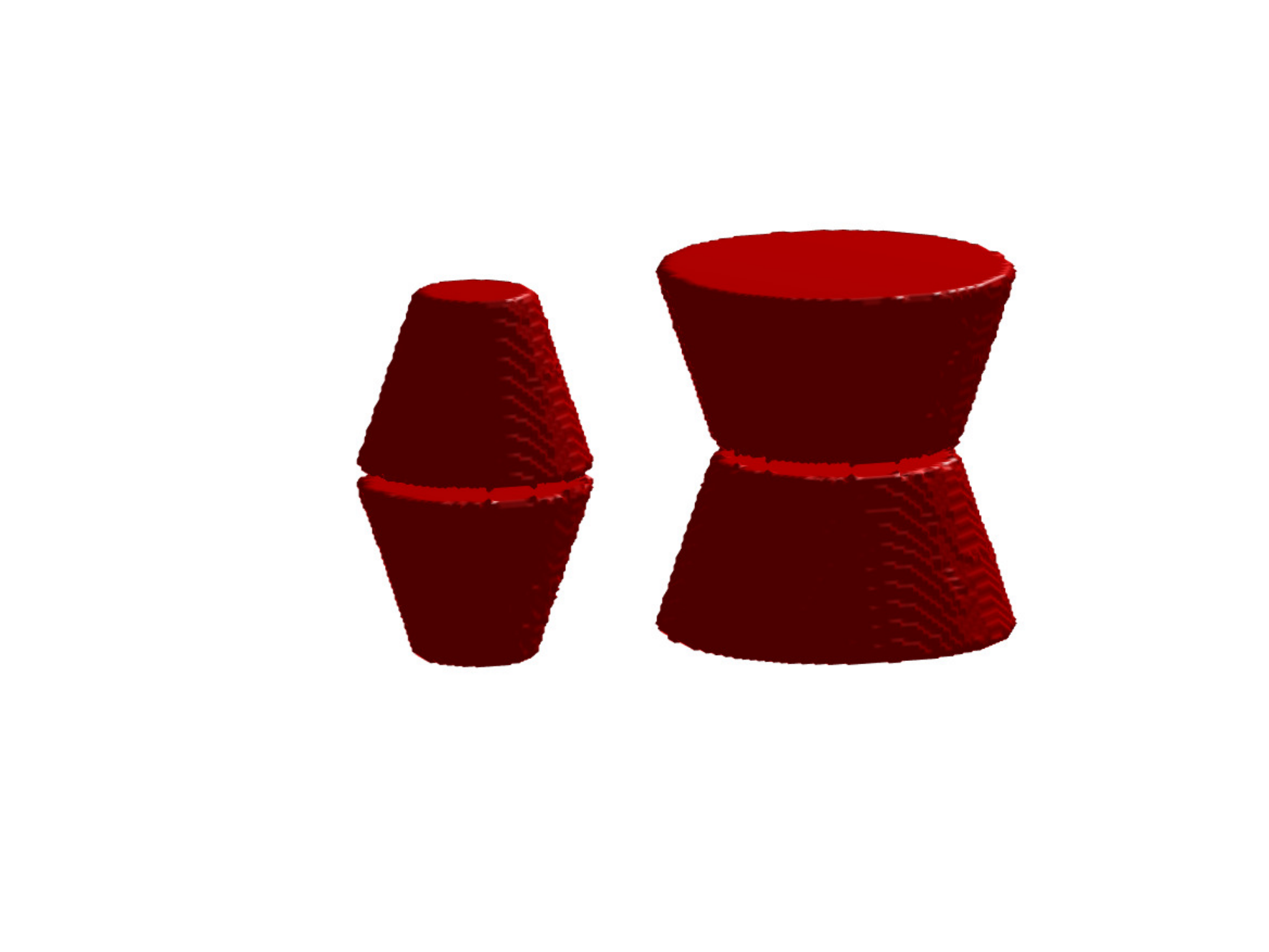}}
\subfigure[\cite{zhang2015fast} (129.80 sec)]{
\includegraphics[width=1.9in,height=1.5in]{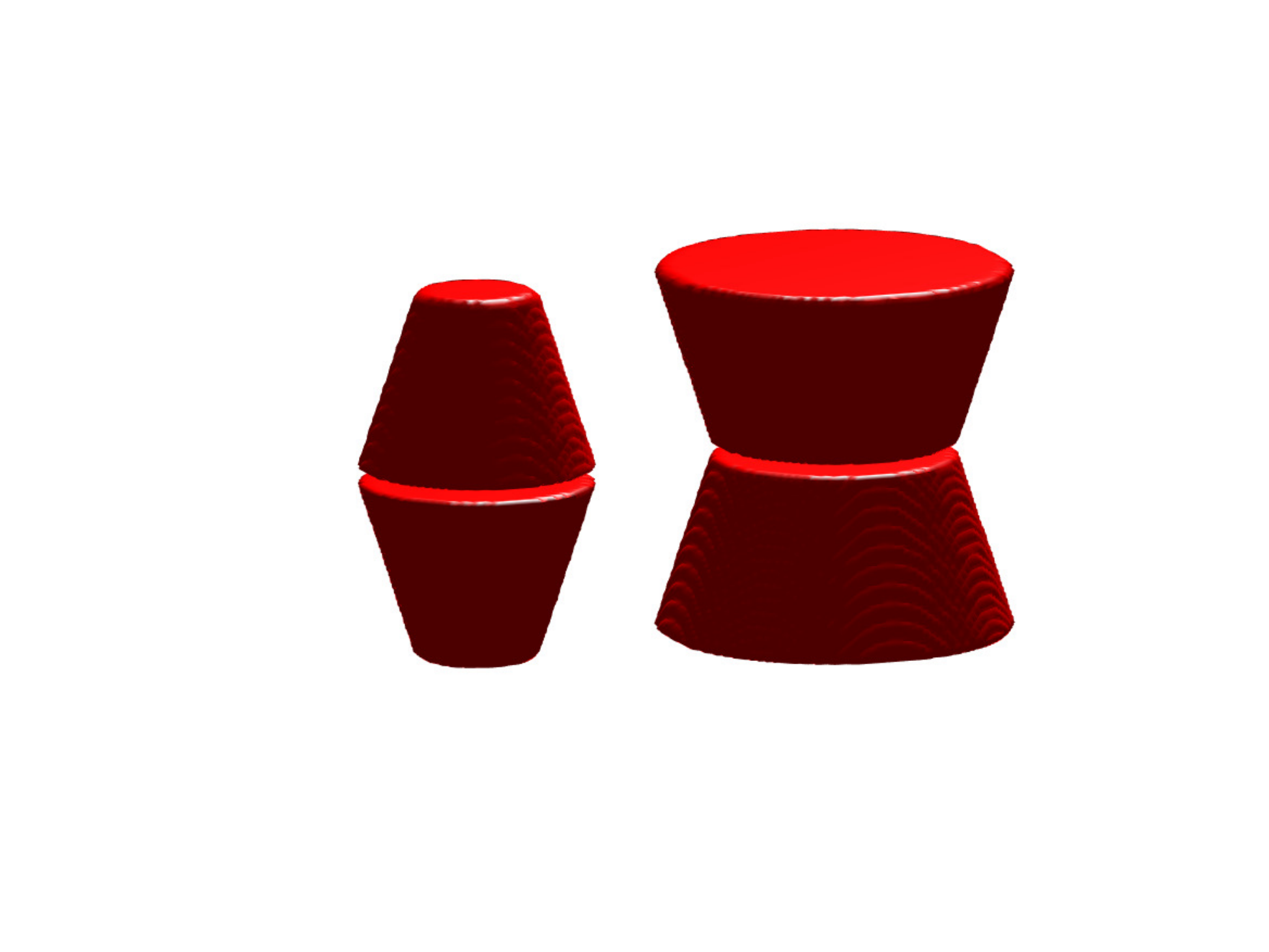}}\\
\subfigure[PM]{
\includegraphics[width=1.9in,height=1.5in]{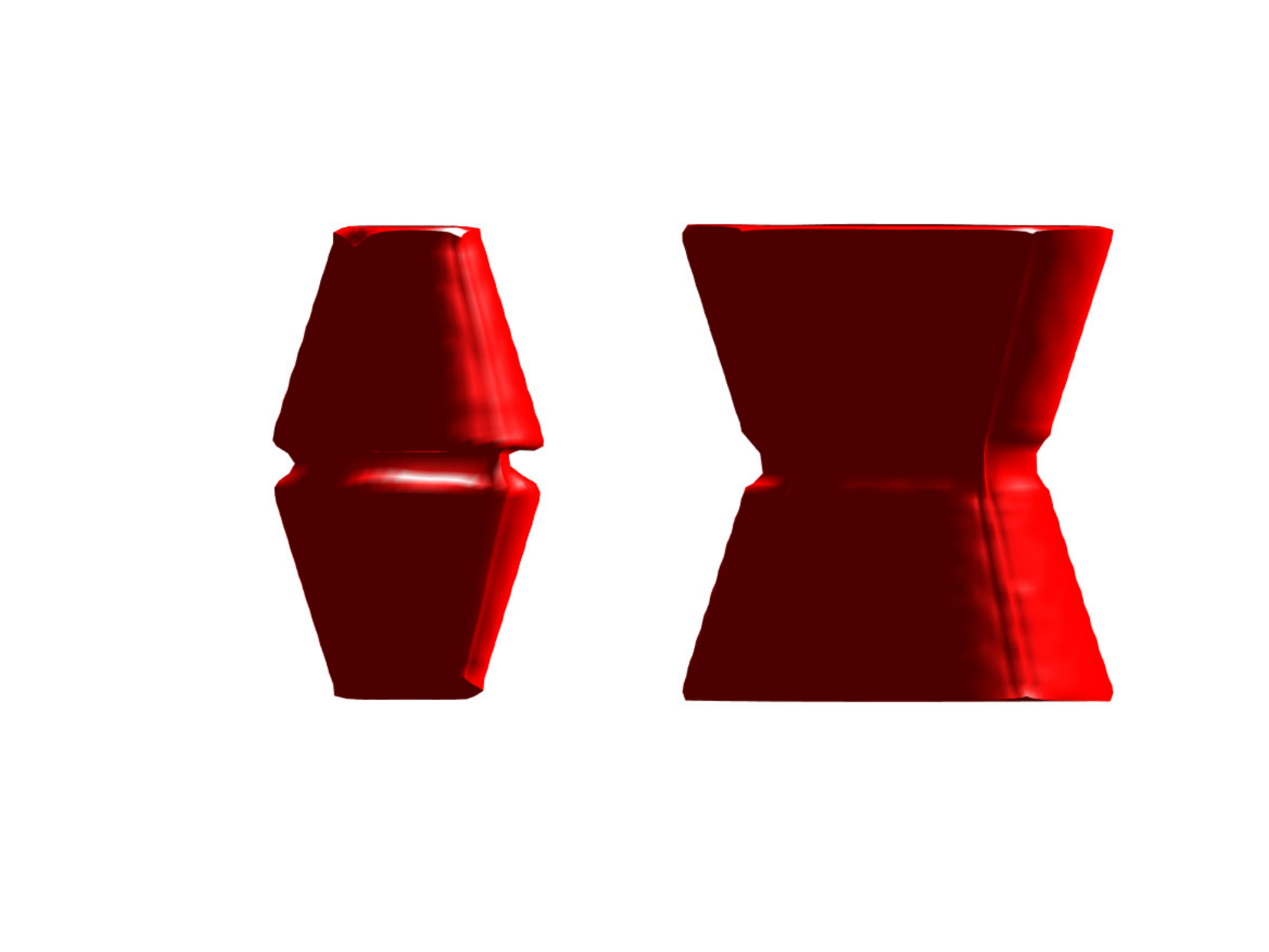}}
\subfigure[CV]{
\includegraphics[width=1.9in,height=1.5in]{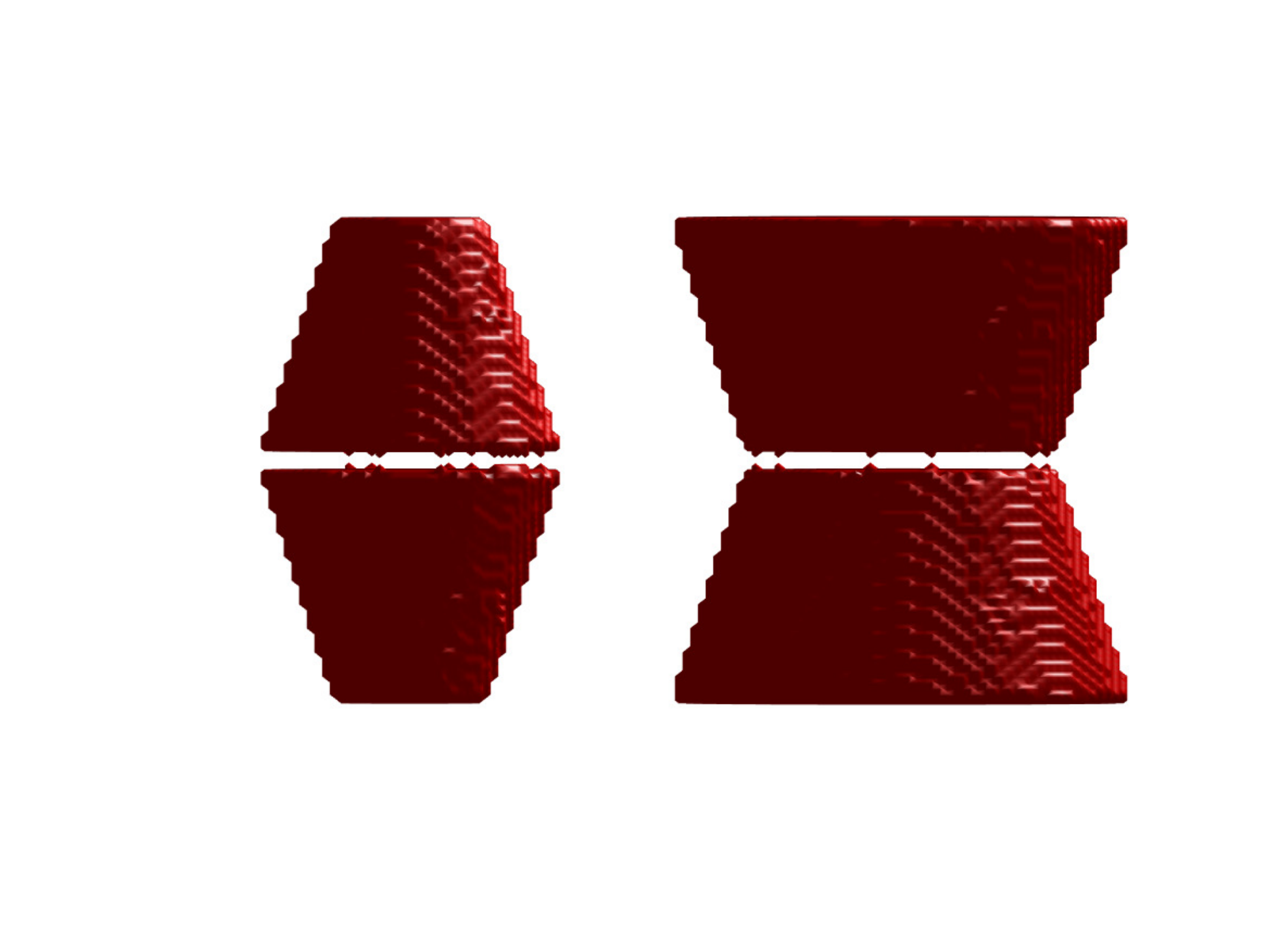}}
\subfigure[\cite{zhang2015fast}]{
\includegraphics[width=1.9in,height=1.5in]{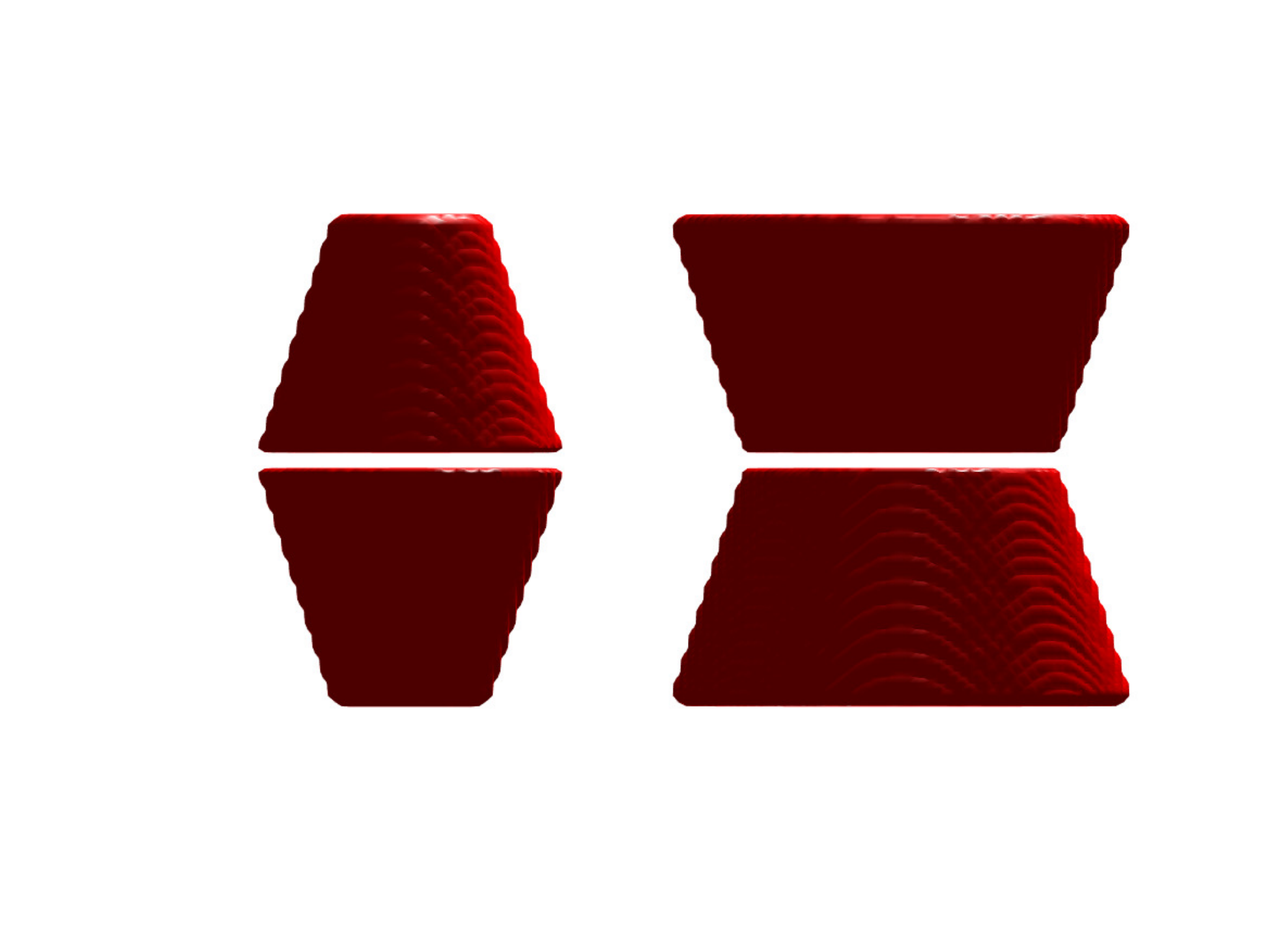}}
\caption{Here, we show the target images in the first row for the second case of the first 3D example. The second and third rows show the segmentation results by the proposed model \eqref{ProposedModel}, the Chan-Vese model and the selective model \cite{zhang2015fast} from different angles, respectively.}\label{Synthetic_Image_2}
\end{figure}

\bigskip

\noindent {\bf Example 5:} In this example, we test our proposed model on a 3D lung CT scans, which are slices of a 3D lung as shown in Figure \ref{3DTarget_1}. According to the topological structure of a human lung, we prescribe a simple topological prior as shown in Figure \ref{3DResult1} (b). The two cuboids give the prior of the lungs. For the parameters of our proposed model \eqref{ProposedModel}, we set $\alpha_{l}=10$, $\alpha_{s}=1$ and $\alpha_{v}=1$. Again, we compare our method with the Chan-Vese segmentation model and a selective model \cite{zhang2015fast} using the default parameter. The segmentation result obtained by our proposed model is shown in (c-e). Our method can produce a topology-preserving segmentation result. With the help of the hyperelastic regularizer, the smoothness of the segmentation result obtained by the proposed model can also be guaranteed. The segmentation results obtained by the Chan-Vese model and the selective model \cite{zhang2015fast} are shown in (f-k). Here, as a global segmentation method, except for the lung, the Chan-Vese model also segments many other parts. To see the inner part clearly, we manually modify the results by the Chan-Vese model to remove the outer outliner, which are shown in (l-n). Obviously, the Chan-Vese model cannot give a topology-preserving result even with the manual modification. For the selective model, from (i-k), we first see that it can not preserve the topological structure and second, the obtained results just segment the outline of the lung and are not accurate enough . For the reason of the latter phenomenon, it is possible that the prior is not good enough and not close to the target objects. Nevertheless, our method can pick up the corresponding segmentation result according to this prior. It again demonstrates the advantage of introducing topological prior to enhance the accuracy of the segmentation result.


\begin{figure}[htbp]
\centering
\subfigure[12nd Slice]{
\includegraphics[width=0.8in,height=0.8in]{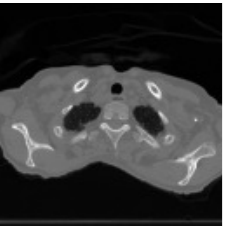}}
\subfigure[22nd Slice]{
\includegraphics[width=0.8in,height=0.8in]{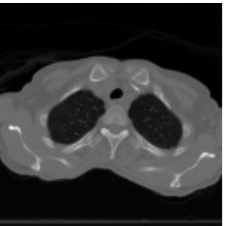}}
\subfigure[32nd Slice]{
\includegraphics[width=0.8in,height=0.8in]{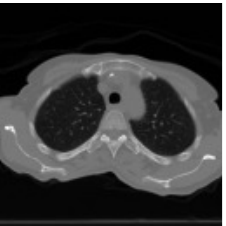}}
\subfigure[42nd Slice]{
\includegraphics[width=0.8in,height=0.8in]{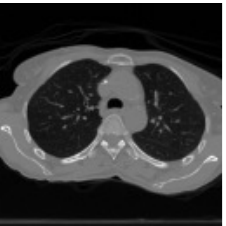}}
\subfigure[52nd Slice]{
\includegraphics[width=0.8in,height=0.8in]{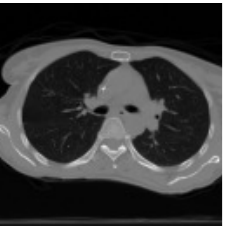}}\\
\subfigure[62nd Slice]{
\includegraphics[width=0.8in,height=0.8in]{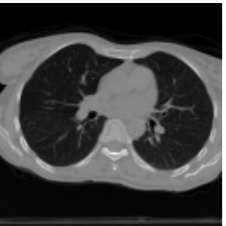}}
\subfigure[72nd Slice]{
\includegraphics[width=0.8in,height=0.8in]{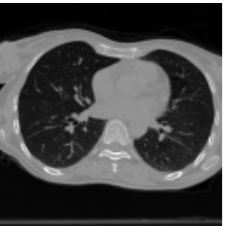}}
\subfigure[82nd Slice]{
\includegraphics[width=0.8in,height=0.8in]{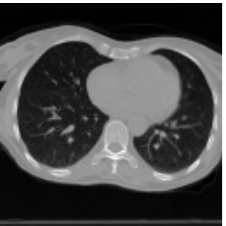}}
\subfigure[92nd Slice]{
\includegraphics[width=0.8in,height=0.8in]{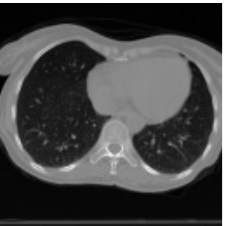}}
\subfigure[102nd Slice]{
\includegraphics[width=0.8in,height=0.8in]{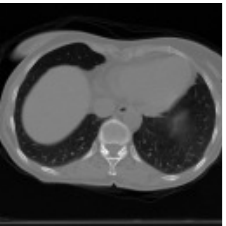}}
\caption{Some slices of the second 3D example.}\label{3DTarget_1}
\end{figure}

\begin{figure}[htbp]
\centering
\subfigure[Target Image]{
\includegraphics[width=1.9in,height=1.5in]{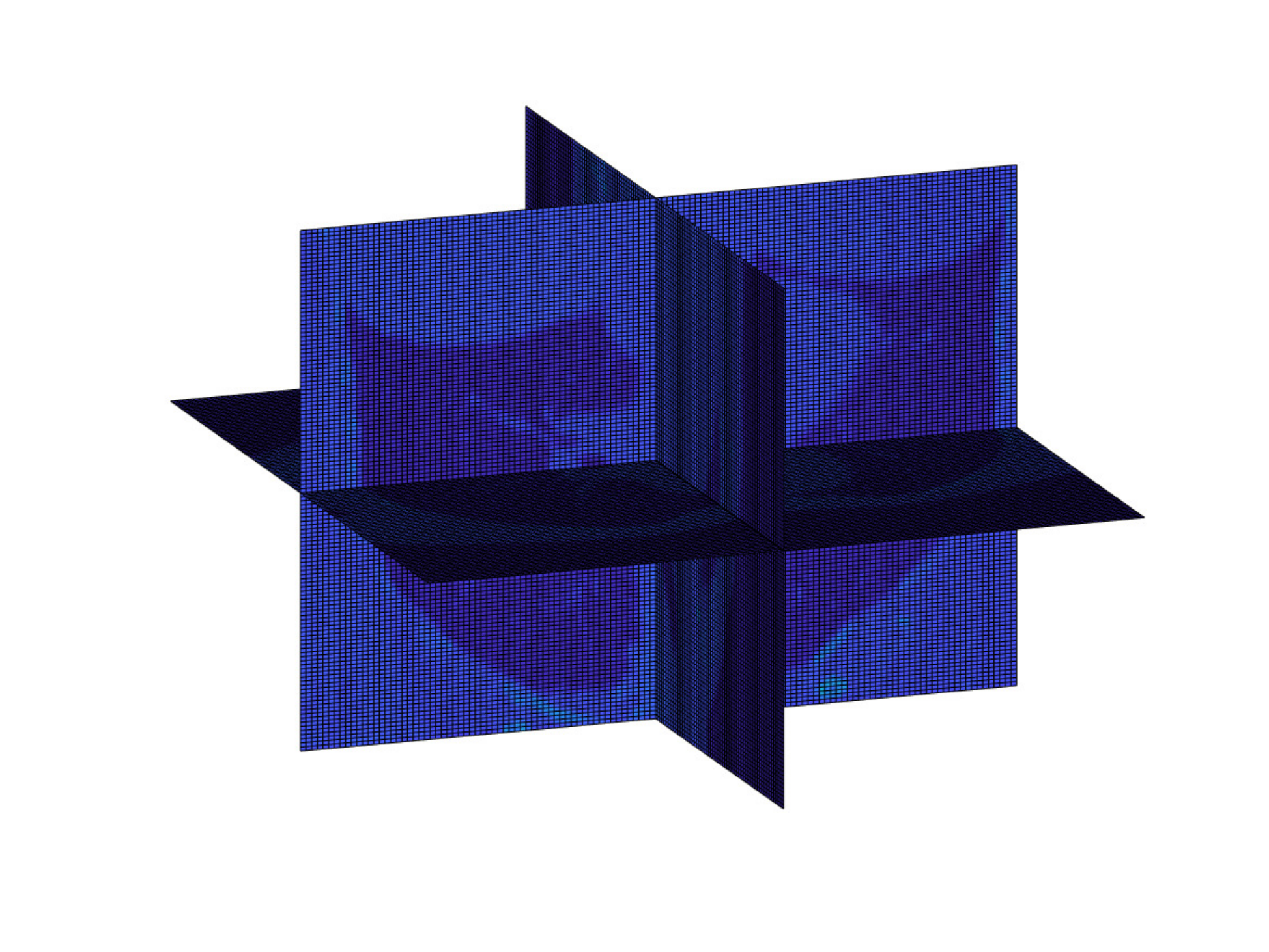}}
\subfigure[Prior Image]{
\includegraphics[width=1.9in,height=1.5in]{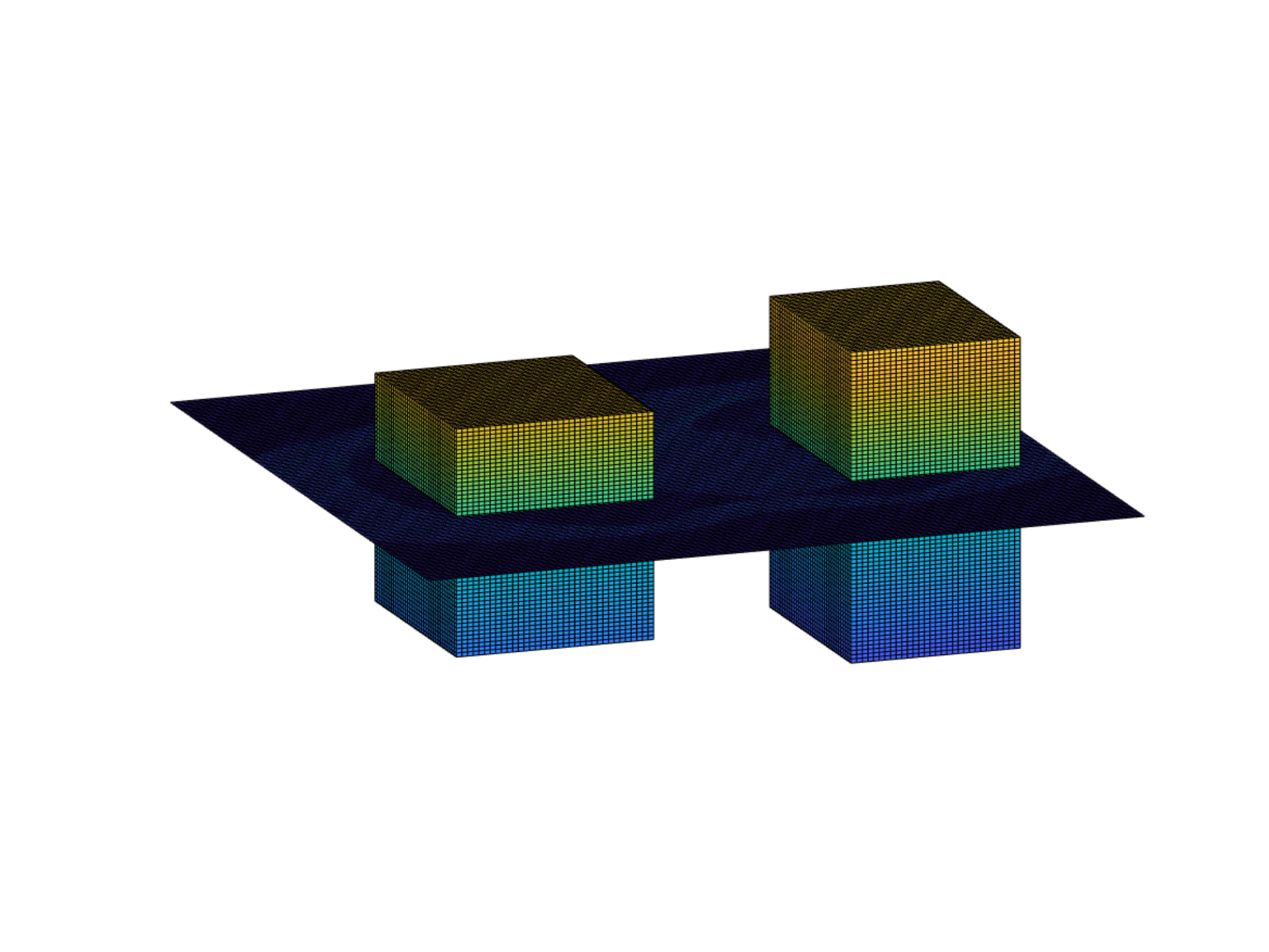}}\\
\subfigure[PM (861.23 sec)]{
\includegraphics[width=1.9in,height=1.5in]{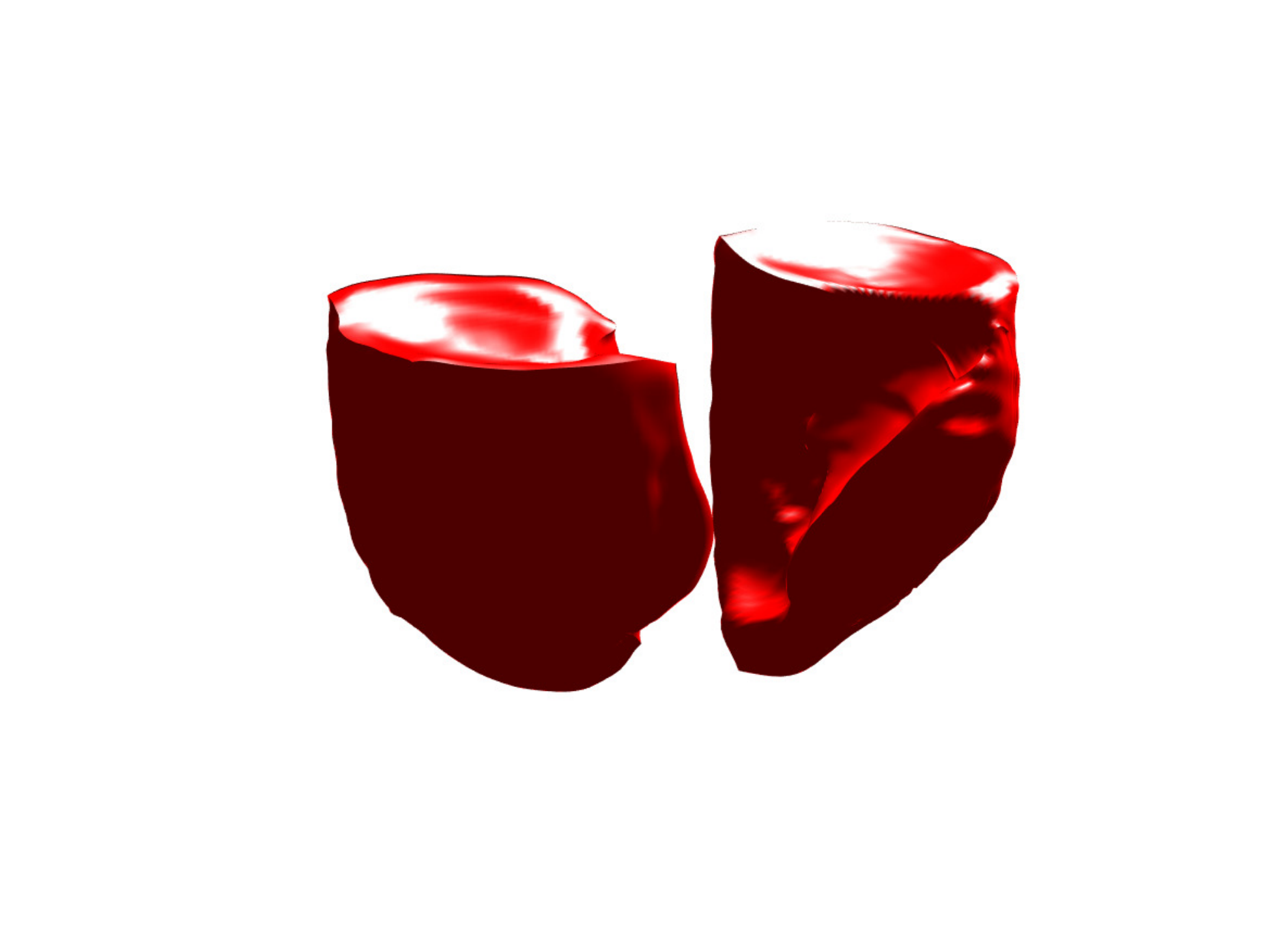}}
\subfigure[PM]{
\includegraphics[width=1.9in,height=1.5in]{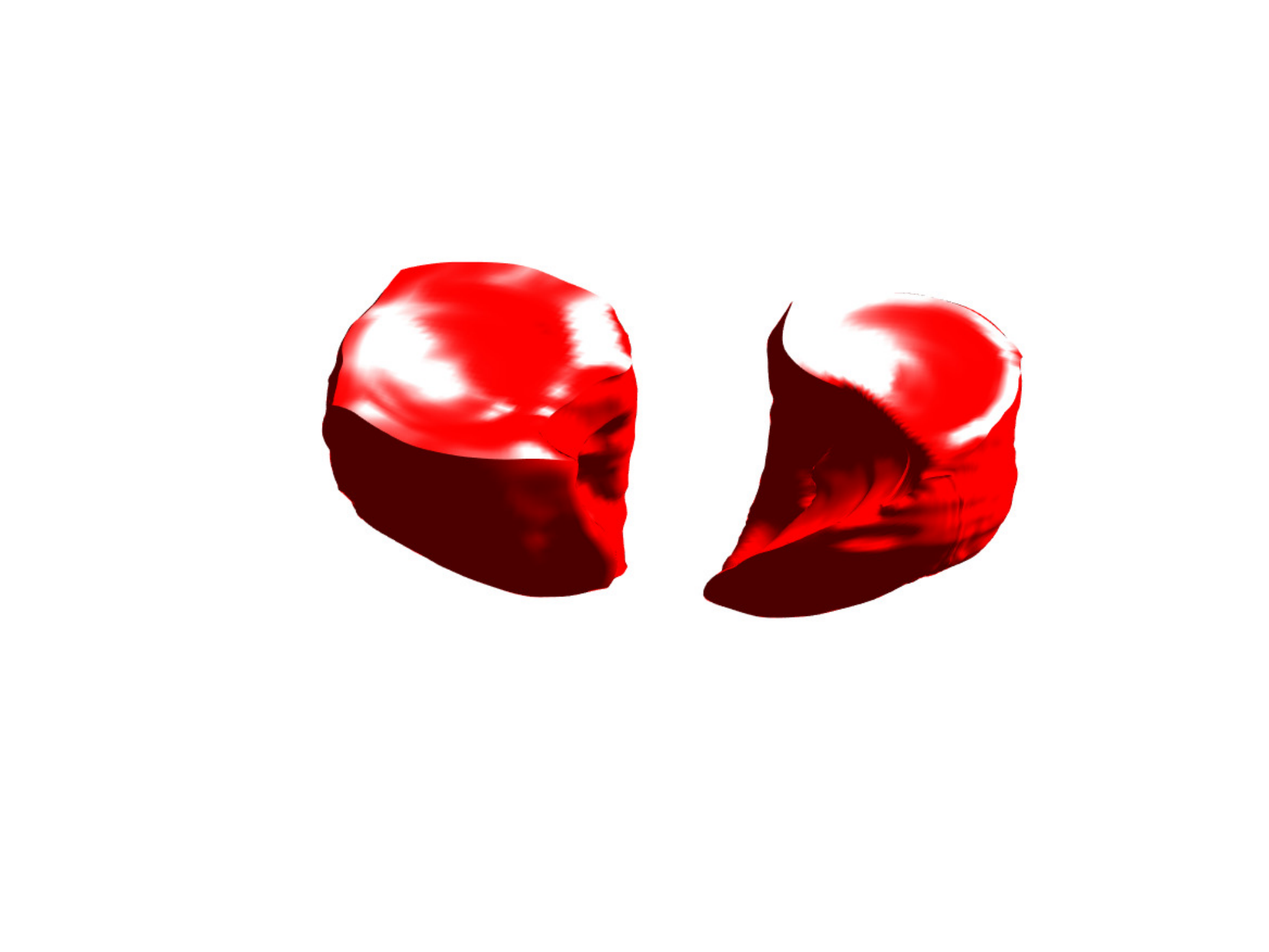}}
\subfigure[PM]{
\includegraphics[width=1.9in,height=1.5in]{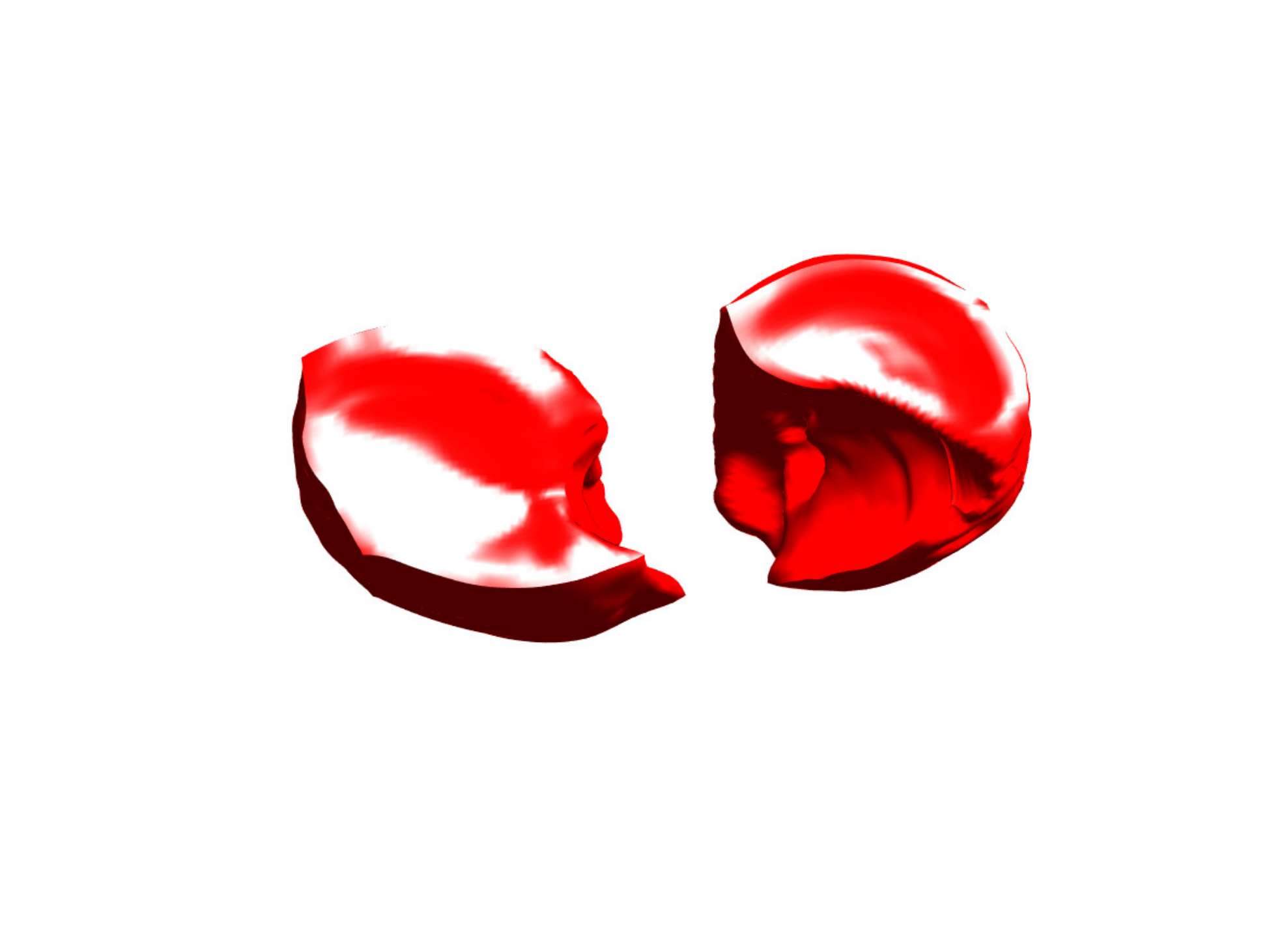}}\\
\subfigure[CV (253.09 sec)]{
\includegraphics[width=1.9in,height=1.5in]{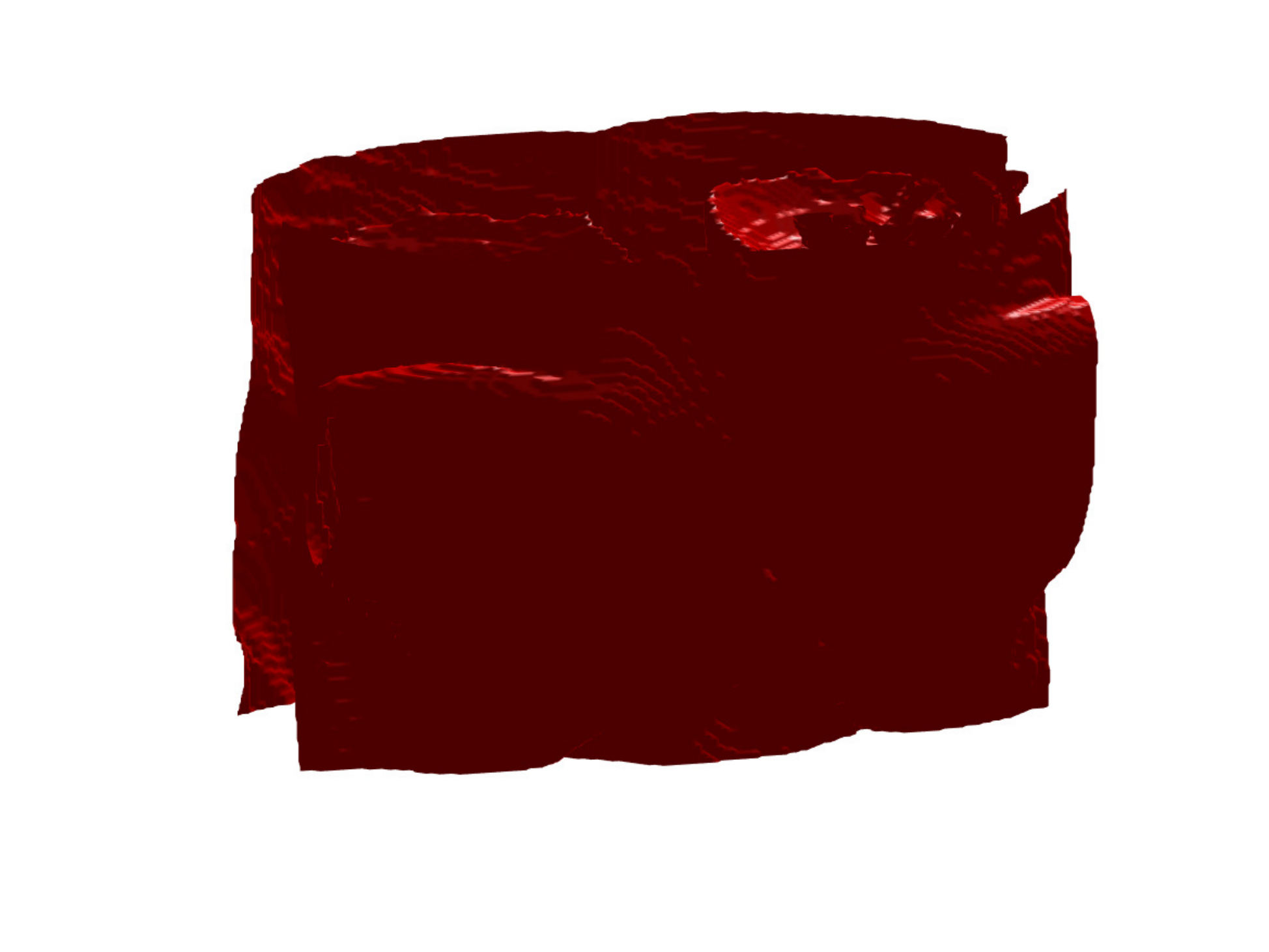}}
\subfigure[CV]{
\includegraphics[width=1.9in,height=1.5in]{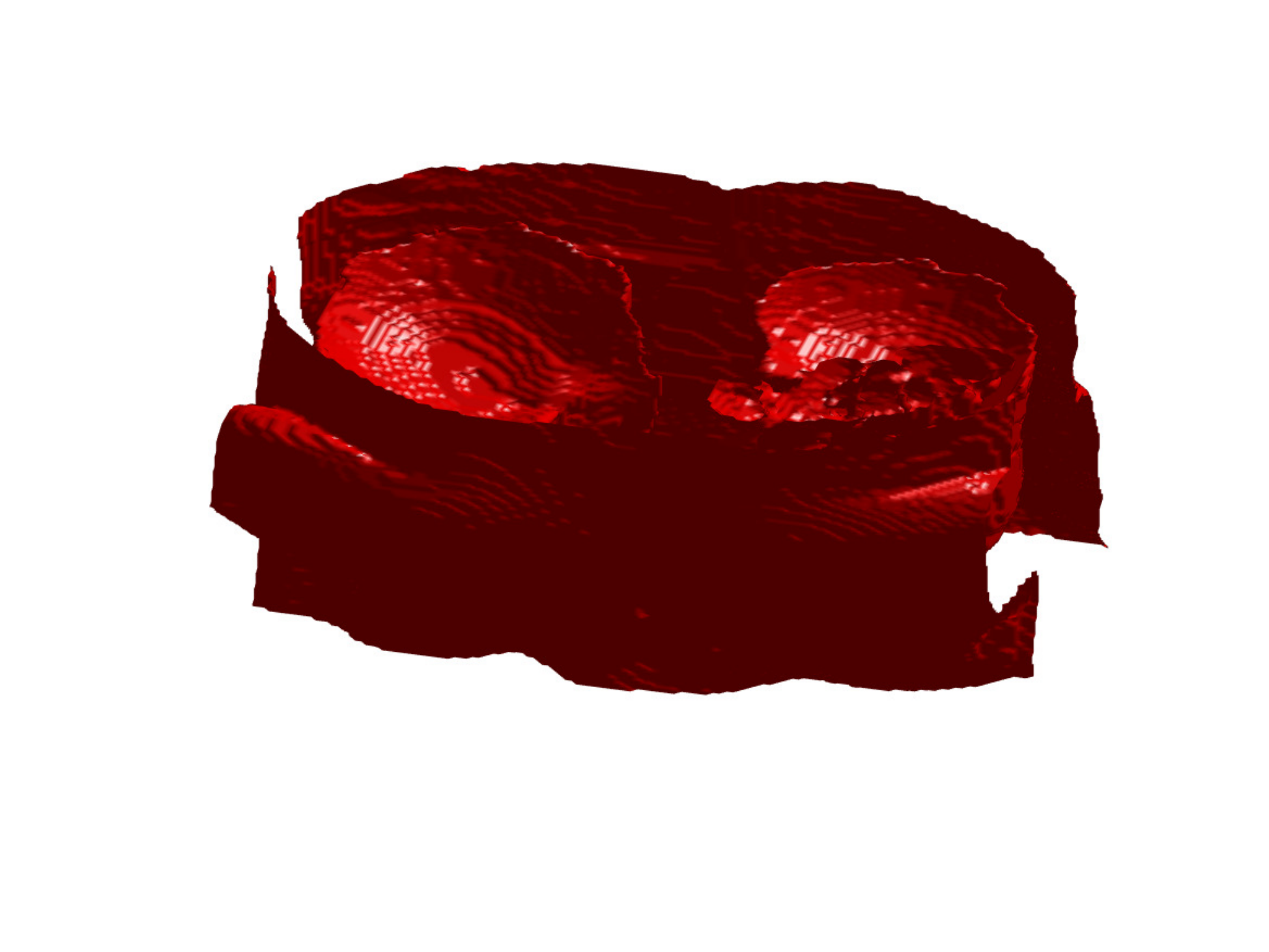}}
\subfigure[CV]{
\includegraphics[width=1.9in,height=1.5in]{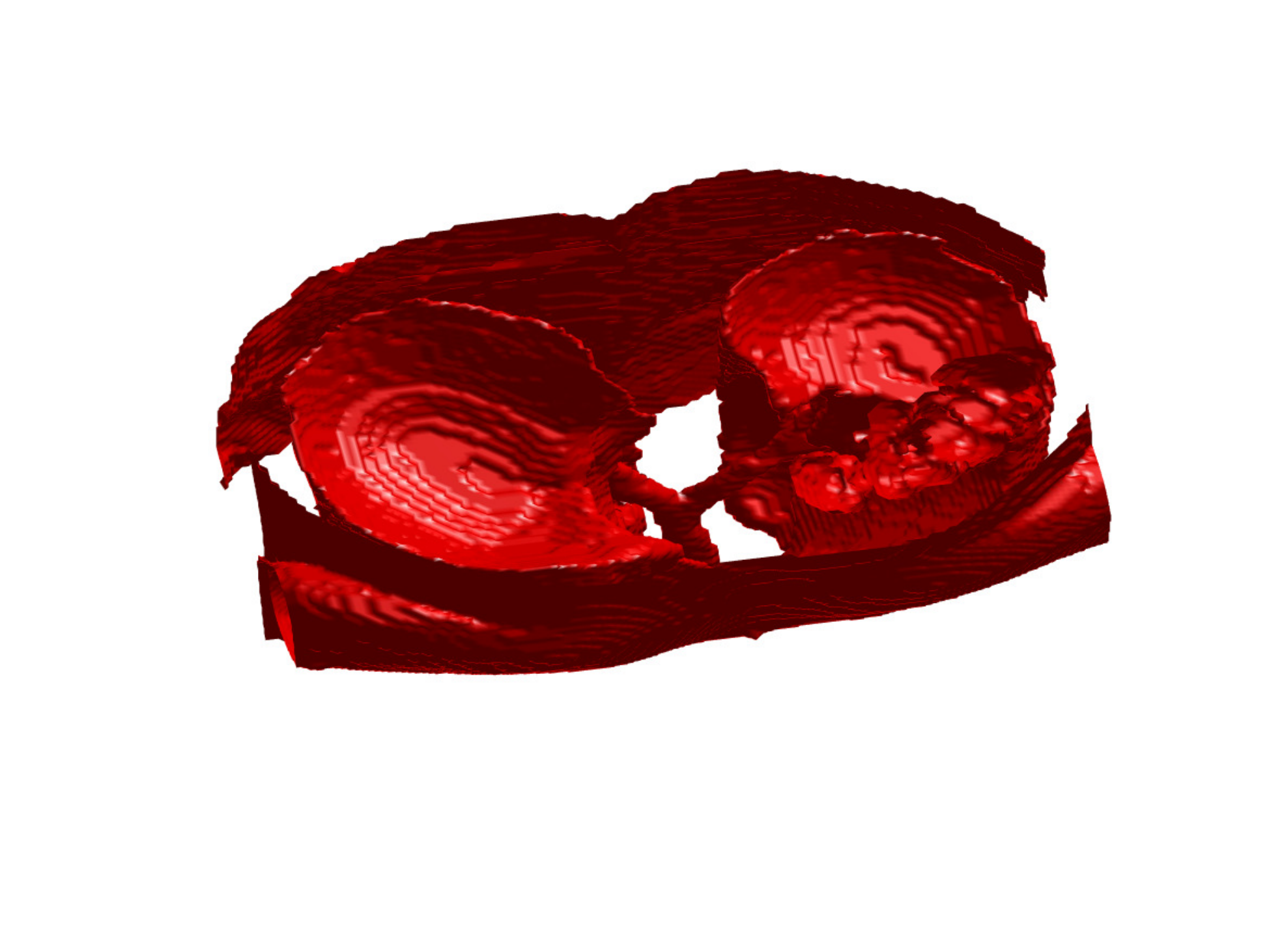}}\\
\subfigure[\cite{zhang2015fast} (682.60 sec)]{
\includegraphics[width=1.9in,height=1.5in]{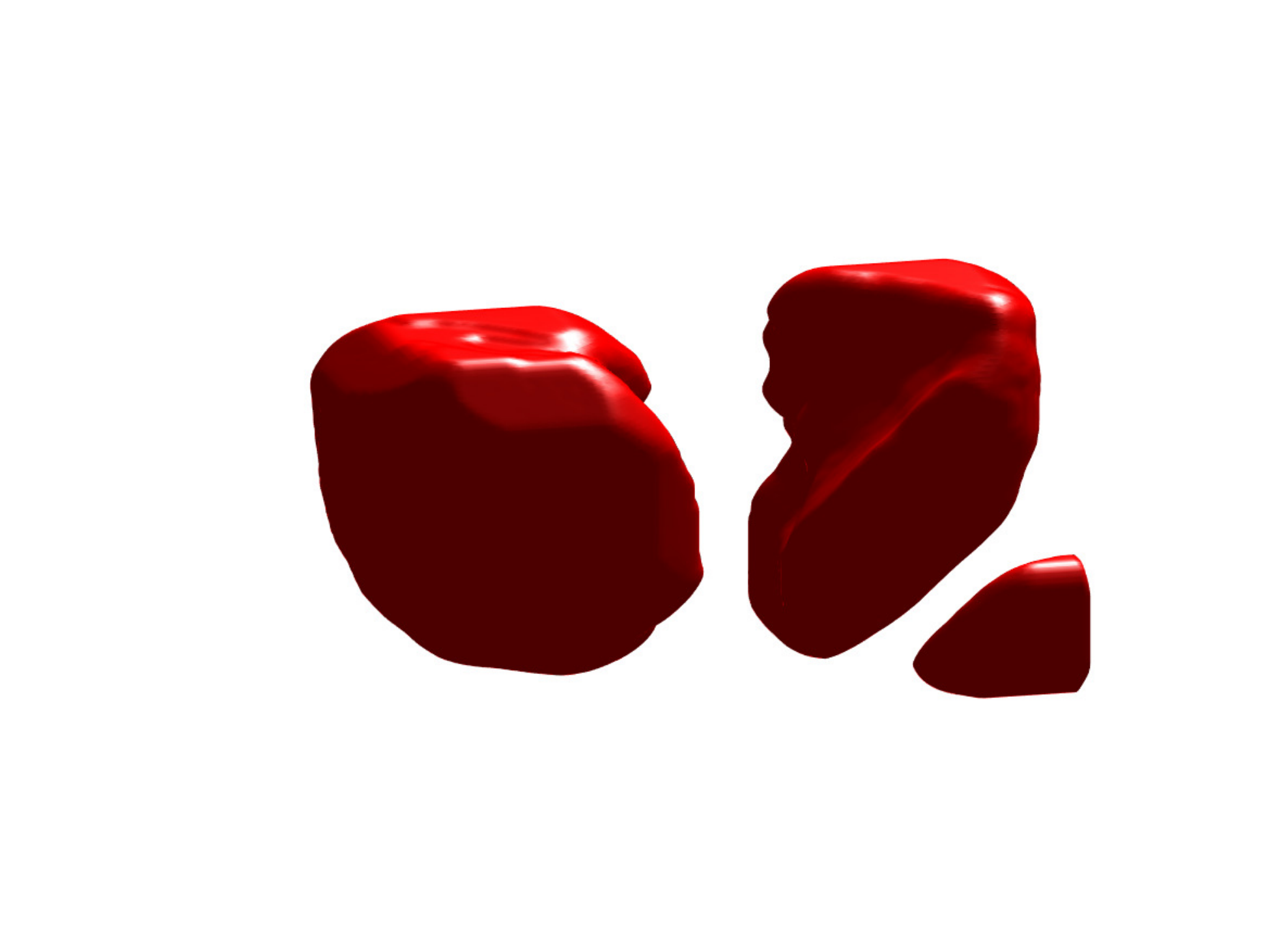}}
\subfigure[\cite{zhang2015fast}]{
\includegraphics[width=1.9in,height=1.5in]{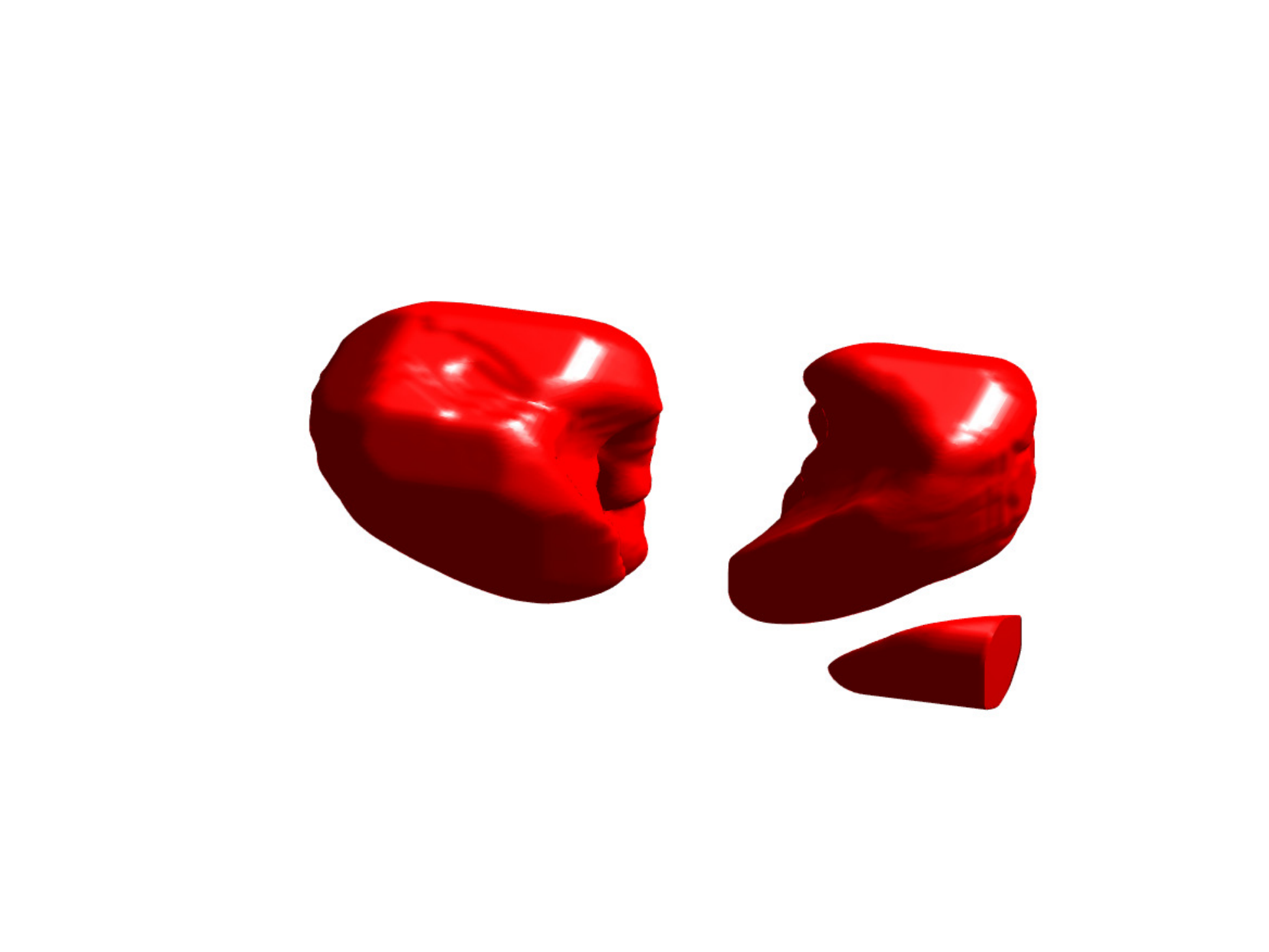}}
\subfigure[\cite{zhang2015fast}]{
\includegraphics[width=1.9in,height=1.5in]{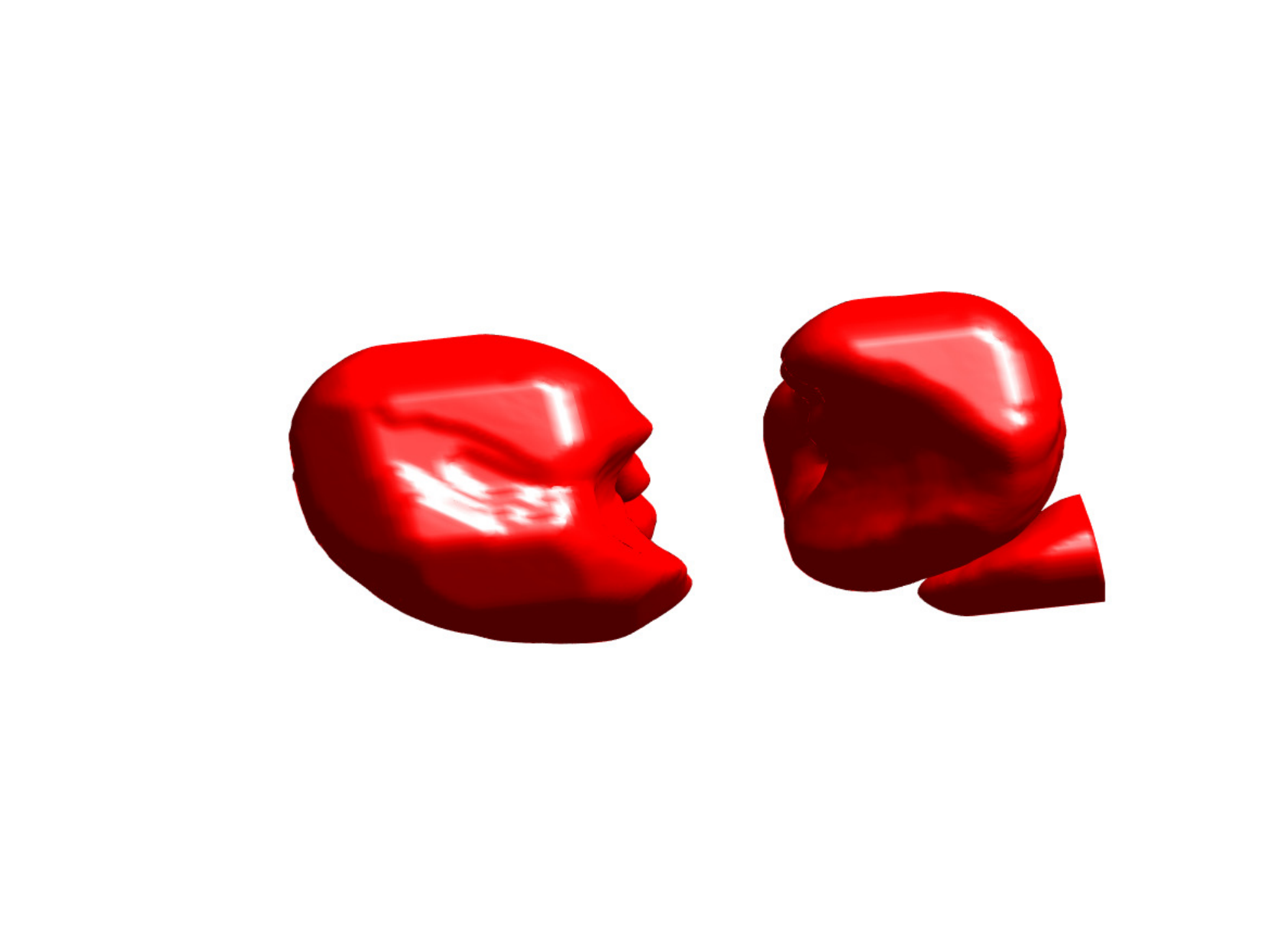}}\\
\subfigure[CV with modification]{
\includegraphics[width=1.9in,height=1.5in]{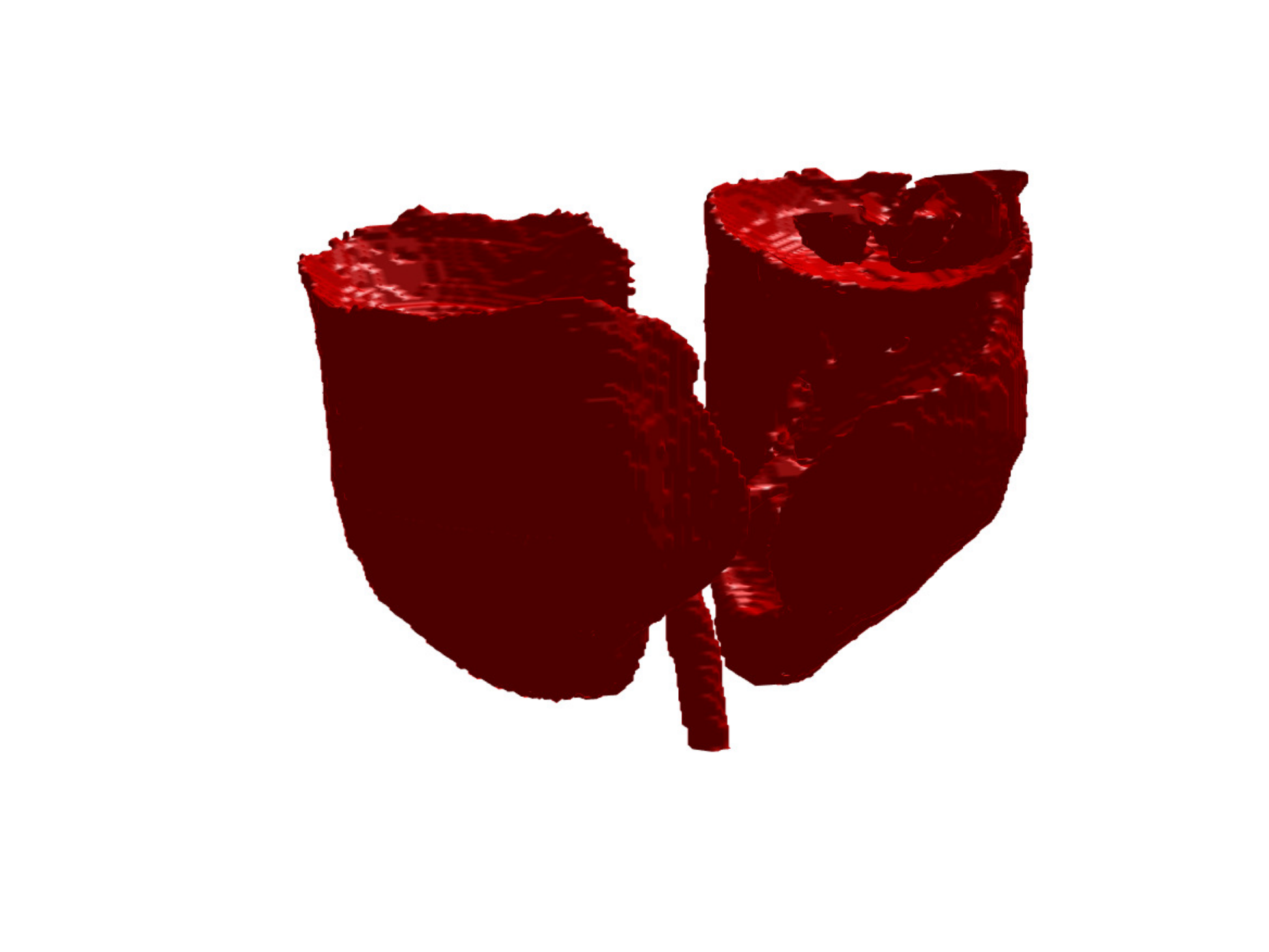}}
\subfigure[CV with modification]{
\includegraphics[width=1.9in,height=1.5in]{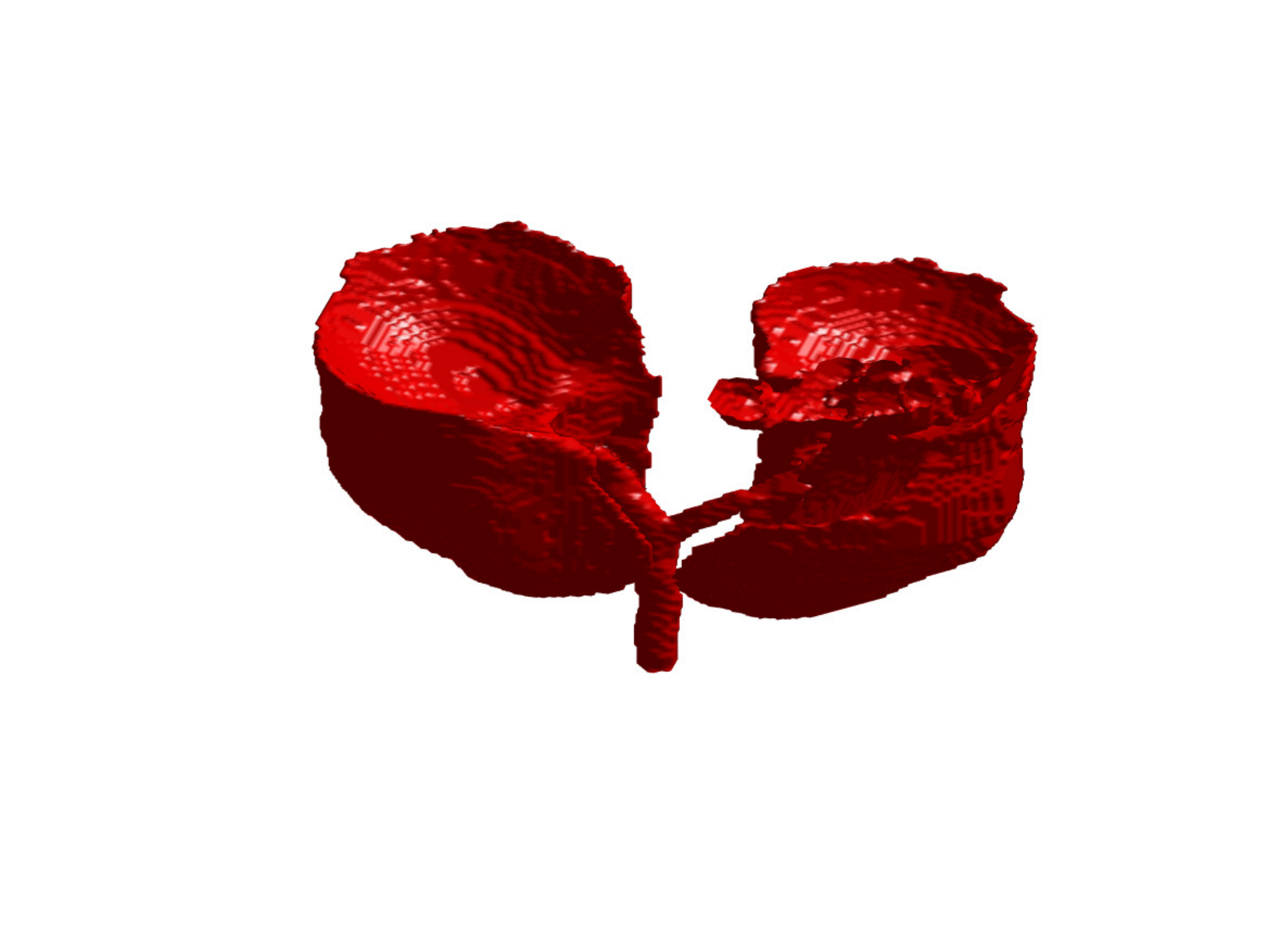}}
\subfigure[CV with modification]{
\includegraphics[width=1.9in,height=1.5in]{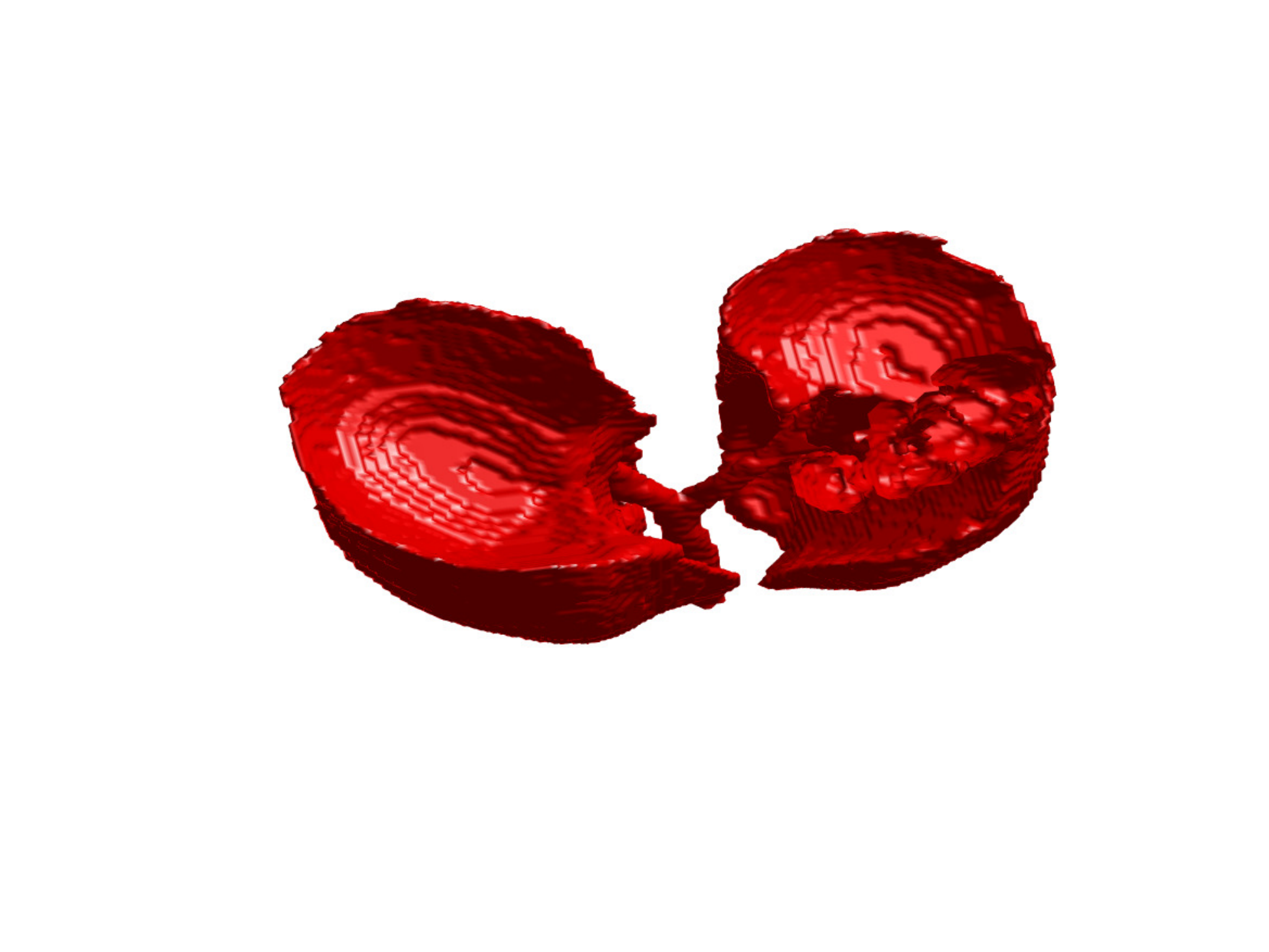}}
\caption{Here, we show the target image and prior image of the second 3D example in the first row. The second to fourth rows show the segmentation results by the proposed model \eqref{ProposedModel}, the Chan-Vese model and the selective model \cite{zhang2015fast} from different angles, respectively. The fifth row shows the segmentation results by the Chan-Vese model with modification to remove the outer part.}\label{3DResult1}
\end{figure}

To simulate the underexposure case, we rescale the 71st to 75th slices' intensity value of the image. From Figure \ref{3DTarget_A}, we can see that the 72nd slice is much darker than the others. For the parameters and prior reference, we follow the choice mentioned above. The segmentation results obtained by the proposed model, the Chan-Vese model and the selective model \cite{zhang2015fast} are listed in Figure \ref{3DResultA}. Once again, we observe that our proposed model \eqref{ProposedModel} successfully gives a topology-preserving segmentation result but the Chan-Vese model does not. In particular, due to the inconsistency of the intensity values of the images, the Chan-Vese model separates the lungs into two parts, which is unreasonable in the real application. For  the selective model, although it dose not separate the target into two parts, the resulting result is not satisfied. Furthermore, from Figure \ref{3DResultA} and \ref{3DResult1}, we can see that the segmentation results produced by our proposed model in this example are similar, which shows that our proposed model is robust with respect to data's perturbations.


\begin{figure}[htbp]
\centering
\subfigure[12nd Slice]{
\includegraphics[width=0.8in,height=0.8in]{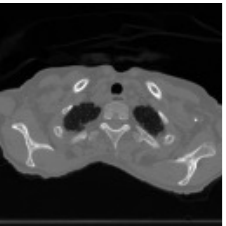}}
\subfigure[22nd Slice]{
\includegraphics[width=0.8in,height=0.8in]{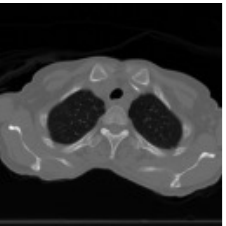}}
\subfigure[32nd Slice]{
\includegraphics[width=0.8in,height=0.8in]{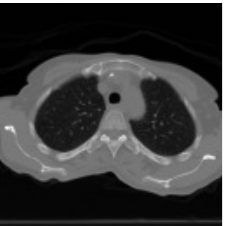}}
\subfigure[42nd Slice]{
\includegraphics[width=0.8in,height=0.8in]{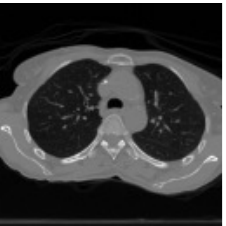}}
\subfigure[52nd Slice]{
\includegraphics[width=0.8in,height=0.8in]{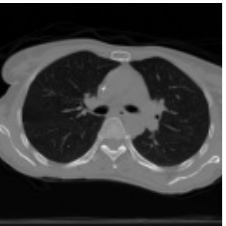}}\\
\subfigure[62nd Slice]{
\includegraphics[width=0.8in,height=0.8in]{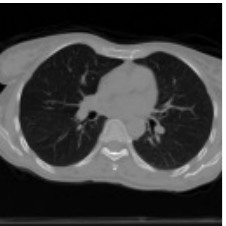}}
\subfigure[72nd Slice]{
\includegraphics[width=0.8in,height=0.8in]{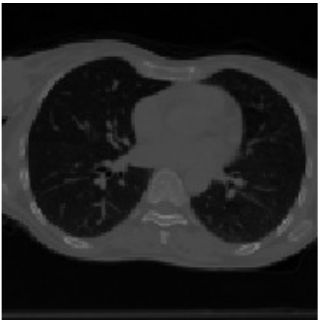}}
\subfigure[82nd Slice]{
\includegraphics[width=0.8in,height=0.8in]{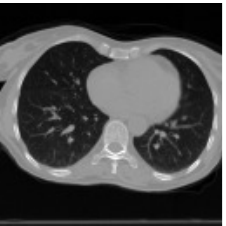}}
\subfigure[92nd Slice]{
\includegraphics[width=0.8in,height=0.8in]{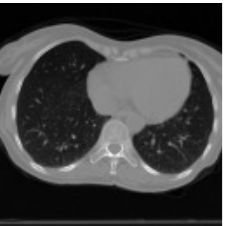}}
\subfigure[102nd Slice]{
\includegraphics[width=0.8in,height=0.8in]{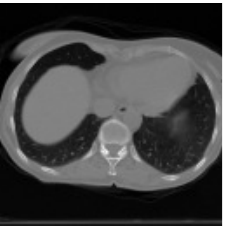}}
\caption{Some slices of the second case of the second 3D example.}\label{3DTarget_A}
\end{figure}

\begin{figure}[htbp]
\centering
\subfigure[Target Image]{
\includegraphics[width=1.9in,height=1.5in]{fig1//view_lung_1-eps-converted-to.pdf}}
\subfigure[Prior Image]{
\includegraphics[width=1.9in,height=1.5in]{fig1//prior_lung_1-eps-converted-to.pdf}}\\
\subfigure[PM (1540.5 sec)]{
\includegraphics[width=1.9in,height=1.5in]{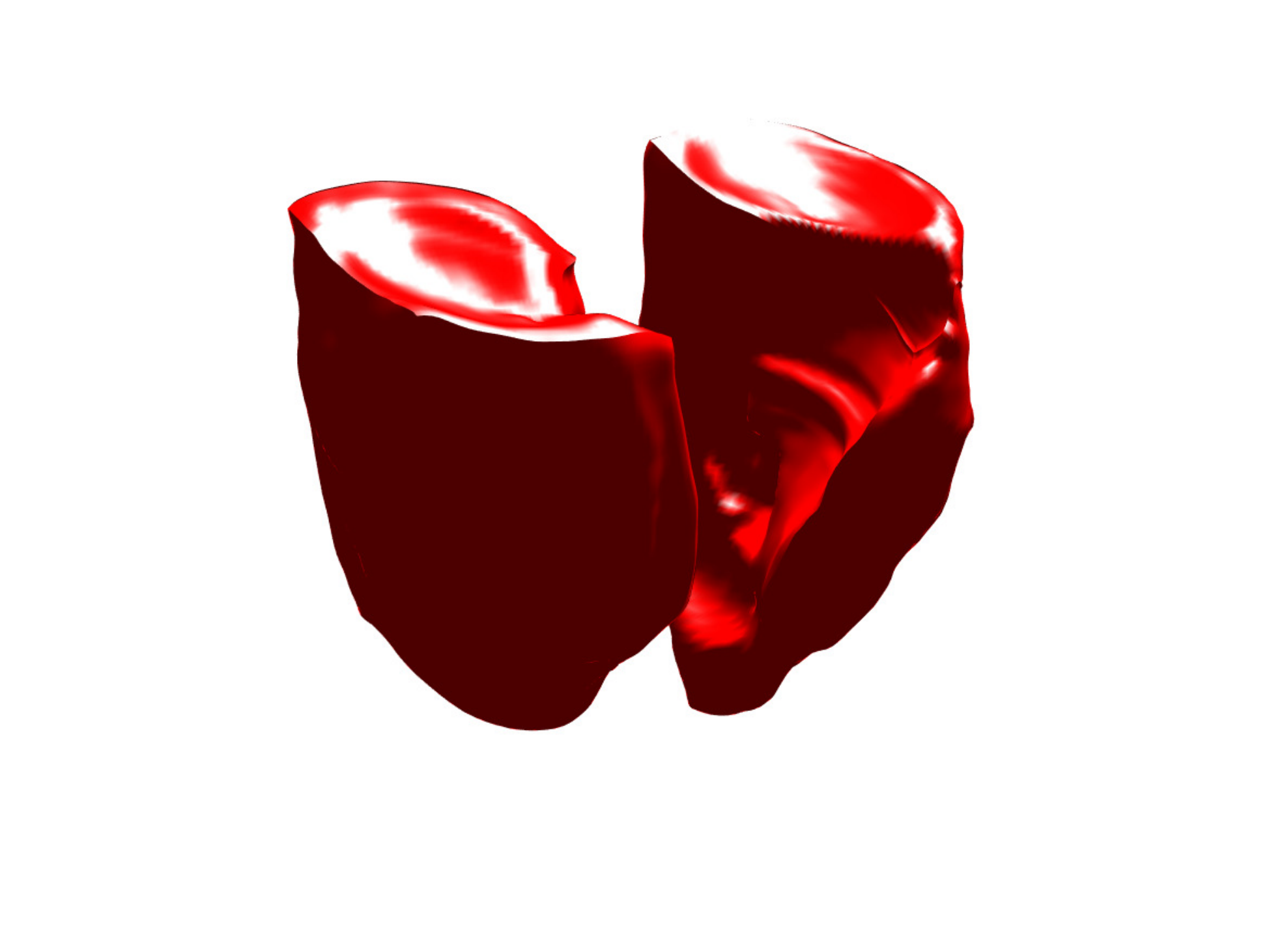}}
\subfigure[PM]{
\includegraphics[width=1.9in,height=1.5in]{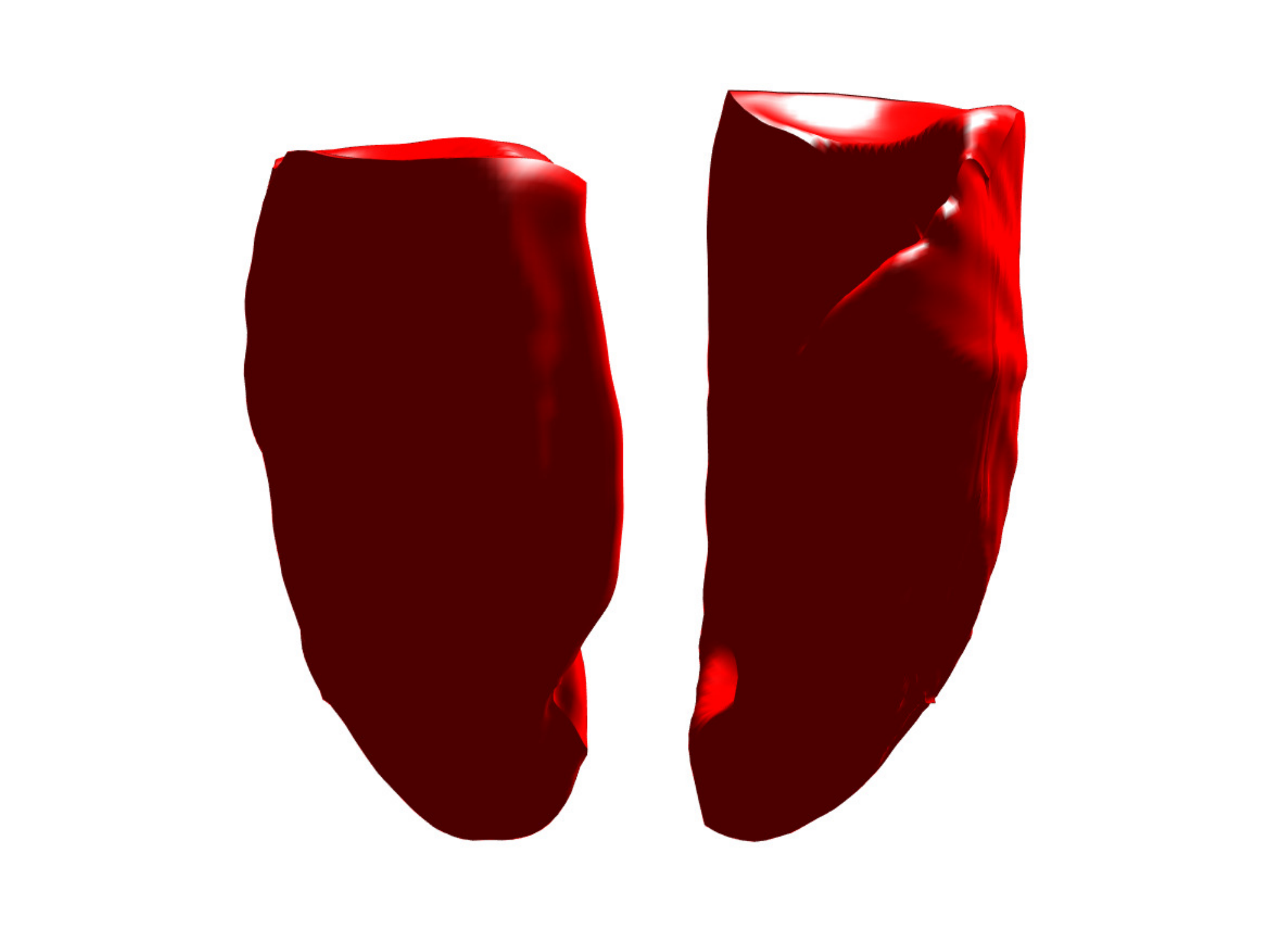}}
\subfigure[PM]{
\includegraphics[width=1.9in,height=1.5in]{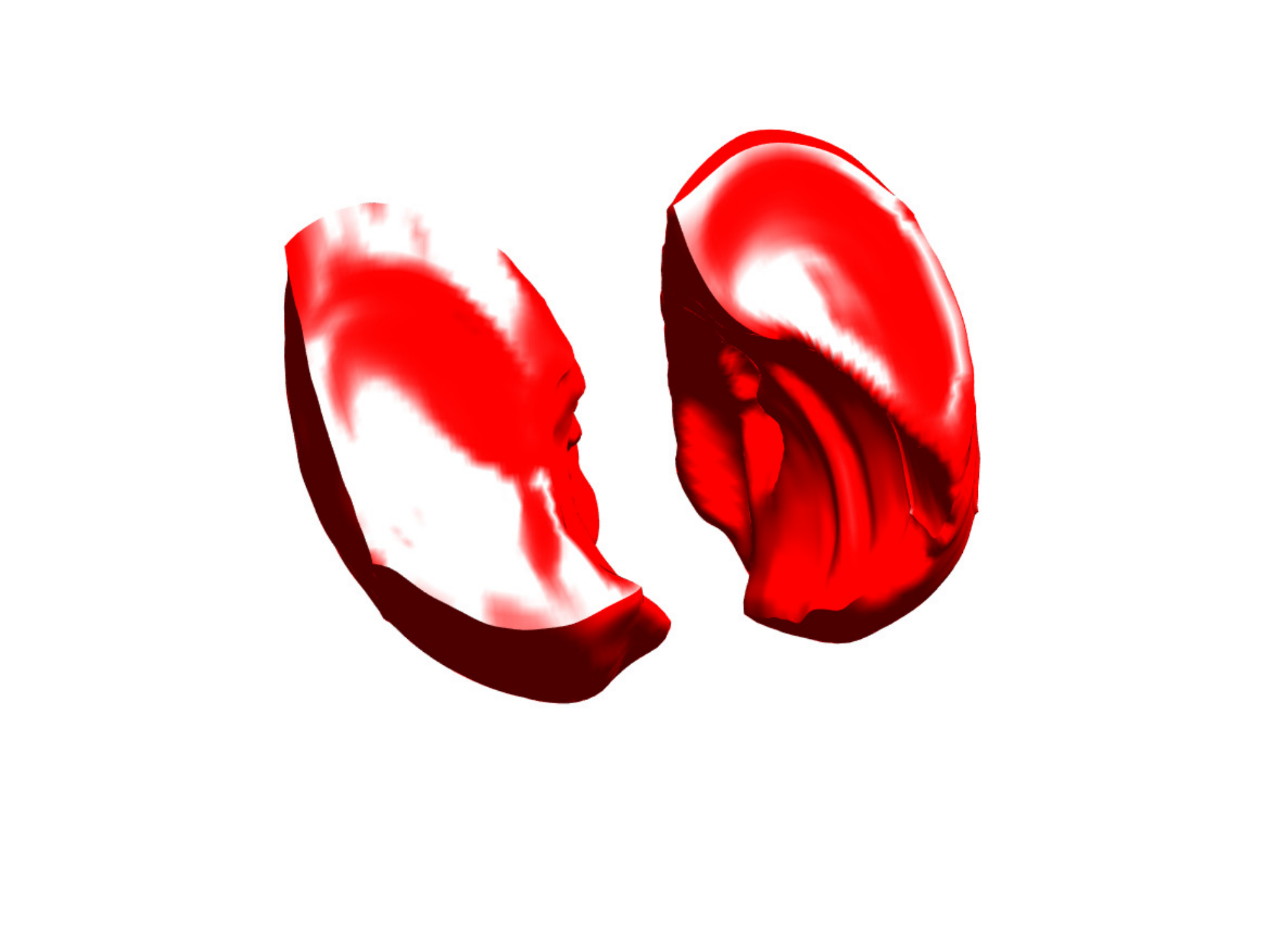}}\\
\subfigure[CV (245.55 sec)]{
\includegraphics[width=1.9in,height=1.5in]{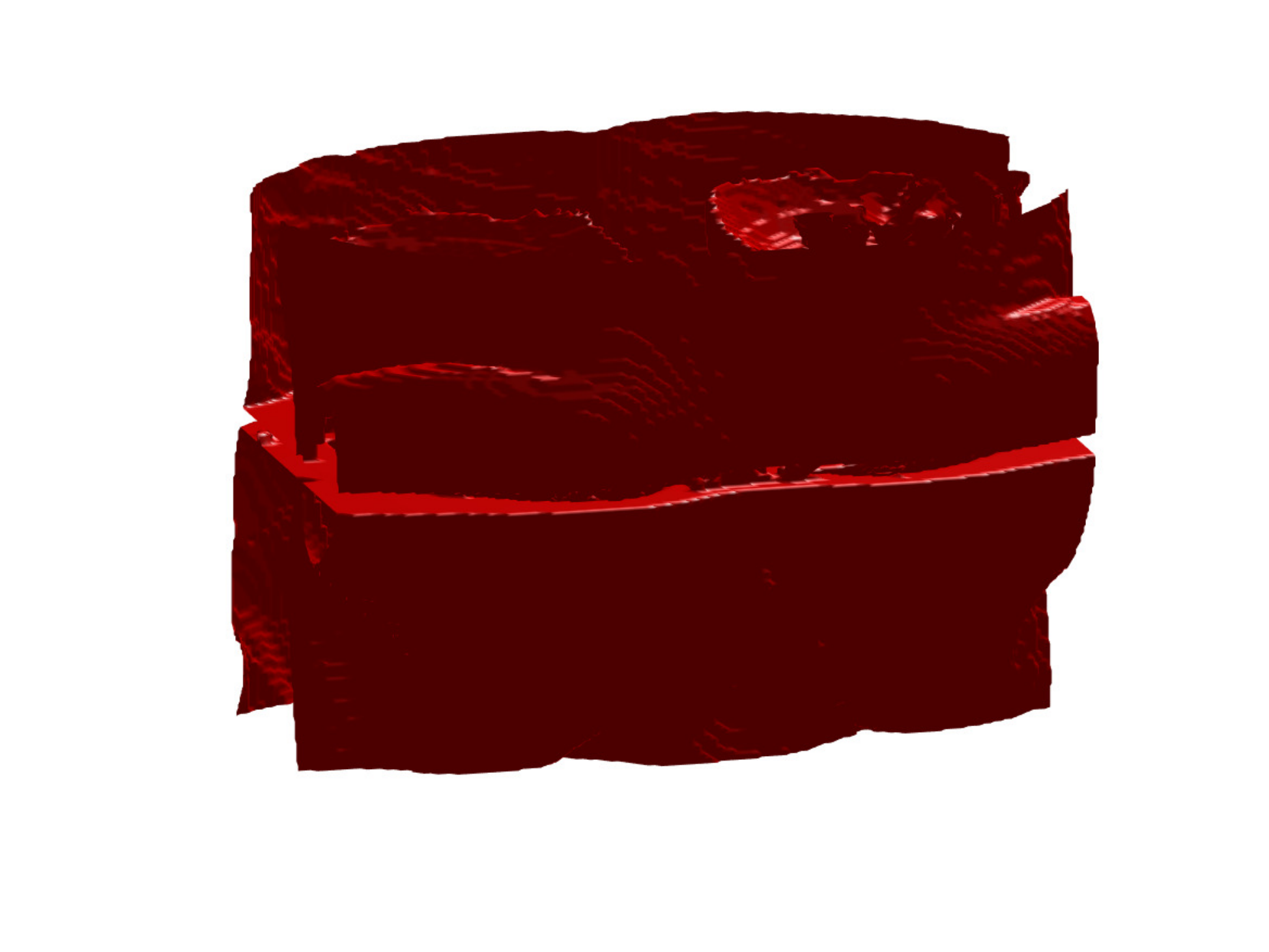}}
\subfigure[CV]{
\includegraphics[width=1.9in,height=1.5in]{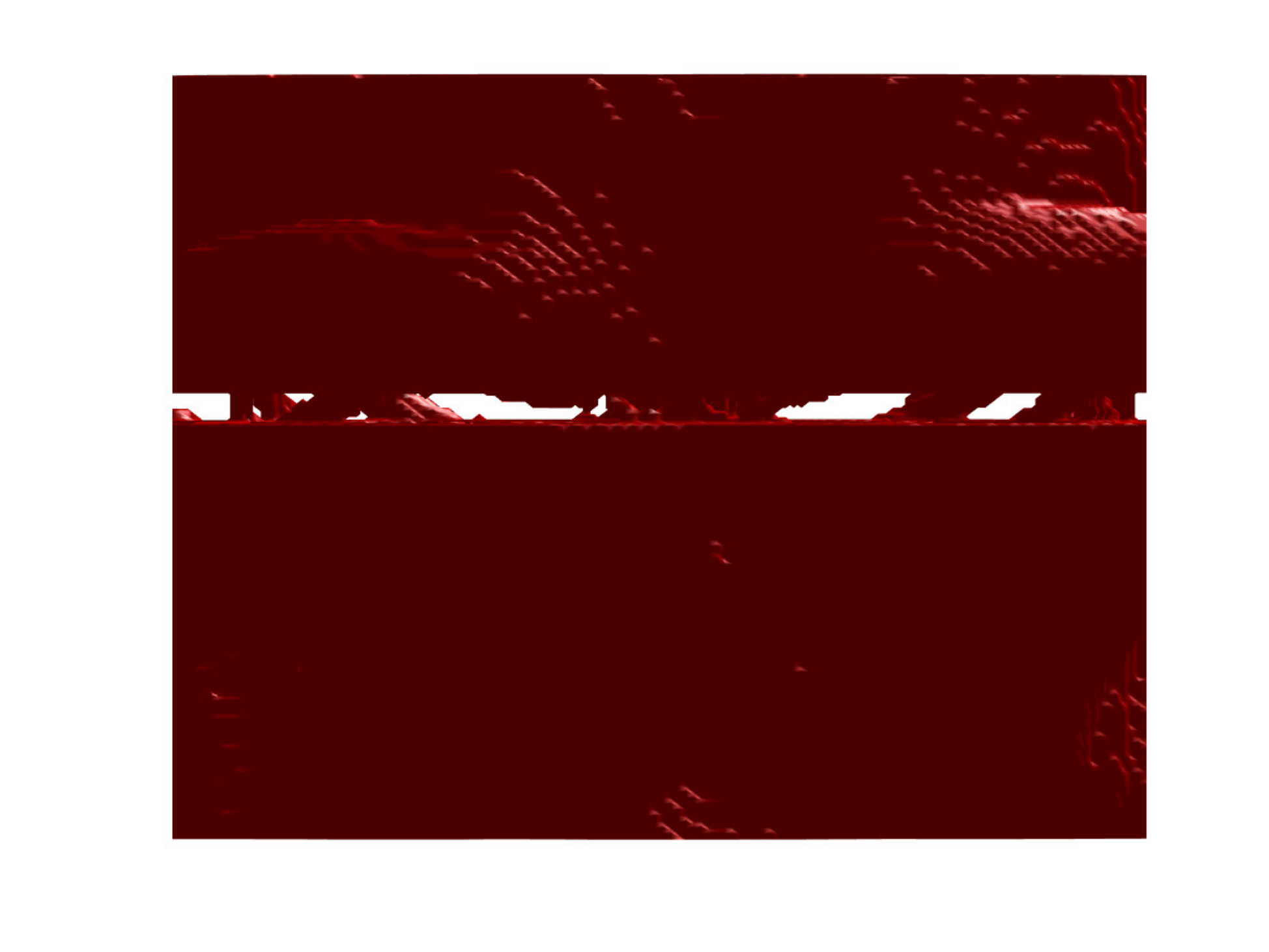}}
\subfigure[CV]{
\includegraphics[width=1.9in,height=1.5in]{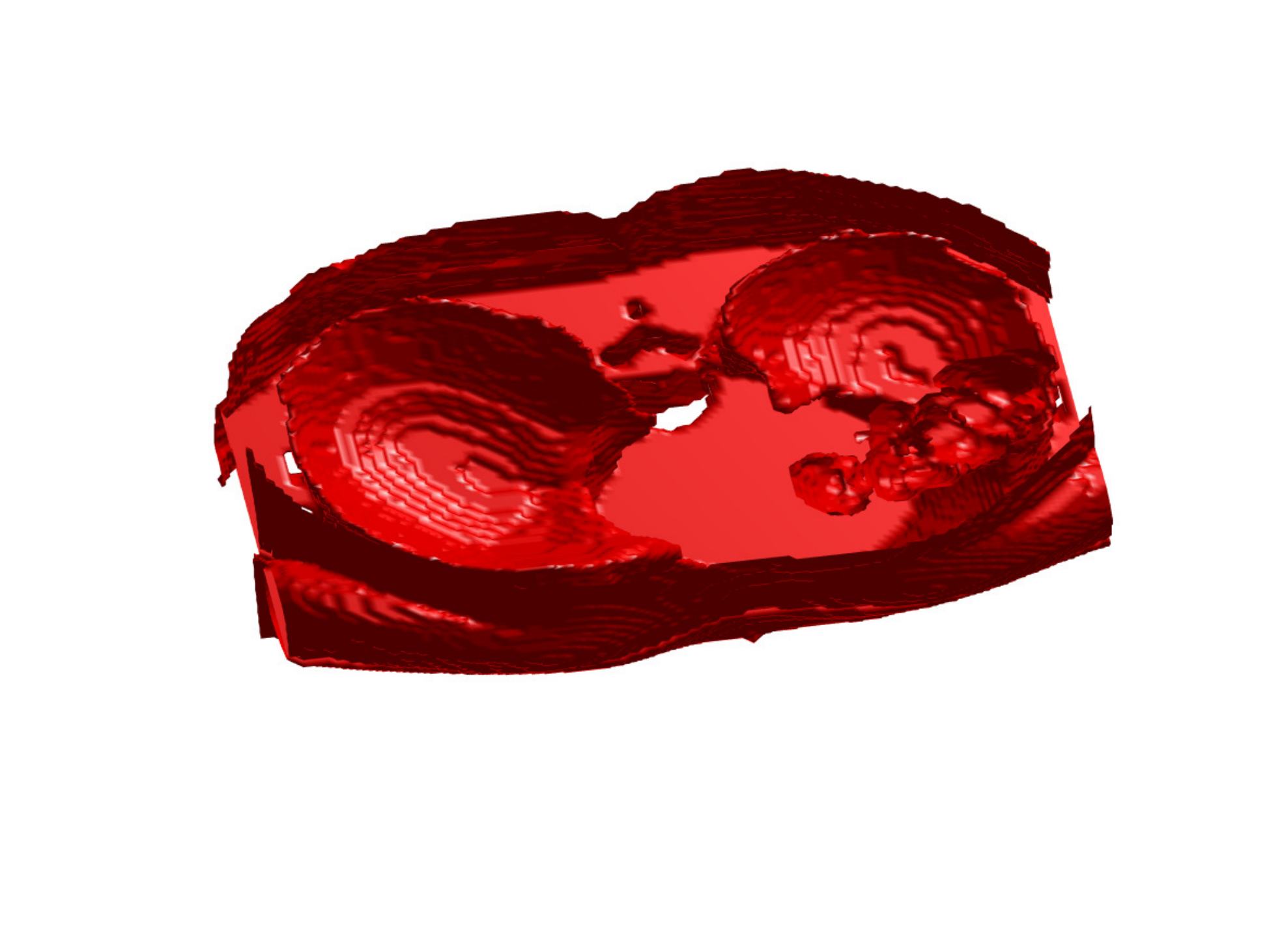}}\\
\subfigure[\cite{zhang2015fast} (1009.2 sec)]{
\includegraphics[width=1.9in,height=1.5in]{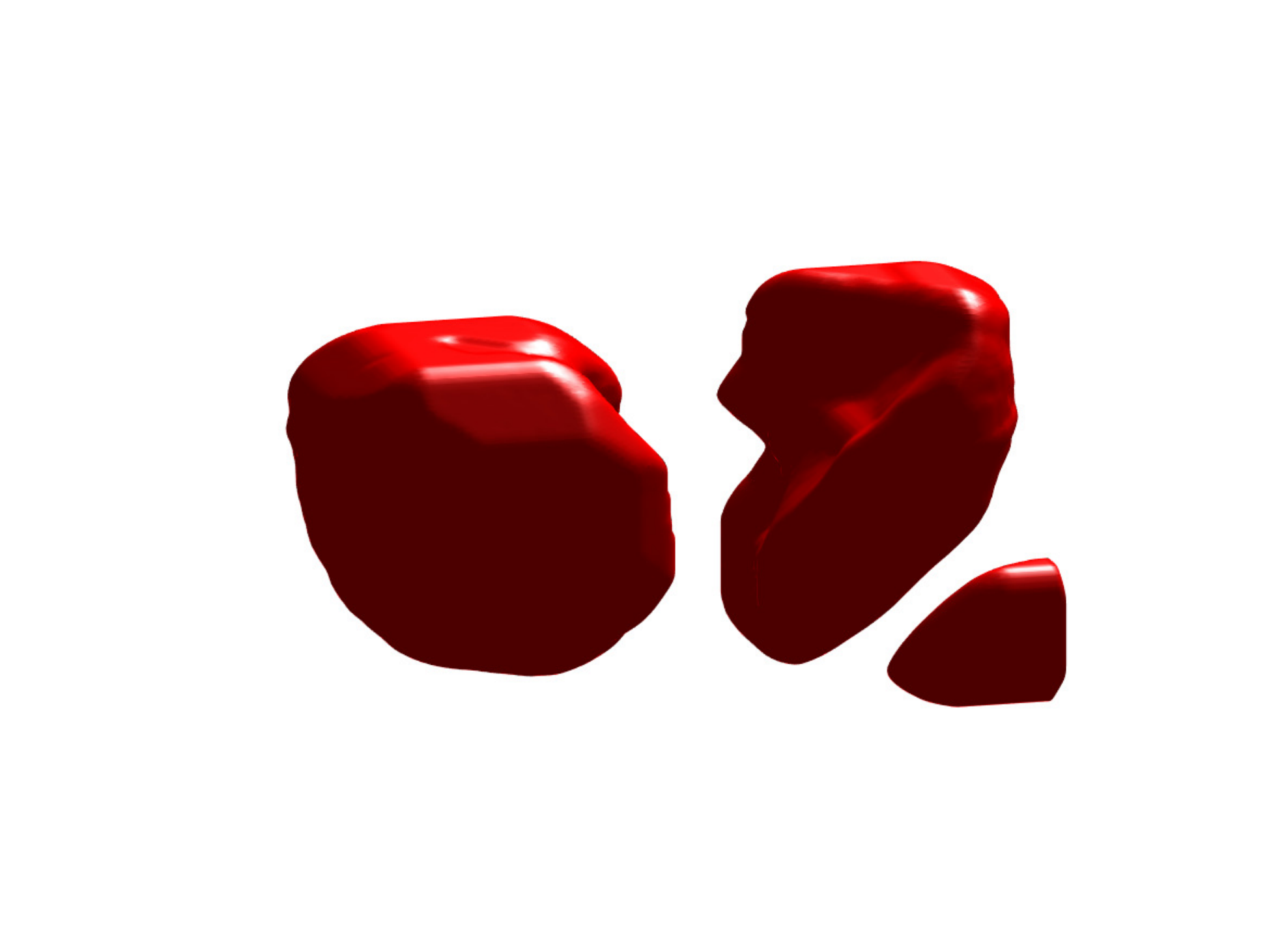}}
\subfigure[\cite{zhang2015fast}]{
\includegraphics[width=1.9in,height=1.5in]{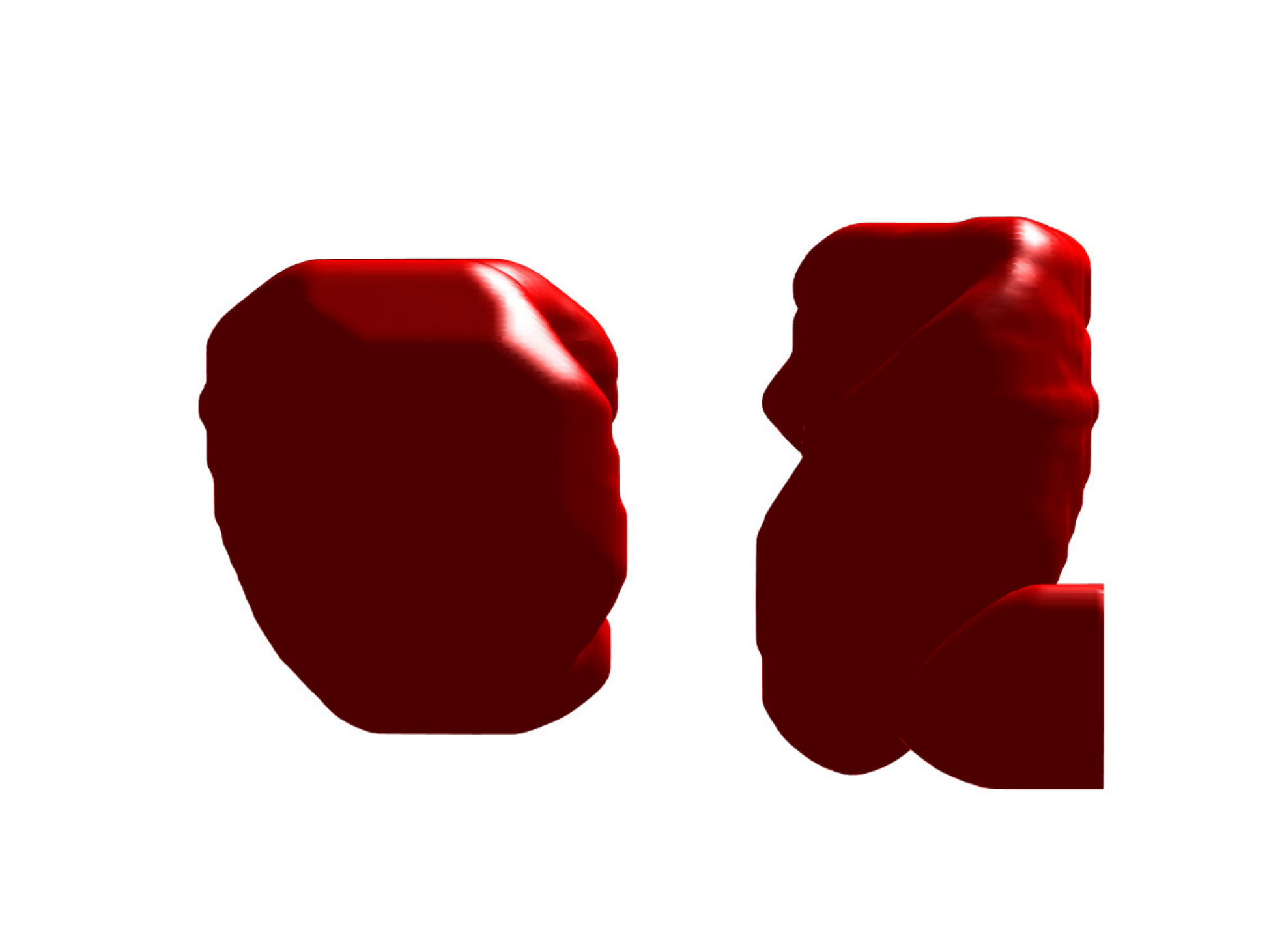}}
\subfigure[\cite{zhang2015fast}]{
\includegraphics[width=1.9in,height=1.5in]{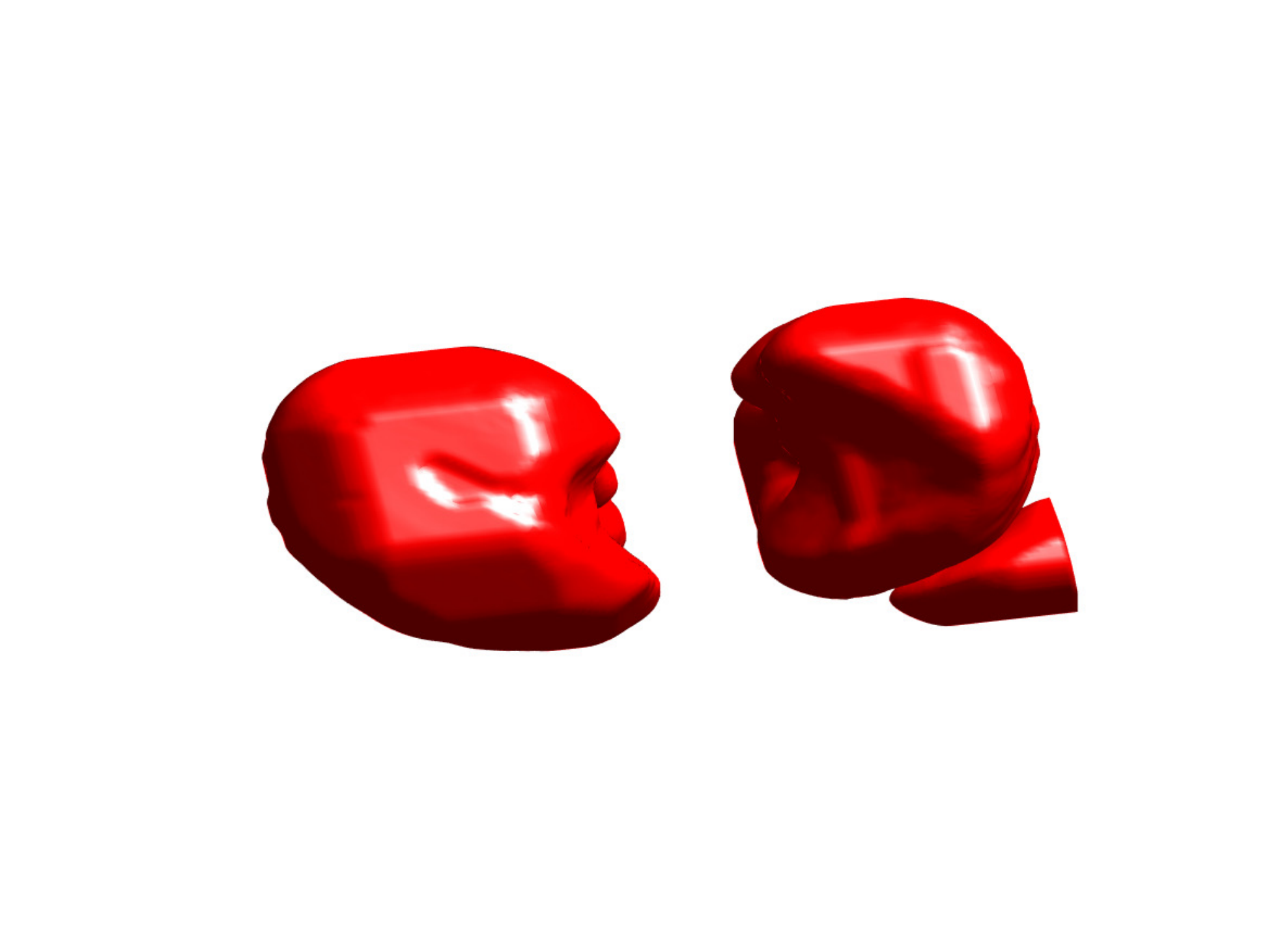}}\\
\subfigure[CV with modification]{
\includegraphics[width=1.9in,height=1.5in]{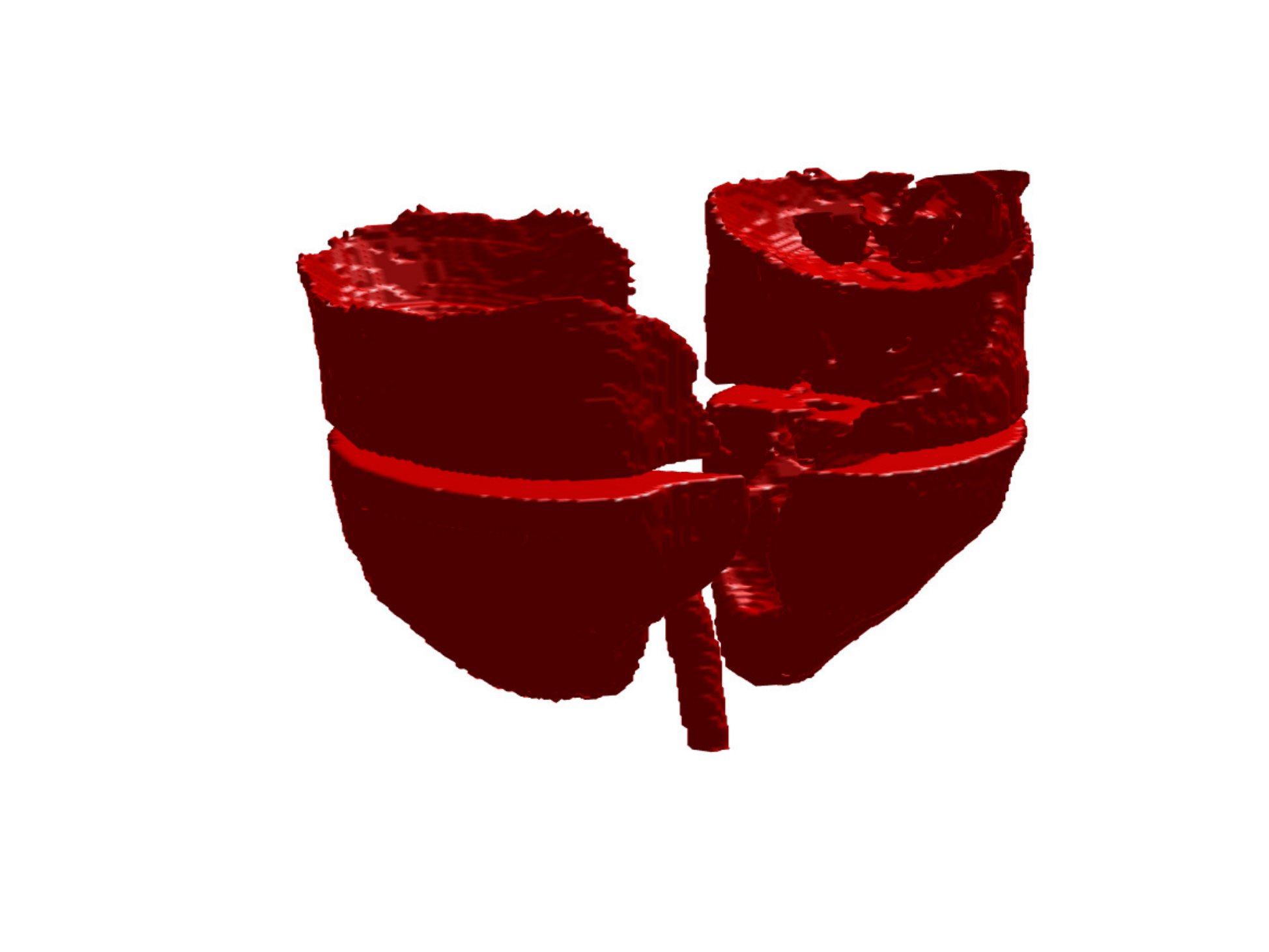}}
\subfigure[CV with modification]{
\includegraphics[width=1.9in,height=1.5in]{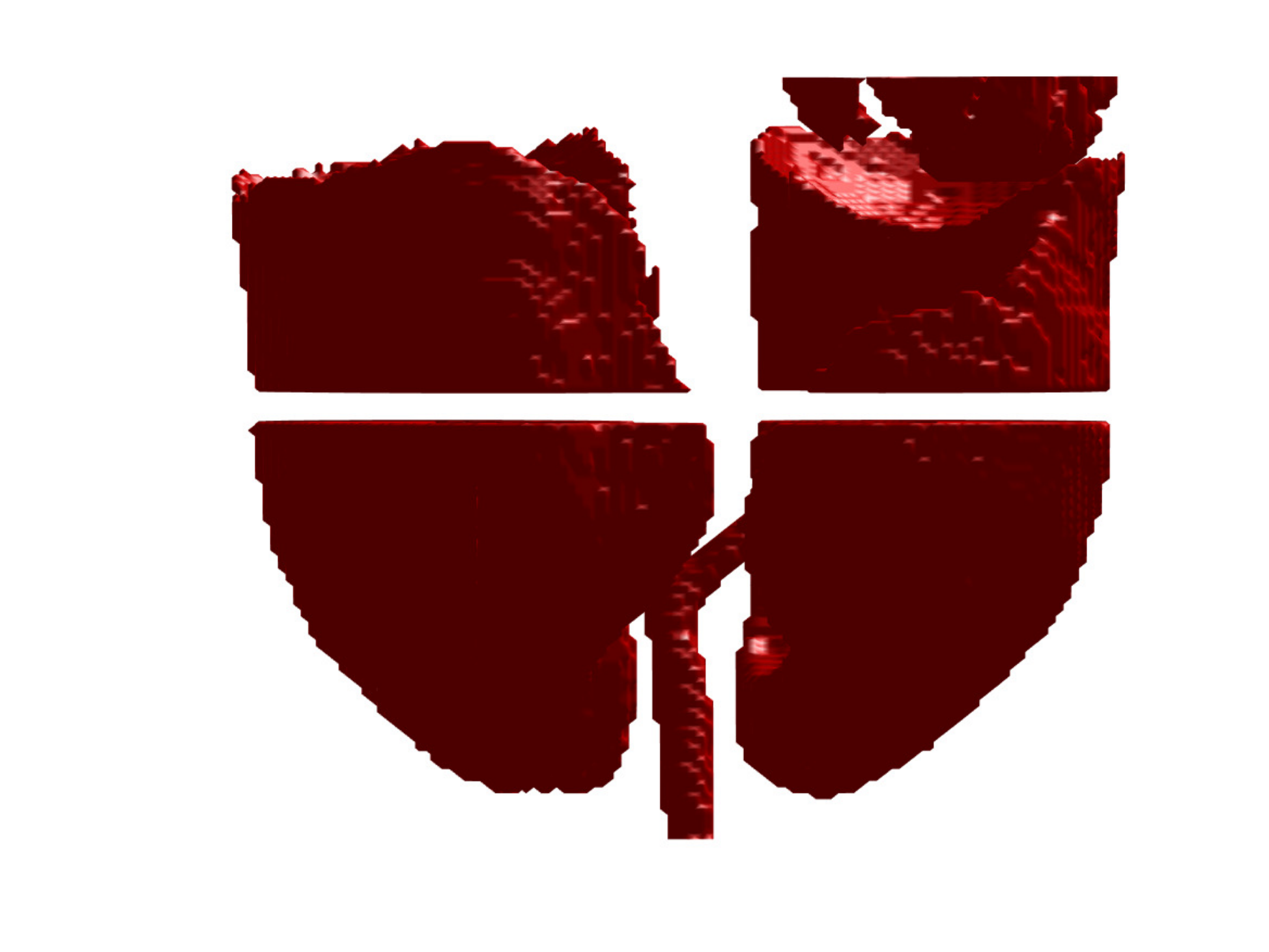}}
\subfigure[CV with modification]{
\includegraphics[width=1.9in,height=1.5in]{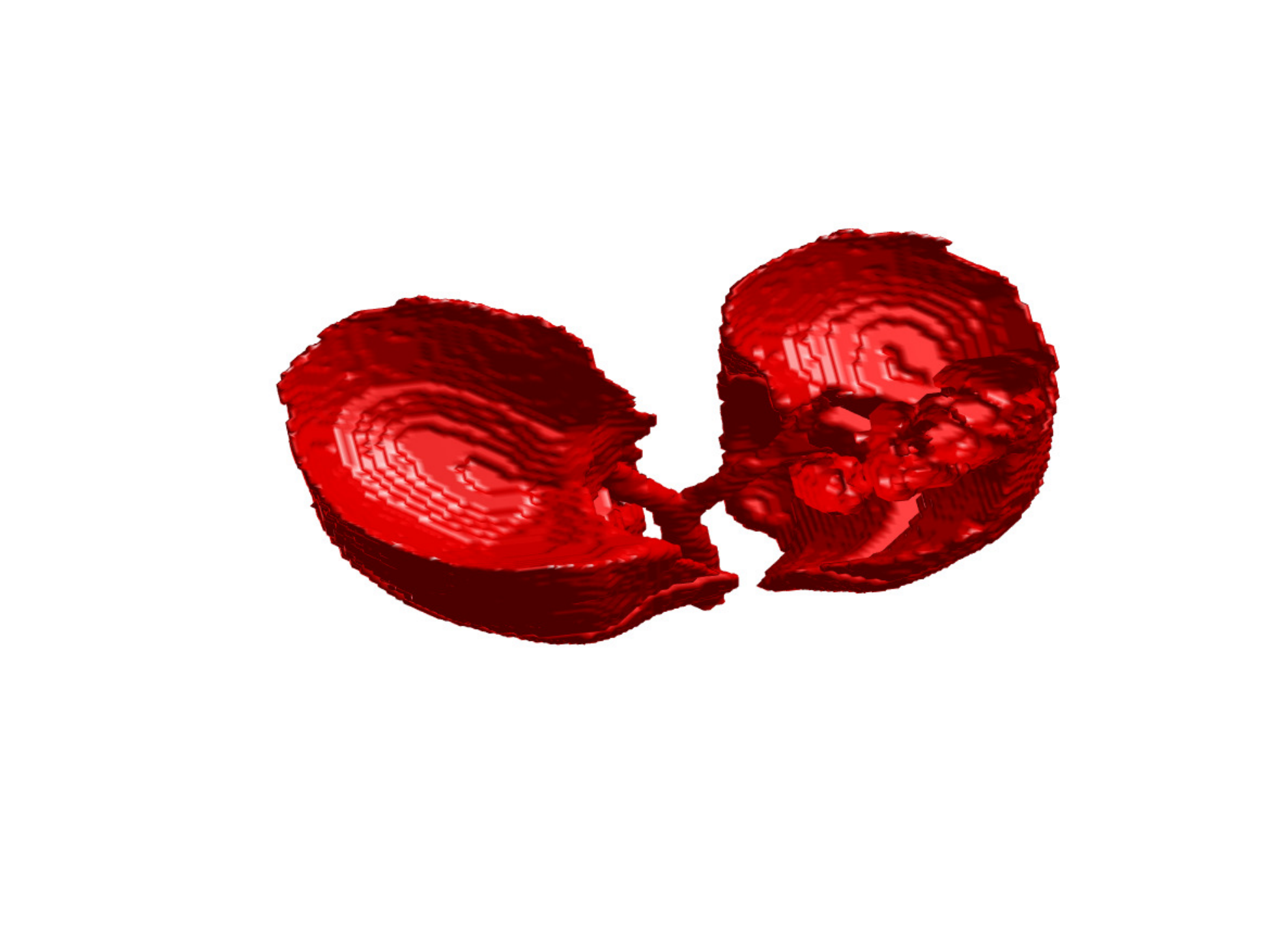}}
\caption{Here, we show the target image and prior image of the second 3D example in the first row. The second to fourth rows show the segmentation results by the proposed model \eqref{ProposedModel}, the Chan-Vese model and the selective model \cite{zhang2015fast} from different angles, respectively. The fifth row shows the segmentation results by the Chan-Vese model with modification to remove the outer part.}\label{3DResultA}
\end{figure}

\bigskip

\noindent {\bf Example 6:} 
The third 3D example is also a 3D lung CT scans. The slices are shown in Figure \ref{3DTarget_2} and the 3D view is shown in Figure \ref{3DResult2} (a). Again, we use two cuboids as the topological prior, as shown in Figure \ref{3DResult2} (b). For the parameters of our proposed model \eqref{ProposedModel}, we set $\alpha_{l}=10$, $\alpha_{s}=1$ and $\alpha_{v}=10$. And for the Chan-Vese model, we again use the default parameter. The segmentation results obtained by our proposed model, the Chan-Vese model and a selective model are displayed in Figure \ref{3DResult2} (c-k). Note that in this example, the Chan-Vese model can only segment the outer contour without giving any meaningful information about the inner part. A possible reason is that the contrast of the inner part is not obvious enough. The selective model can give a satisfied result with preserving topology. However, with the prior reference, our proposed model can successfully segment the lungs and the result is topology-preserving. This example again demonstrates that the segmentation result can be significantly improved by providing a reasonable topological prior. 

\bigskip

Finally, the energy plots versus iterations of all cases in 3D examples are displayed in Figure \ref{Energy3D}. For the computational time, we can see that although our proposed model needs more time than the other two models, this sacrifice should be deserved because our model can absolutely preserve the topological structure. 

\begin{remark}
We note that for these 2D and 3D examples, our proposed model \eqref{ProposedModel} possesses the same advantage with \cite{chan2018topology}: the shapes of the priors do not need to be similar with the target objectives. However, our proposed model \eqref{ProposedModel} can deal with the 3D segmentation, which is a bottleneck of the registration-based segmentation model using the Beltrami representation \eqref{TPB}.
\end{remark}
 

\begin{figure}[htbp]
\centering
\subfigure[12nd Slice]{
\includegraphics[width=0.8in,height=0.8in]{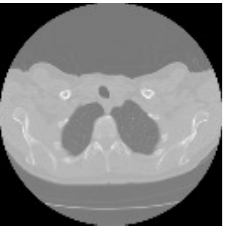}}
\subfigure[22nd Slice]{
\includegraphics[width=0.8in,height=0.8in]{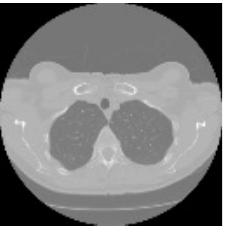}}
\subfigure[32nd Slice]{
\includegraphics[width=0.8in,height=0.8in]{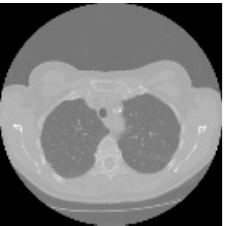}}
\subfigure[42nd Slice]{
\includegraphics[width=0.8in,height=0.8in]{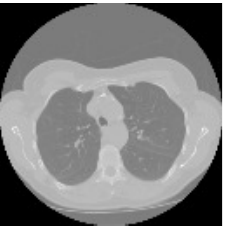}}
\subfigure[52nd Slice]{
\includegraphics[width=0.8in,height=0.8in]{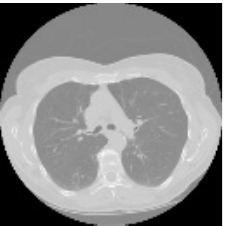}}\\
\subfigure[62nd Slice]{
\includegraphics[width=0.8in,height=0.8in]{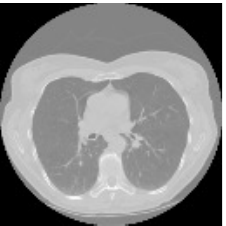}}
\subfigure[72nd Slice]{
\includegraphics[width=0.8in,height=0.8in]{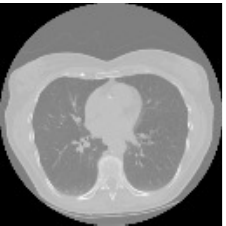}}
\subfigure[82nd Slice]{
\includegraphics[width=0.8in,height=0.8in]{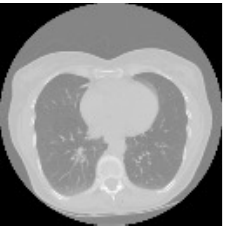}}
\subfigure[92nd Slice]{
\includegraphics[width=0.8in,height=0.8in]{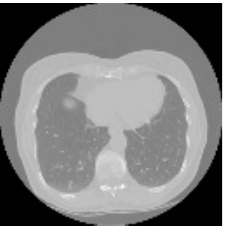}}
\subfigure[102nd Slice]{
\includegraphics[width=0.8in,height=0.8in]{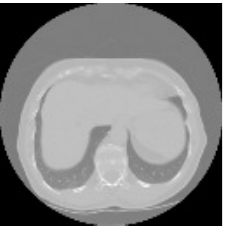}}
\caption{Some slices of the third 3D example.}\label{3DTarget_2}
\end{figure}

\begin{figure}[htbp]
\centering
\subfigure[Target Image]{
\includegraphics[width=1.9in,height=1.5in]{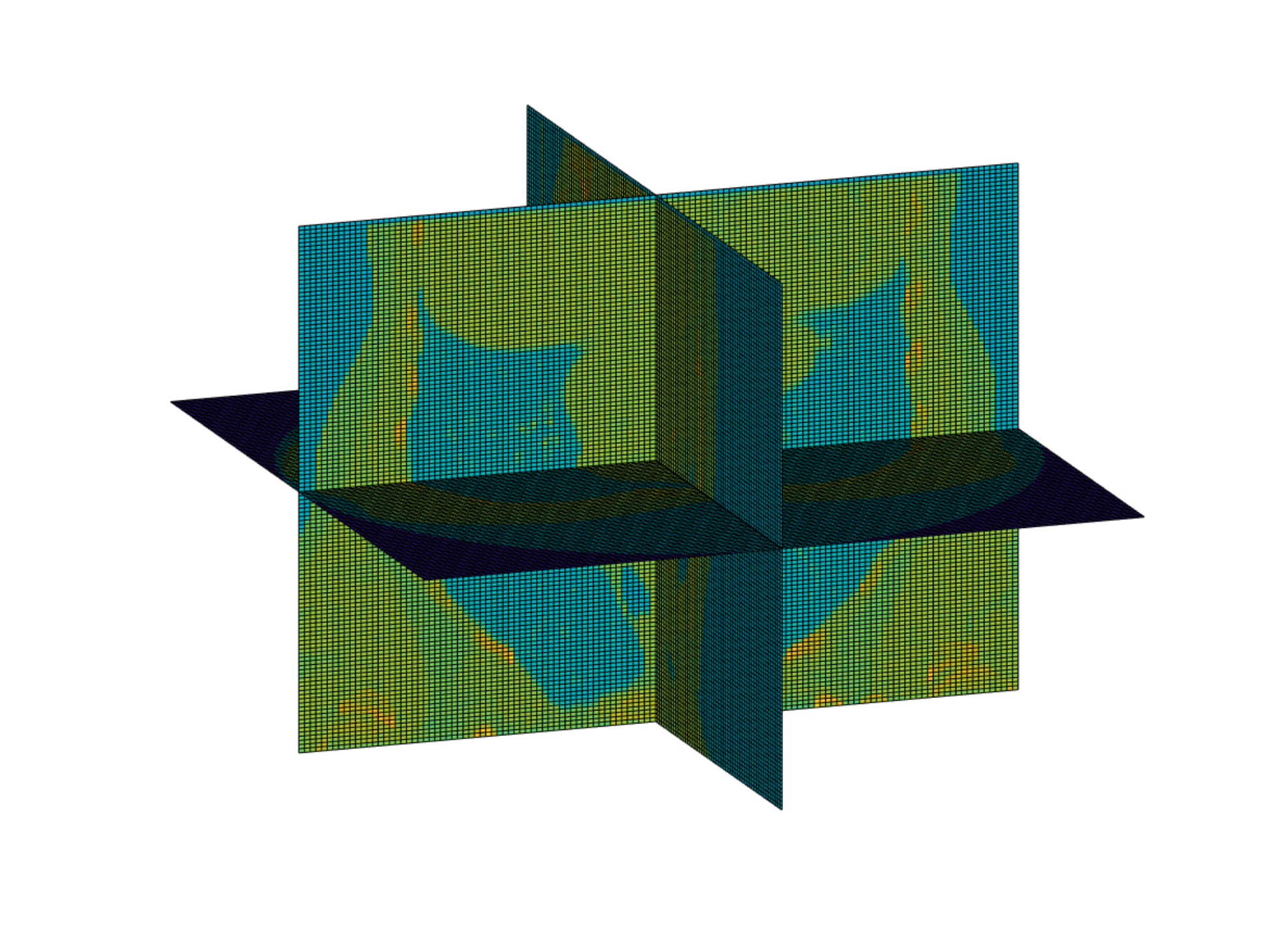}}
\subfigure[Prior Image]{
\includegraphics[width=1.9in,height=1.5in]{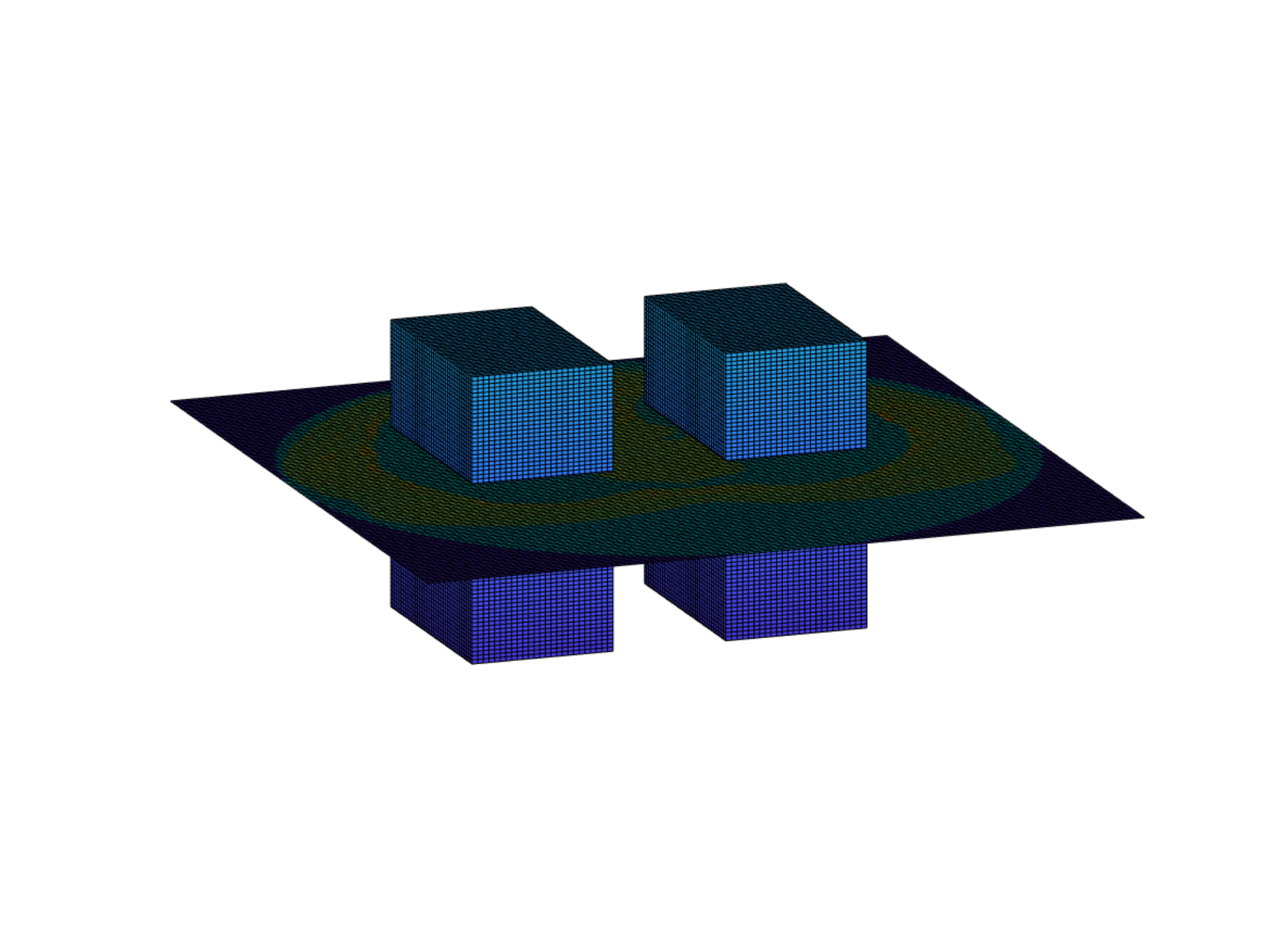}}\\
\subfigure[PM (1904.5 sec)]{
\includegraphics[width=1.9in,height=1.5in]{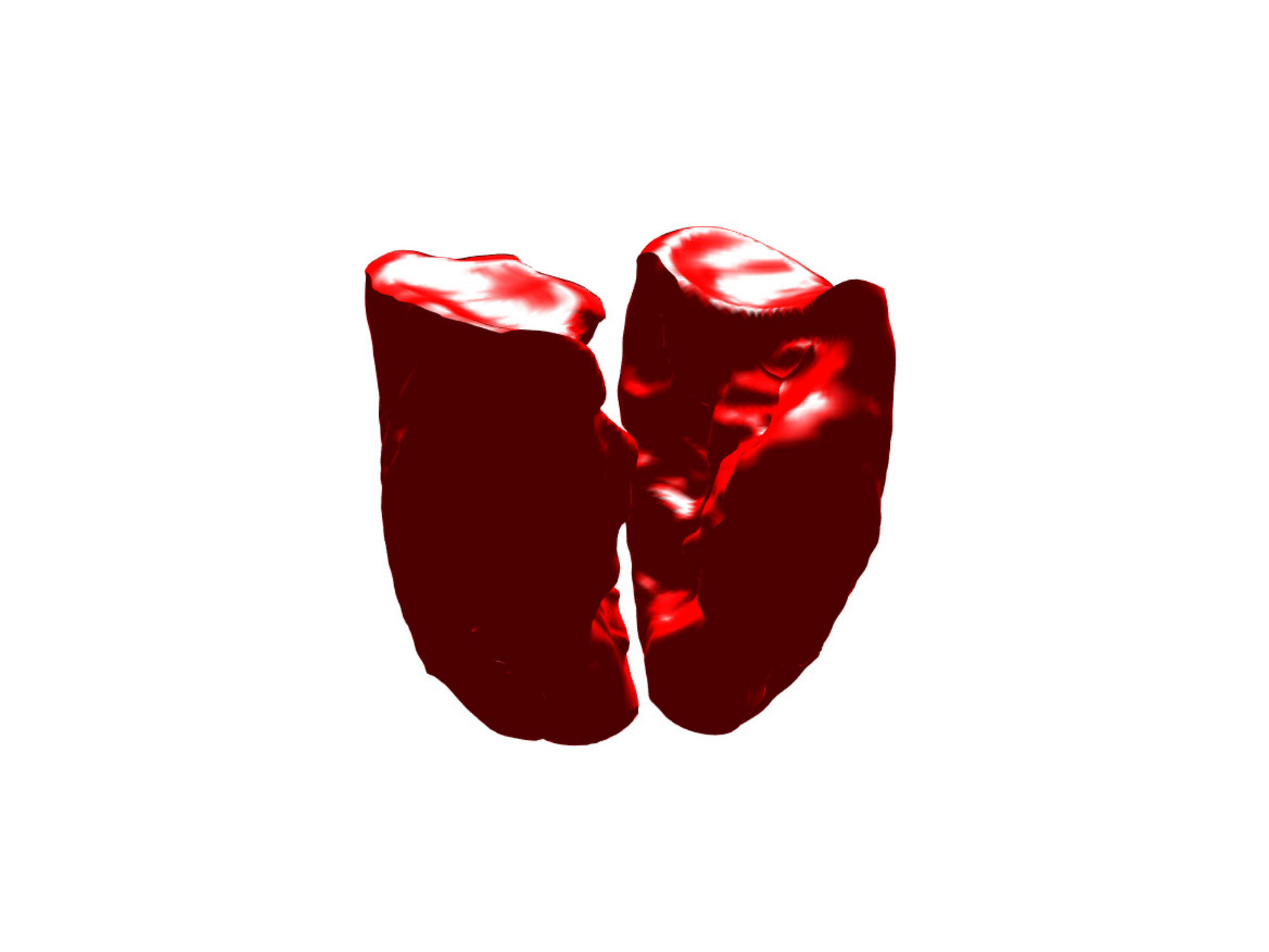}}
\subfigure[PM]{
\includegraphics[width=1.9in,height=1.5in]{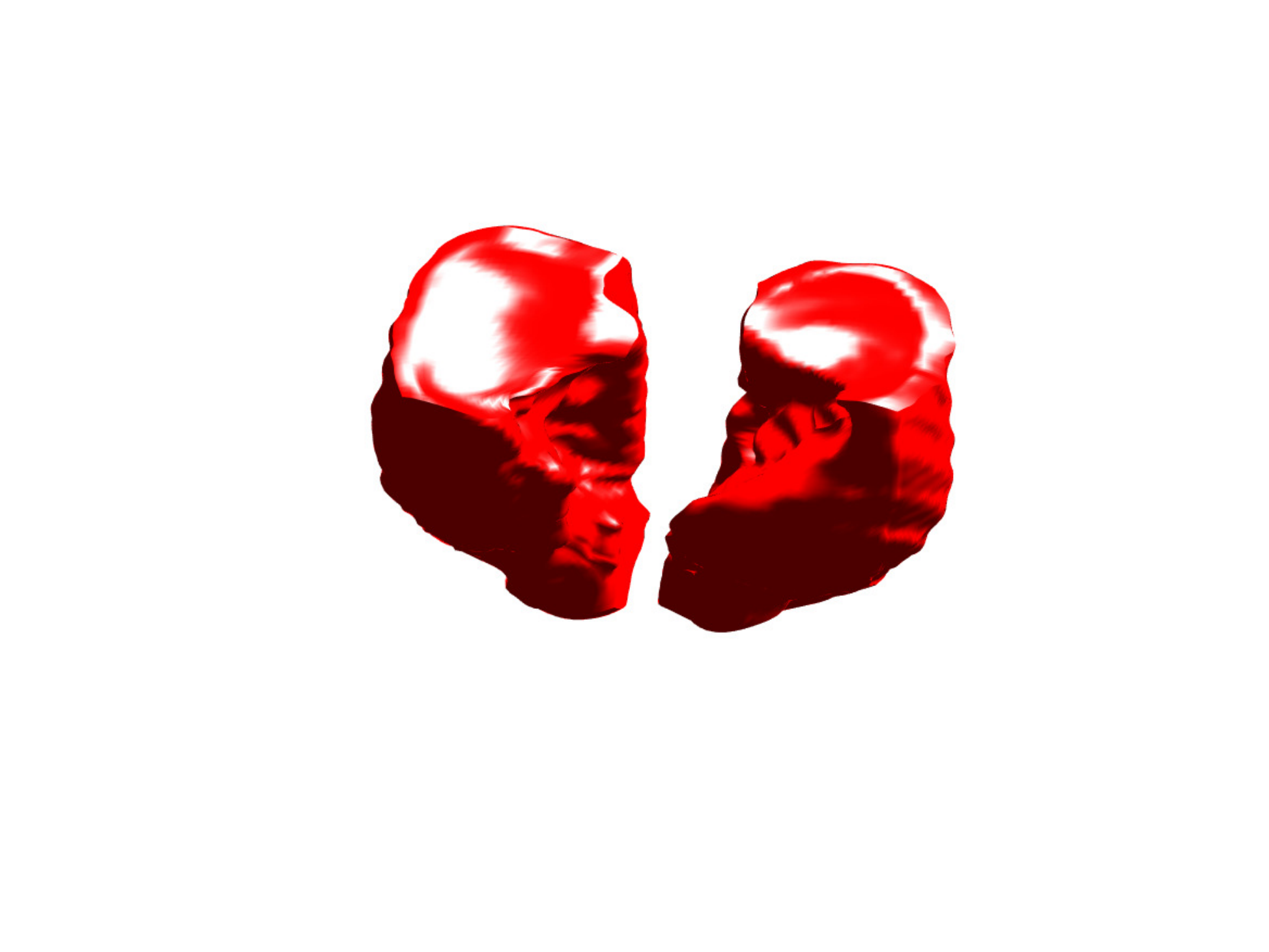}}
\subfigure[PM]{
\includegraphics[width=1.9in,height=1.5in]{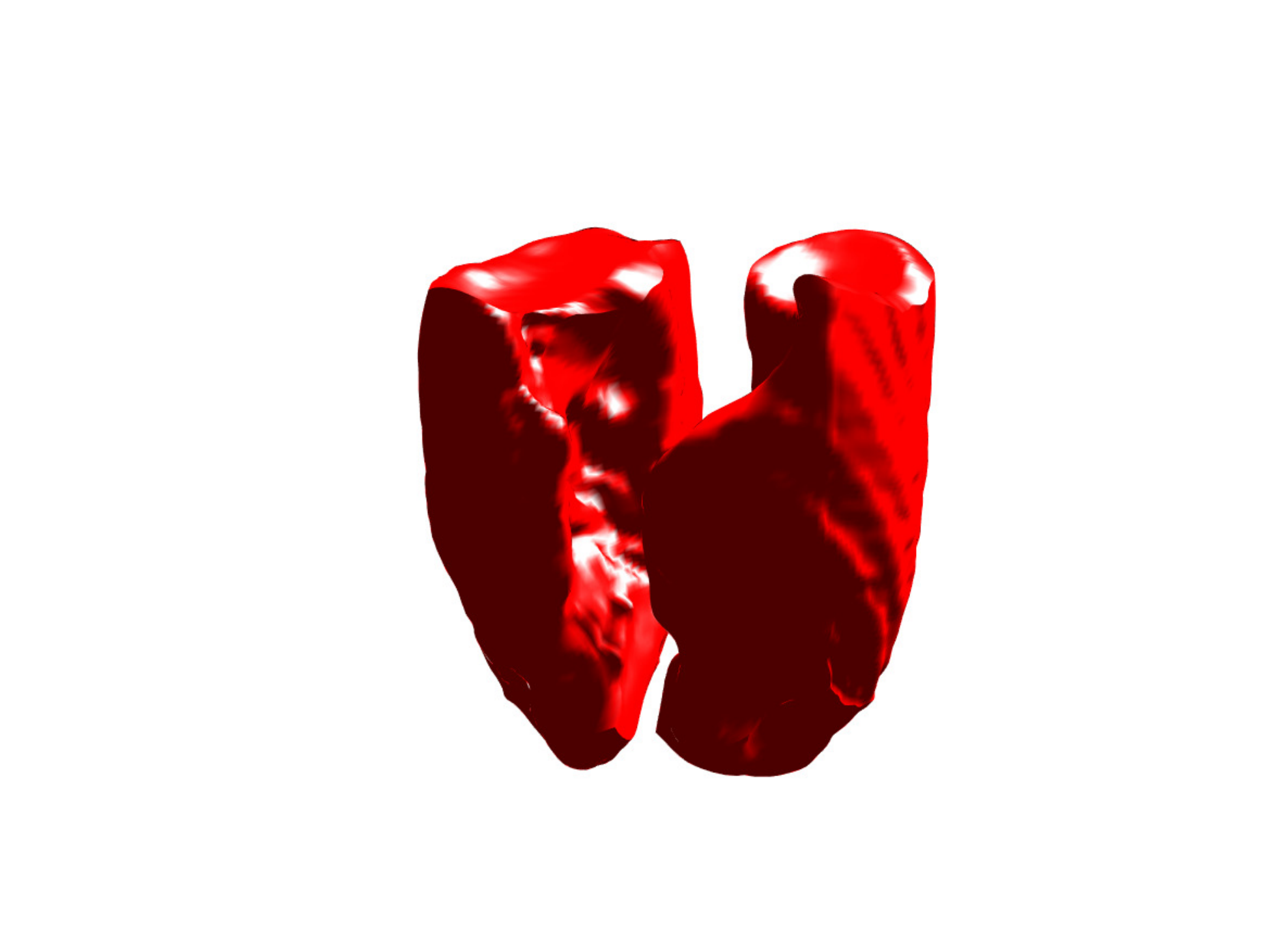}}\\
\subfigure[CV (150.12 sec)]{
\includegraphics[width=1.9in,height=1.5in]{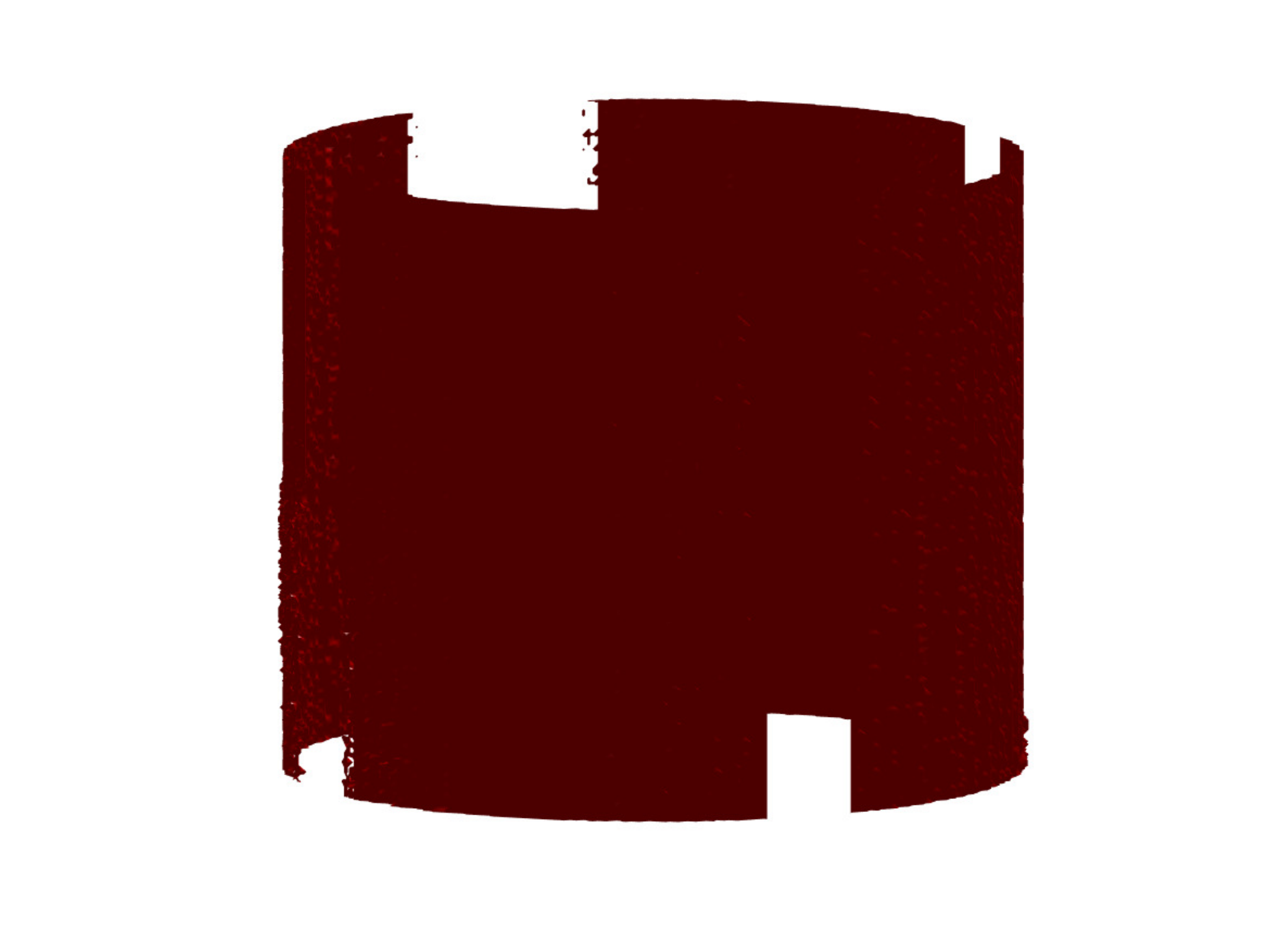}}
\subfigure[CV]{
\includegraphics[width=1.9in,height=1.5in]{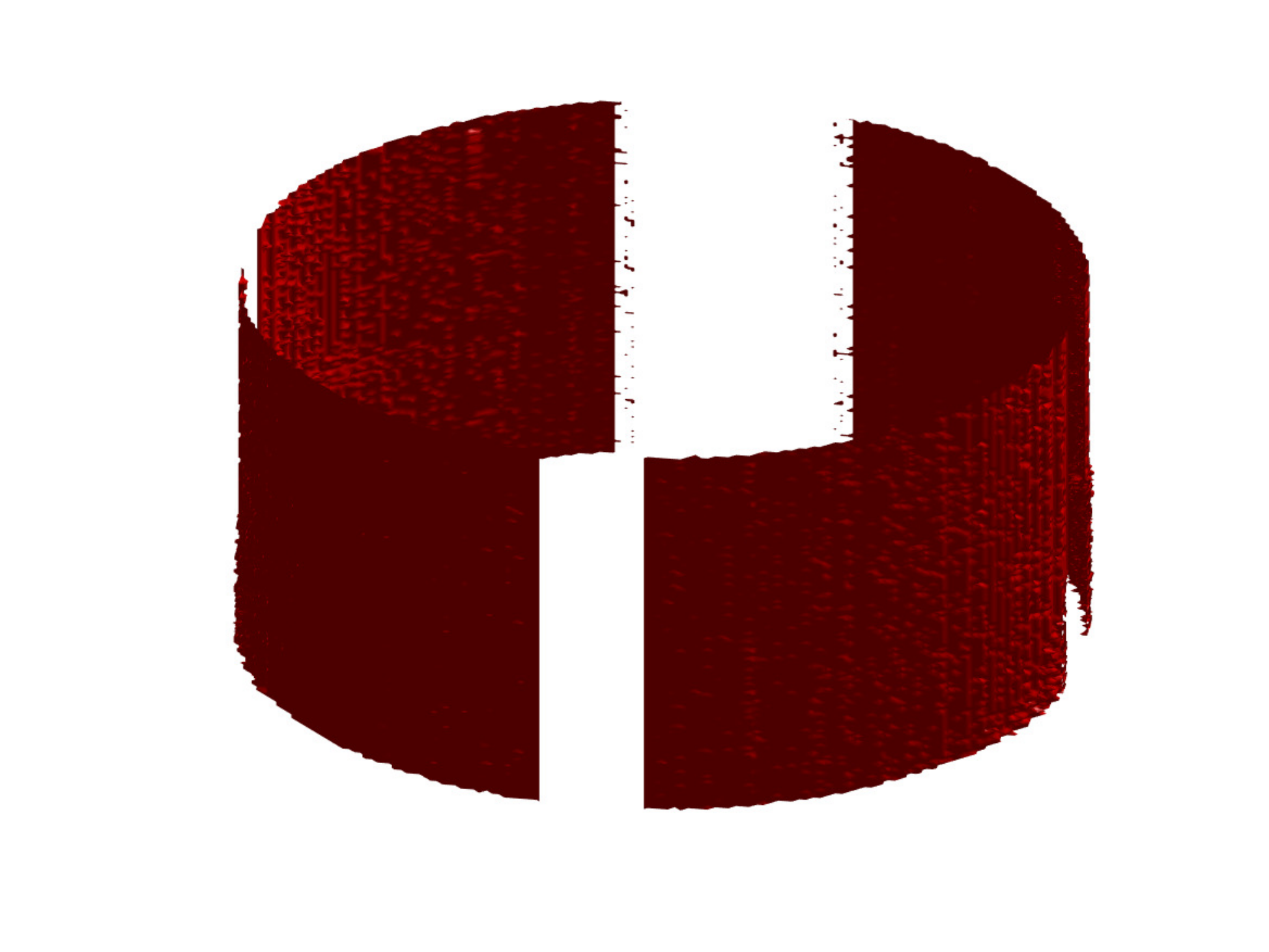}}
\subfigure[CV]{
\includegraphics[width=1.9in,height=1.5in]{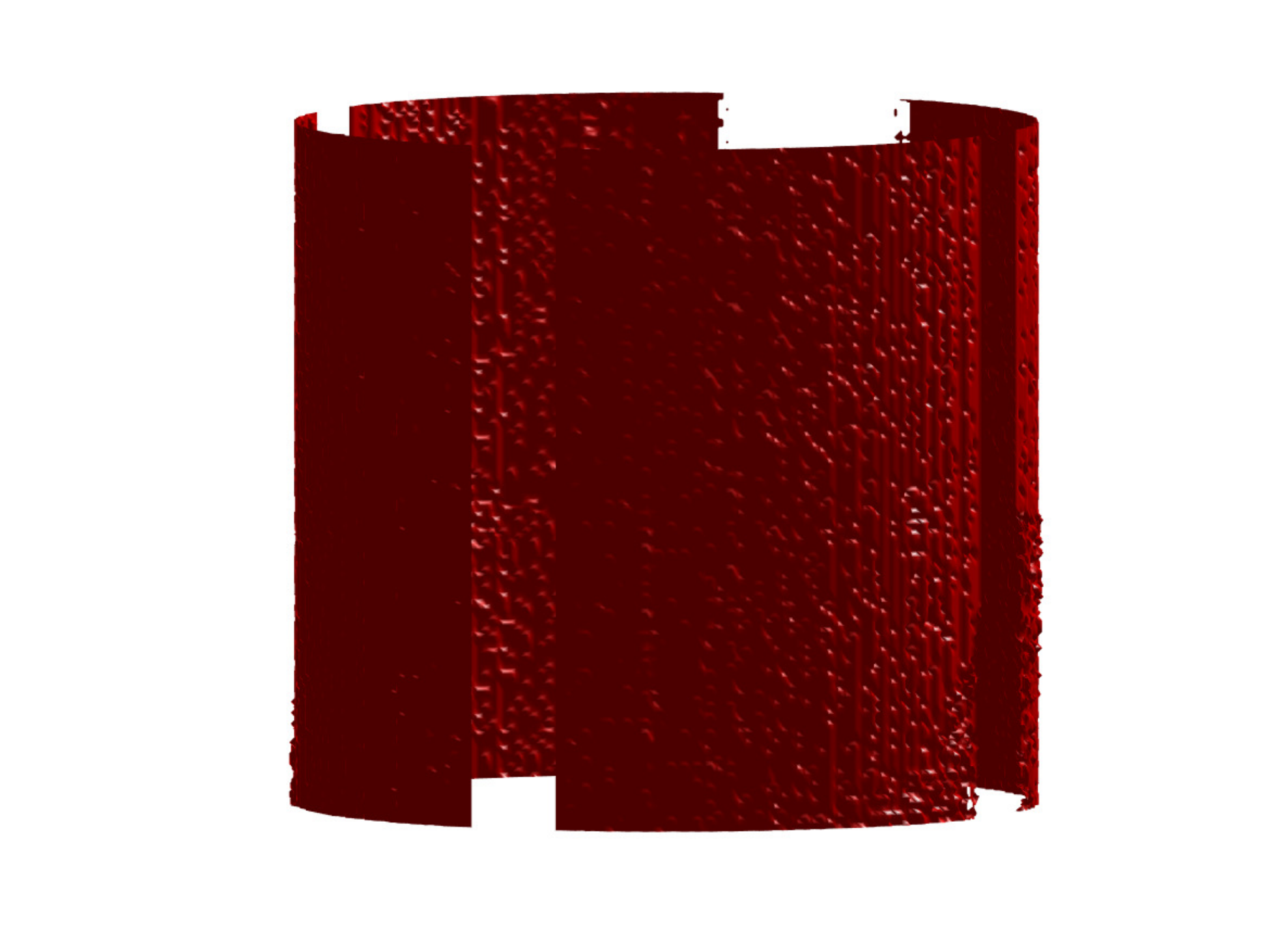}}\\
\subfigure[\cite{zhang2015fast} (1392.0 sec)]{
\includegraphics[width=1.9in,height=1.5in]{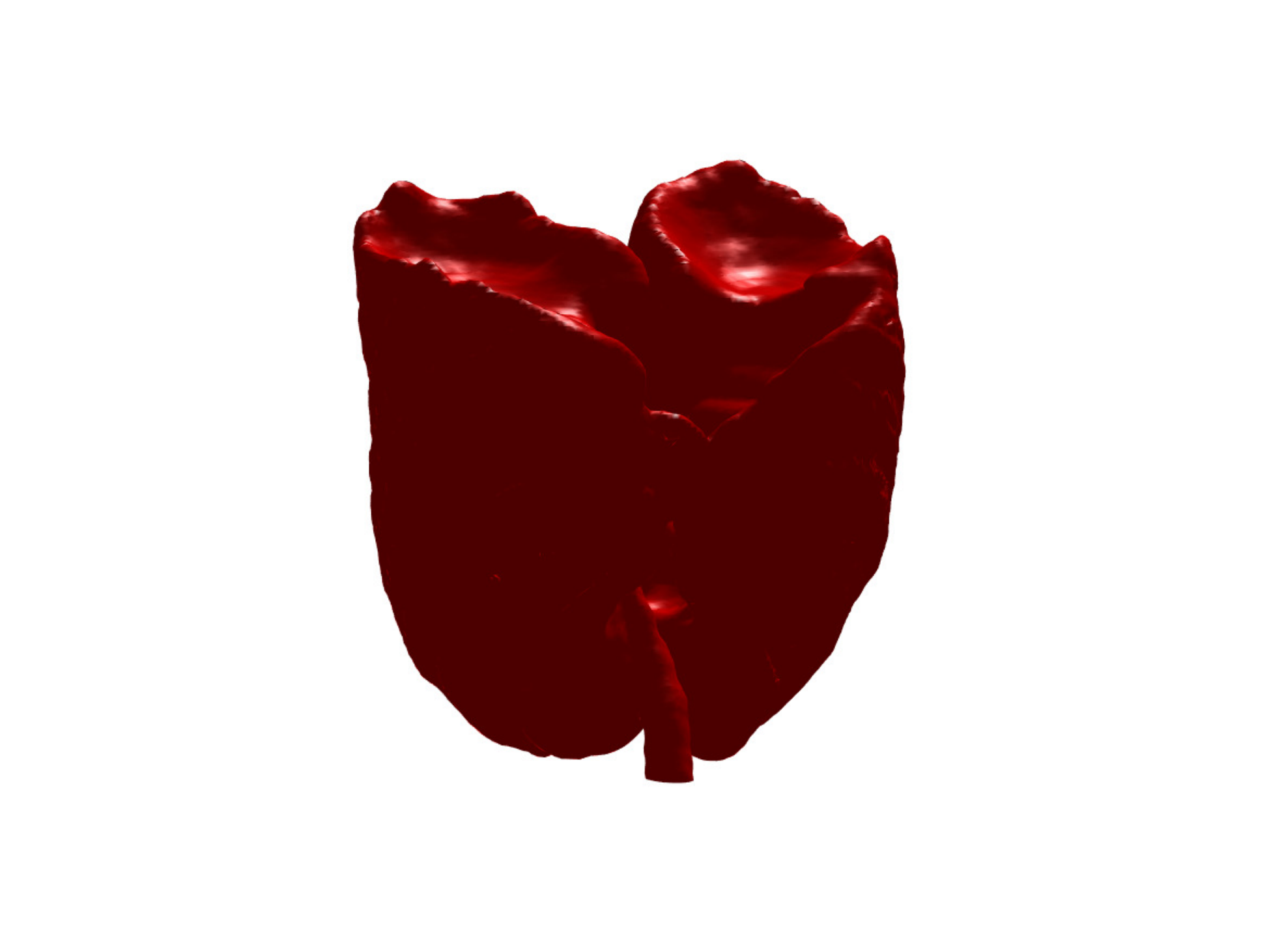}}
\subfigure[\cite{zhang2015fast}]{
\includegraphics[width=1.9in,height=1.5in]{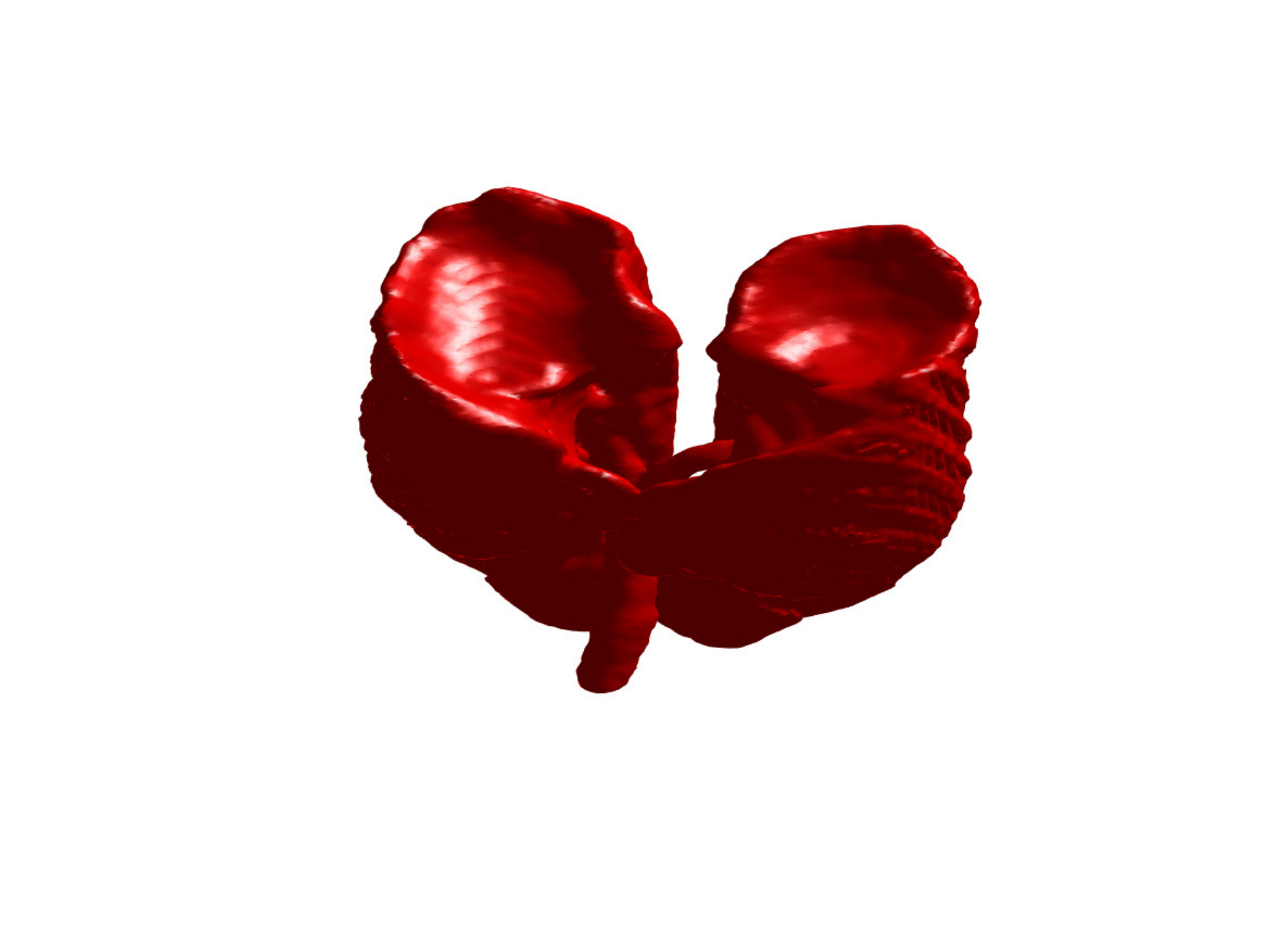}}
\subfigure[\cite{zhang2015fast}]{
\includegraphics[width=1.9in,height=1.5in]{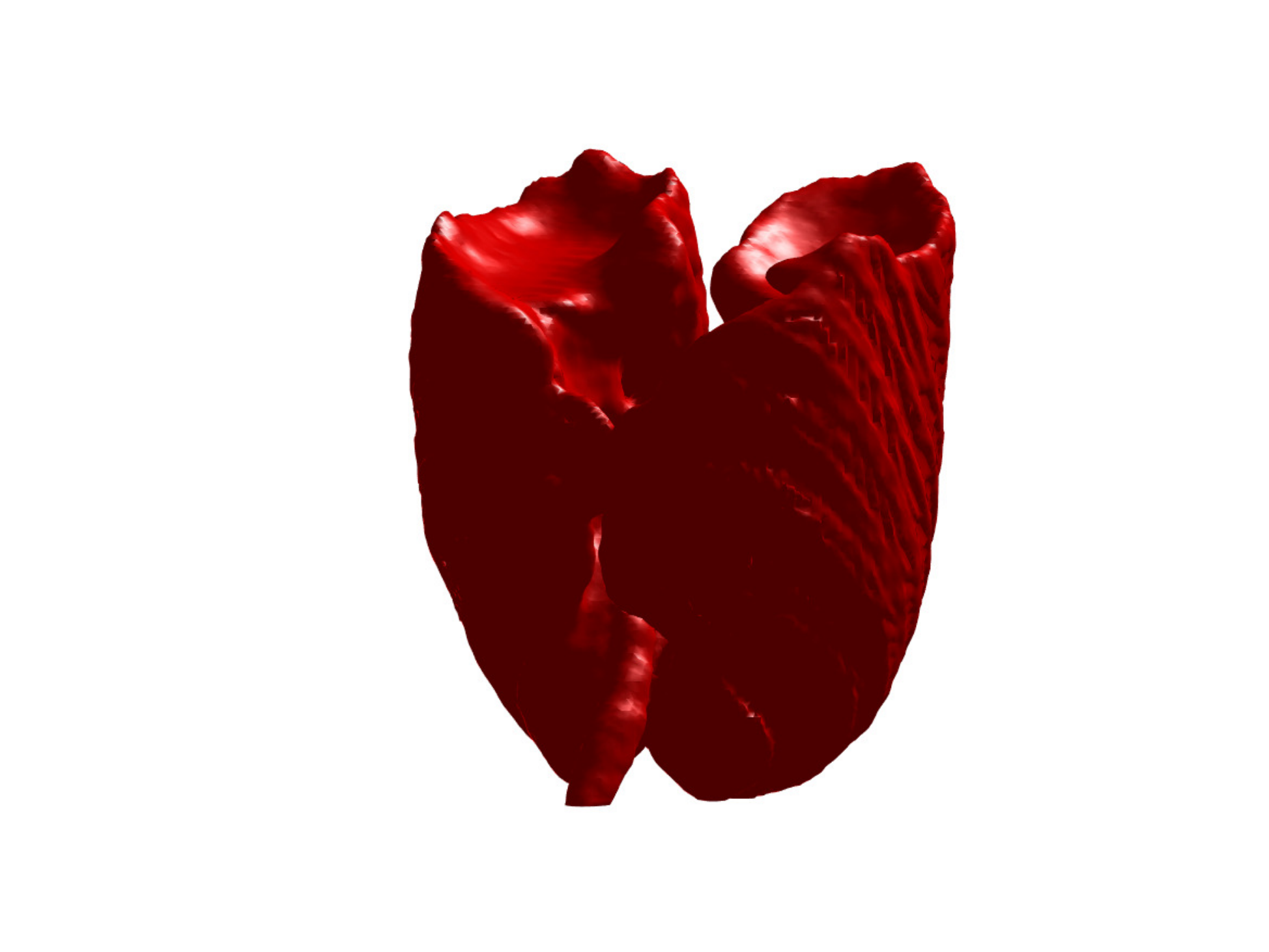}}\\
\caption{Here, we show the target image and prior image of the third 3D example in the first row. The second to fourth rows show the segmentation results by the proposed model \eqref{ProposedModel}, the Chan-Vese model and the selective model \cite{zhang2015fast} from different angles, respectively.}\label{3DResult2}
\end{figure}

\section{Conclusion}\label{SConclusion}
In this paper, we propose a topology-preserving segmentation model based on the hyperelastic registration. The proposed model \eqref{ProposedModel} is a registration-based segmentation model, which deforms a prior image to segment the target objects in a given image. The proposed model can handle both 2D and 3D images. By deforming a prior image bijectively, a 3D topology-preserving segmentation result can be guaranteed. The existence of the solution of the proposed model is theoretically established. In addition, we propose in this paper the generalized Gauss-Newton numerical scheme to solve the proposed model, whose convergence is rigorously shown. We test our proposed model on both synthetic and real images. Numerical experiments demonstrate the effectiveness of our proposed model for both 2D and 3D topology-preserving segmentation.

\section*{Acknowledgement}
We would like to thank Prof. Jan Modersitzki for his FAIR package \cite{modersitzki2009fair}
(https://github.com/C4IR/FAIR.m). We would also like to thank the anonymous reviewers for their valuable comments and suggestions to improve the quality of this manuscript. This work is partly supported by HKRGC GRF (Project ID: 2130656).

%
%

\appendix

\section{Computation of $A$ in \eqref{disR1}.}\label{A}
$A = I_{3}\otimes (A_{1}^{T},A_{2}^{T},A_{3}^{T})^{T}$, $A_{1} = I_{(n_{3}+1)}\otimes I_{(n_{2}+1)}\otimes \partial_{n_{1}}^{1,h_{1}}$, $A_{2}=I_{(n_{3}+1)}\otimes \partial_{n_{2}}^{1,h_{2}}\otimes I_{(n_{1}+1)}$, $A_{3}=\partial_{n_{3}}^{1,h_{3}}\otimes I_{(n_{2}+1)}\otimes I_{(n_{1}+1)}$ and
\begin{equation}
\partial_{n_{l}}^{1,h_{l}} = \frac{1}{h_{l}}\begin{pmatrix}
-1 & 1 & & \\
    & \cdot   & \cdot & \\
    &   &  -1 & 1
\end{pmatrix}\in\mathbb{R}^{n_{l},n_{l}+1},\quad 1\leq l \leq 3.
\end{equation}
Here, $\otimes$ indicates Kronecker product.

\section{Computation of $\bm{s}(Y)$ and $\bm{v}(Y)$ in \eqref{disR2}.}\label{S_V}
In each tetrahedron $\Omega^{i,j,k,l}$, set $\textbf L^{i,j,k,l}(\bm{x})= (L_{1}^{i,j,k,l}(\bm{x}),L_{2}^{i,j,k,l}(\bm{x}),L_{3}^{i,j,k,l}(\bm{x}))= (a^{i,j,k,l}_{1}x_{1}+a^{i,j,k,l}_{2}x_{2}+a^{i,j,k,l}_{3}x_{3}+b_{1}^{i,j,k,l}, a^{i,j,k,l}_{4}x_{1}+a^{i,j,k,l}_{5}x_{2}+a^{i,j,k,l}_{6}x_{3}+b_{2}^{i,j,k,l},a^{i,j,k,l}_{7}x_{1}+a^{i,j,k,l}_{8}x_{2}+a^{i,j,k,l}_{9}x_{3}+b_{3}^{i,j,k,l})$, which is the linear interpolation for $\bm{y}$ in the $\Omega^{i,j,k,l}$. Note that 
\begin{equation}
\begin{split}
\partial_{x_{1}} L^{i,j,k,l}_{1} = a^{i,j,k,l}_{1}, \partial_{x_{2}} L^{i,j,k,l}_{1} = a^{i,j,k,l}_{2},\partial_{x_{3}} L^{i,j,k,l}_{1} = a^{i,j,k,l}_{3},\\
\partial_{x_{1}} L^{i,j,k,l}_{2} = a^{i,j,k,l}_{4}, \partial_{x_{2}} L^{i,j,k,l}_{2} = a^{i,j,k,l}_{5},\partial_{x_{3}} L^{i,j,k,l}_{2} = a^{i,j,k,l}_{6},\\
\partial_{x_{1}} L^{i,j,k,l}_{3} = a^{i,j,k,l}_{7}, \partial_{x_{2}} L^{i,j,k,l}_{3} = a^{i,j,k,l}_{8},\partial_{x_{3}} L^{i,j,k,l}_{3} = a^{i,j,k,l}_{9}.
\end{split}
\end{equation}
Then the following approximation can be built:
\begin{equation}\label{disR2_1}
\int_{\Omega} \alpha_{s}\phi_{w}(\mathrm{cof}\nabla\bm{y})+\alpha_{v}\phi_{v}(\det\nabla\bm{y})\mathrm{d}\bm{x}\approx \frac{h^{3}}{6}\sum_{i=1}^{n}\sum_{j=1}^{n}\sum_{k=1}^{n}\sum_{l=1}^{6}(\alpha_{s}\phi_{w}(s^{i,j,k,l})+\alpha_{v}\phi_{v}(v^{i,j,k,l})),
\end{equation}
where 
\begin{equation}
s^{i,j,k,l} = 
\begin{pmatrix}
a_{5}^{i,j,k,l}a_{9}^{i,j,k,l}-a_{6}^{i,j,k,l}a_{8}^{i,j,k,l} & a_{6}^{i,j,k,l}a_{7}^{i,j,k,l}-a_{4}^{i,j,k,l}a_{9}^{i,j,k,l} & a_{4}^{i,j,k,l}a_{8}^{i,j,k,l}-a_{5}^{i,j,k,l}a_{7}^{i,j,k,l} \\
a_{3}^{i,j,k,l}a_{8}^{i,j,k,l}-a_{2}^{i,j,k,l}a_{9}^{i,j,k,l} & a_{1}^{i,j,k,l}a_{9}^{i,j,k,l}-a_{3}^{i,j,k,l}a_{7}^{i,j,k,l} & a_{2}^{i,j,k,l}a_{7}^{i,j,k,l}-a_{1}^{i,j,k,l}a_{8}^{i,j,k,l} \\
a_{2}^{i,j,k,l}a_{6}^{i,j,k,l}-a_{3}^{i,j,k,l}a_{5}^{i,j,k,l} & a_{3}^{i,j,k,l}a_{4}^{i,j,k,l}-a_{1}^{i,j,k,l}a_{6}^{i,j,k,l} & a_{1}^{i,j,k,l}a_{5}^{i,j,k,l}-a_{2}^{i,j,k,l}a_{4}^{i,j,k,l}
\end{pmatrix}
\end{equation}
and
\begin{equation}
\begin{split}
v^{i,j,k,l}&=a^{i,j,k,l}_{1}a^{i,j,k,l}_{5}a^{i,j,k,l}_{9}+a^{i,j,k,l}_{2}a^{i,j,k,l}_{6}a^{i,j,k,l}_{7}+a^{i,j,k,l}_{4}a^{i,j,k,l}_{8}a^{i,j,k,l}_{3}\\
&-a^{i,j,k,l}_{2}a^{i,j,k,l}_{4}a^{i,j,k,l}_{9}-a^{i,j,k,l}_{1}a^{i,j,k,l}_{6}a^{i,j,k,l}_{8}-a^{i,j,k,l}_{3}a^{i,j,k,l}_{5}a^{i,j,k,l}_{7}.
\end{split}
\end{equation}

In order to write \eqref{disR2_1} into a compact form, we construct $D_{l}, 1\leq l\leq 9$:
\begin{equation}\label{MatrixD}
\begin{matrix}
D_{1} = [M_{1},0,0], & D_{4} = [0,M_{1},0], & D_{7} = [0,0,M_{1}],\\
D_{2} = [M_{2},0,0], & D_{5} = [0,M_{2},0], & D_{8} = [0,0,M_{2}],\\
D_{3} = [M_{3},0,0], & D_{6} = [0,M_{3},0], & D_{9} = [0,0,M_{3}],
\end{matrix}
\end{equation}
where $M_{1}$, $M_{2}$ and $M_{3}$ are the discrete operators of $\partial_{x_{1}}$, $\partial_{x_{2}}$ and $\partial_{x_{3}}$ respectively and how to construct them is shown in \textbf{Appendix} \ref{M1M2M3}. Then we define $\bm{s}(Y)$ and $\bm{v}(Y)$ as follows:
\begin{equation}\label{matrixrepresentation}
\begin{split}
\bm{s}(Y) &= 
\begin{pmatrix}
D_{5}Y\odot D_{9}Y-D_{6}Y\odot D_{8}Y & D_{6}Y\odot D_{7}Y-D_{4}Y\odot D_{9}Y & D_{4}Y\odot D_{8}Y-D_{5}Y\odot D_{7}Y \\
D_{3}Y\odot D_{8}Y-D_{2}Y\odot D_{9}Y & D_{1}Y\odot D_{9}Y-D_{3}Y\odot D_{7}Y & D_{2}Y\odot D_{7}Y-D_{1}Y\odot D_{8}Y \\
D_{2}Y\odot D_{6}Y-D_{3}Y\odot D_{5}Y & D_{3}Y\odot D_{4}Y-D_{1}Y\odot D_{6}Y & D_{1}Y\odot D_{5}Y-D_{2}Y\odot D_{4}Y
\end{pmatrix},\\
\bm{v}(Y) &= D_{1}Y\odot D_{5}Y\odot D_{9}Y + D_{2}Y\odot D_{6}Y\odot D_{7}Y +D_{4}Y\odot D_{8}Y\odot D_{3}Y\\
           &-D_{2}Y\odot D_{4}Y\odot D_{9}Y-D_{1}Y\odot D_{6}Y\odot D_{8}Y-D_{3}Y\odot D_{5}Y\odot D_{7}Y,\\
\end{split}
\end{equation}
where $\odot$ denotes the Hadamard product of two matrices. Furthermore, set the $i$th component of $\bm{s}(Y)$ as 
\begin{equation}
\bm{s}(Y)_{i} = 
\begin{pmatrix}
(D_{5}Y\odot D_{9}Y-D_{6}Y\odot D_{8}Y)_{i} & (D_{6}Y\odot D_{7}Y-D_{4}Y\odot D_{9}Y)_{i} & (D_{4}Y\odot D_{8}Y-D_{5}Y\odot D_{7}Y)_{i} \\
(D_{3}Y\odot D_{8}Y-D_{2}Y\odot D_{9}Y)_{i} & (D_{1}Y\odot D_{9}Y-D_{3}Y\odot D_{7}Y)_{i} & (D_{2}Y\odot D_{7}Y-D_{1}Y\odot D_{8}Y)_{i} \\
(D_{2}Y\odot D_{6}Y-D_{3}Y\odot D_{5}Y)_{i} & (D_{3}Y\odot D_{4}Y-D_{1}Y\odot D_{6}Y)_{i} & (D_{1}Y\odot D_{5}Y-D_{2}Y\odot D_{4}Y)_{i}
\end{pmatrix}.
\end{equation}
Then we can see that $\bm{s}(Y)$ and $\bm{v}(Y)$ contain all approximated cofactors and determinants for all tetrahedrons.

\section{Computation of $M_{1}$, $M_{2}$ and $M_{3}$ in \eqref{MatrixD}}\label{M1M2M3}
We first investigate the linear approximation $L(x_{1},x_{2},x_{3}) = a_{1}x_{1}+a_{2}x_{2}+a_{3}x_{3}+b$ in the tetrahedron  $V_{3}V_{4}V_{5}V_{7}$ (Figure \ref{partition}). Denote these 4 vertices of this tetrahedron by $V_{3} = \bm{x}^{1,1,1}$, $V_{4} = \bm{x}^{2,2,2}$, $V_{5} = \bm{x}^{3,3,3}$ and $V_{7} = \bm{x}^{4,4,4}$. Set $L(\bm{x}^{1,1,1}) = y^{1,1,1}$, $L(\bm{x}^{2,2,2}) = y^{2,2,2}$, $L(\bm{x}^{3,3,3}) = y^{3,3,3}$ and $L(\bm{x}^{4,4,4}) = y^{4,4,4}$.
Substituting $V_{3},V_{4}$, $V_{5}$ and $V_{7}$ into $L$, we get
\begin{equation}
\begin{pmatrix}
x_{1}^{1} & x_{2}^{1} & x_{3}^{1} & 1\\
x_{1}^{2} & x_{2}^{2} & x_{3}^{2} & 1\\
x_{1}^{3} & x_{2}^{3} & x_{3}^{3} & 1\\
x_{1}^{4} & x_{2}^{4} & x_{3}^{4} & 1
\end{pmatrix}
\begin{pmatrix}
a_{1} \\
a_{2} \\
a_{3} \\
b
\end{pmatrix}
=
\begin{pmatrix}
y^{1,1,1}\\
y^{2,2,2}\\
y^{3,3,3}\\
y^{4,4,4}
\end{pmatrix}.
\end{equation}
Then eliminating $b$, we obtain
\begin{equation}
\begin{pmatrix}
x_{1}^{1}-x_{1}^{4} & x_{2}^{1}-x_{2}^{4} & x_{3}^{1}-x_{1}^{4} \\
x_{1}^{2}-x_{2}^{4} & x_{2}^{2}-x_{2}^{4} & x_{3}^{2}-x_{2}^{4} \\
x_{1}^{3}-x_{3}^{4} & x_{2}^{3}-x_{2}^{4} & x_{3}^{3}-x_{3}^{4}
\end{pmatrix}
\begin{pmatrix}
a_{1} \\
a_{2} \\
a_{3}
\end{pmatrix}
=
\begin{pmatrix}
y^{1,1,1}-y^{4,4,4}\\
y^{2,2,2}-y^{4,4,4}\\
y^{3,3,3}-y^{4,4,4}
\end{pmatrix}.
\end{equation}
Set
\begin{equation}
C = \begin{pmatrix}
x_{1}^{1}-x_{1}^{4} & x_{2}^{1}-x_{2}^{4} & x_{3}^{1}-x_{1}^{4} \\
x_{1}^{2}-x_{2}^{4} & x_{2}^{2}-x_{2}^{4} & x_{3}^{2}-x_{2}^{4} \\
x_{1}^{3}-x_{3}^{4} & x_{2}^{3}-x_{2}^{4} & x_{3}^{3}-x_{3}^{4}
\end{pmatrix}.
\end{equation}
Then we have
\begin{equation}
\begin{pmatrix}
a_{1} \\
a_{2} \\
a_{3}
\end{pmatrix}
= \frac{1}{\det}\begin{pmatrix}
C_{11} & C_{21} & C_{31} \\
C_{12} & C_{22} & C_{32} \\
C_{13} & C_{23} & C_{33}
\end{pmatrix}
\begin{pmatrix}
y^{1,1,1}-y^{4,4,4}\\
y^{2,2,2}-y^{4,4,4}\\
y^{3,3,3}-y^{4,4,4}
\end{pmatrix},
\end{equation}
where $\det$ is the determinant of $C$ and $C_{ij}$ is the $(i,j)$ cofactor of $C$. Since the domain $\Omega$ has been divided into $N$ voxels, in order to find all $a_{1}$ in the tetrahedron with the same position of each voxel, we can make it as follows:
\begin{equation}
\begin{pmatrix}
a_{1}^{1} \\
\vdots \\
a_{1}^{N}
\end{pmatrix}
= \frac{1}{\det}(C_{11}(E_{3}Y-E_{7}Y)+C_{21}(E_{4}Y-E_{7}Y)+C_{31}(E_{5}Y-E_{7}Y)),
\end{equation}
where $E_{l},l\in\{3,4,5,7\}$ is a matrix which extracts the corresponding positions of the vertices. Set $G_{1} = \frac{1}{\det}(C_{11}(E_{3}-E_{7})+C_{21}(E_{4}-E_{7})+C_{31}(E_{5}-E_{7}))$. For other 5 tetrahedrons, we can also build $G_{l}, l\in\{2,...,6\}$. Then we get
\begin{equation}\label{M1}
M_{1} =
\begin{pmatrix}
G_{1}\\
\vdots\\
G_{6}
\end{pmatrix}.
\end{equation}
Similarly, we can obtain $M_{2}$ and $M_{3}$.

\clearpage 

\bibliographystyle{spmpsci}     
\bibliography{reference}
\end{document}